\pdfoutput=1
\documentclass{article} %
\usepackage{iclr2027_conference,times}

\usepackage[utf8]{inputenc}
\usepackage[T1]{fontenc}
\usepackage{hyperref}
\usepackage{url}
\usepackage{latexml}
\usepackage{graphicx}
\usepackage{booktabs}
\usepackage{enumitem}
\usepackage{amsmath}
\usepackage{amssymb}
\usepackage{amsfonts}
\usepackage{amsthm}
\usepackage{mathtools}
\usepackage{microtype}
\usepackage{wrapfig}
\usepackage{longtable}
\usepackage{tikz}
\usetikzlibrary{positioning, arrows.meta, shapes.geometric, backgrounds, fit,
                calc, decorations.pathreplacing, patterns, matrix}
\usepackage{pgfplots}
\usepackage{pgfplotstable}
\pgfplotsset{compat=1.18}
\usepackage{xcolor}

\definecolor{cRoute}{HTML}{2C5F8A}   %
\definecolor{cRetr} {HTML}{C1590E}   %
\definecolor{cThird}{HTML}{2E7D6B}   %
\definecolor{cGrey} {HTML}{6B6B6B}   %
\definecolor{cLight}{HTML}{D8E2EC}   %
\tikzset{
  fignode/.style   = {font=\footnotesize, inner sep=2pt},
  figbox/.style    = {draw=cGrey, line width=.5pt, rounded corners=1.5pt,
                      font=\footnotesize, inner sep=3pt, align=center},
  figarrow/.style  = {-{Stealth[length=4pt,width=3pt]}, line width=.6pt,
                      draw=cGrey},
  cell/.style      = {draw=cGrey!55, line width=.35pt, minimum size=#1,
                      inner sep=0pt, outer sep=0pt},
  cell/.default    = 4.2mm,
}
\pgfplotsset{
  figaxis/.style = {
    width=\linewidth, font=\footnotesize,
    axis line style={draw=cGrey, line width=.5pt},
    tick align=outside, tick style={draw=cGrey, line width=.4pt},
    every axis label/.append style={font=\footnotesize},
    label style={font=\footnotesize}, tick label style={font=\scriptsize},
    title style={font=\footnotesize, yshift=-1pt},
    legend style={font=\scriptsize, draw=cGrey!60, line width=.4pt,
                  fill=white, fill opacity=.9, text opacity=1, inner sep=2pt},
    grid style={draw=cGrey!22, line width=.35pt},
  },
}

\theoremstyle{plain}
\newtheorem{theorem}{Theorem}
\newtheorem{proposition}[theorem]{Proposition}
\newtheorem{lemma}[theorem]{Lemma}
\newtheorem{corollary}[theorem]{Corollary}
\newtheorem{example}[theorem]{Example}
\theoremstyle{definition}
\newtheorem{definition}[theorem]{Definition}
\theoremstyle{remark}
\newtheorem{remark}[theorem]{Remark}
\newtheorem*{rmk}{Remark}
\newtheorem*{exmp}{Example}

\newcommand{\Grid}{\mathcal{G}}
\newcommand{\Gd}{\widehat{\Grid}}
\newcommand{\States}{\mathcal{S}}
\newcommand{\Nset}{\mathcal{N}}
\newcommand{\Uset}{\mathcal{U}}
\newcommand{\Wmax}{W_{\max}}
\newcommand{\usucc}{\hat u_{d}}
\newcommand{\dmodel}{d_{\mathrm{model}}}
\newcommand{\dS}{d_{S}}
\newcommand{\dpos}{d_{p}}
\newcommand{\dmlp}{d_{\mathrm{MLP}}}
\newcommand{\eret}{\varepsilon_{\mathrm{ret}}}
\newcommand{\eblock}{\bar\varepsilon}
\newcommand{\KL}{D_{\mathrm{KL}}}
\newcommand{\rank}{\operatorname{rank}}
\newcommand{\supp}{\operatorname{supp}}

\newcommand{\apppart}[2]{%
  \par\addvspace{1.4\baselineskip}%
  \addtocontents{toc}{%
    \protect\addvspace{5pt}%
    \protect\noindent\textbf{#1 --- #2}\protect\par\protect\addvspace{1pt}}%
  \noindent\rule{\linewidth}{0.7pt}\par\vspace{2pt}
  \noindent{\normalsize\textbf{\textsc{#1} --- \textsc{#2}}}\par\vspace{2pt}
  \noindent\rule{\linewidth}{0.7pt}%
  \par\addvspace{0.7\baselineskip}\nobreak}

\title{Spatial Induction Heads: In-Context Learning\\of Multidimensional
Cellular Automata\thanks{Project page: \url{https://kimzemian.github.io/spatial-induction-heads/}}}

\iflatexml
\author{Kimia Kazemian\thanks{Equal contribution.} \\ Cornell University \\ \texttt{kk983@cornell.edu}
\and Menghan Xu\thanks{Equal contribution.} \\ Cornell University \\ \texttt{mx253@cornell.edu}
\and John Thickstun \\ Cornell University \\ \texttt{jthickstun@cornell.edu}
\and Sarah Dean \\ Cornell University \\ \texttt{sdean@cornell.edu}}
\else
\author{Kimia Kazemian\thanks{Equal contribution.}, \; Menghan Xu\footnotemark[2], \; John Thickstun, \; Sarah Dean \\
Cornell University \\
\texttt{\{kk983,mx253,jthickstun,sdean\}@cornell.edu}}
\fi

\iclrfinalcopy
\begin{document}

\maketitle
\lhead{Under review as a conference paper at ICLR 2027}
\suppressfloats[t]
\addtocontents{toc}{\protect\setcounter{tocdepth}{-10}}

\begin{abstract}
Induction heads provide a mechanistic account of in-context learning in sequential data, but existing theory largely assumes that the context relevant to a prediction forms a contiguous block. In multidimensional data, serialization breaks this assumption by scattering spatial neighbors across distant positions in the token sequence. We study how transformers overcome this routing problem in multi-dimensional stochastic and deterministic cellular automata, where each trajectory is generated by an unknown local rule and presented as a flattened sequence without an explicit coordinate-based spatial inductive bias. We introduce \emph{spatial induction heads}, two-layer gather-and-match circuits in which the first layer reconstructs the relevant spatial neighborhood and the second matches the resulting configuration against earlier occurrences. We give two explicit realizations of the gather and show that the positional dimension required for spatial routing depends only on the local neighborhood and spatial dimension, not on grid volume or trajectory horizon. We further construct a matching layer which implements Bayesian counting.
The end-to-end circuit can approximate the Bayesian posterior arbitrarily closely for stochastic rules and can predict exactly for deterministic rules. Empirically, trained two-layer transformers generalize to unseen rules in one and two dimensional settings, achieving near-perfect deterministic rollouts and less than $0.005$ nats KL from the Bayes-optimal predictor on stochastic rules. Attention patterns and layerwise probes align with the predicted gather-and-match computation, providing mechanistic evidence for spatial induction in trained transformers.
\end{abstract}

\section{Introduction}
\label{sec:intro}

\begin{figure}[t]
\centering
\includegraphics[width=0.9\linewidth]{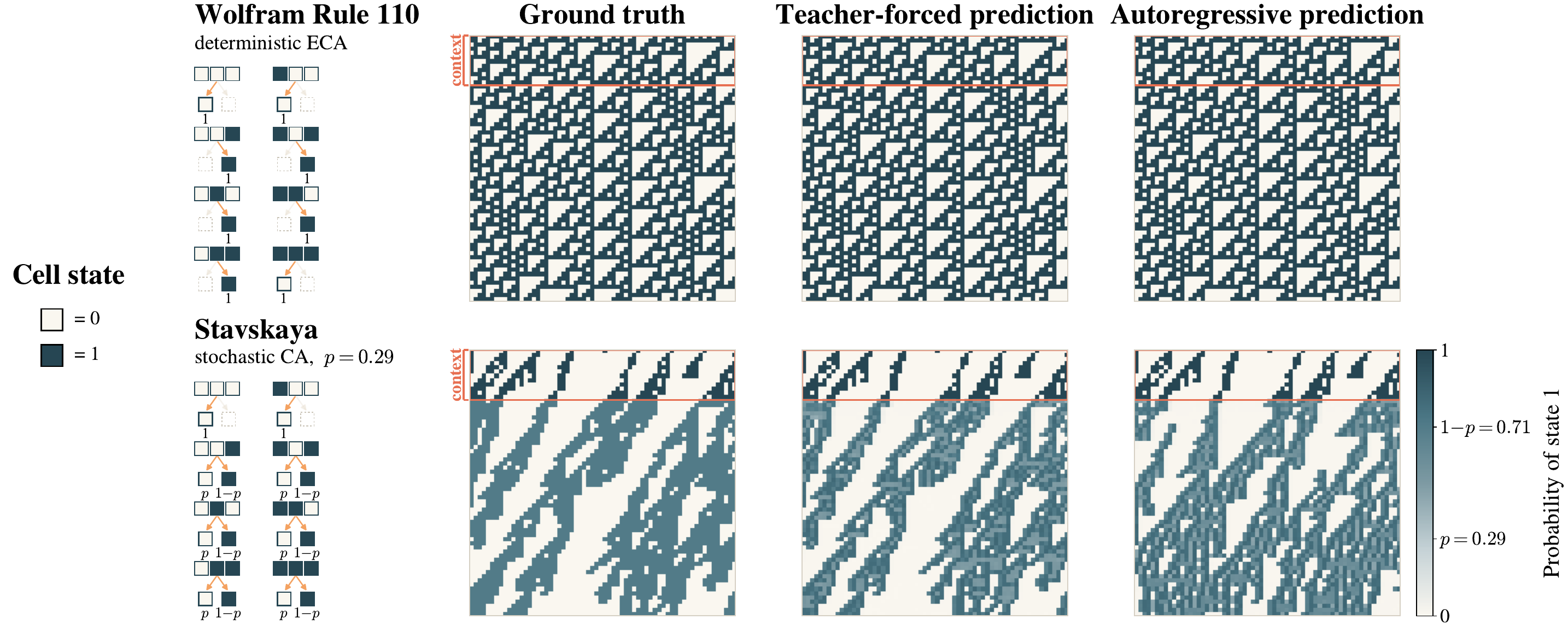}
\caption{\textbf{Trained-model predictions on two held-out rules.} 
Top row: Wolfram Rule~110, a Turing-complete deterministic ECA. Bottom row: Stavskaya, a stochastic CA with $p{=}0.29$; predictions are colored by the probability the model assigns to state~1. Teacher-forced predictions (middle column) condition on the ground-truth history at every step; autoregressive predictions (right column) are rolled out from the context alone, conditioning on the sample from the previous timestep.} 
\label{fig:rollout}
\vspace{-2mm}
\end{figure}

The \emph{induction head} provides one of the best-understood mechanistic accounts of in-context learning \citep{elhage2021mathematical,olsson2022context}. This two-layer circuit predicts novel sequences by completing patterns from precedent: the first layer records preceding tokens, while the second layer matches the current context to past records and copies the subsequent tokens. Induction heads have been used to explain the in-context behavior of large language models \citep{brown2020language,wei2022emergent}, and their theoretical foundations have been rigorously established for Markov chains \citep{bietti2023birth,edelman2024evolution}. In this setting, the \emph{statistical induction head} attends to the preceding $k$ tokens to retrieve the empirical transition law \citep{ekbote2025one}. However, this theory relies on a rarely emphasized assumption: the context block is a contiguous sequence immediately preceding the query, reachable by a single attention layer using relative position.

When transformers are applied to multidimensional spatiotemporal data, this assumption breaks down. Serializing a spatial grid into a one-dimensional token stream destroys local adjacency, scattering immediate neighbors across distant intervals in the sequence. 
The original Vision Transformer \citep{dosovitskiy2020image}
uses learned 1D absolute positional embeddings and shows that
these embeddings recover aspects of the underlying 2D grid
structure. Learned absolute positional embeddings also appear
in recent image generators \citep{li2024autoregressive,yu2024randomized,bachmann2025flextok}. Their success motivates studying
how Transformers learn structural relationships from token
sequences without explicit geometric inductive biases
on their positional embeddings.

Despite this empirical success, there is no mechanistic understanding of how a transformer leverages a learned positional table to gather a neighborhood fractured by serialization, nor how it performs \emph{spatial in-context learning}. Existing theories for $k$-order Markov chains assume contiguous context \citep{edelman2024evolution}, offering no analogue for sequential-to-spatial routing. This raises two fundamental questions: How does a transformer address and assemble scattered tokens to reconstruct a spatial neighborhood? And once assembled, how does it optimally infer the underlying dynamical rule without any coordinate-based spatial inductive bias?

We address these questions by studying transformers on multidimensional stochastic cellular automata, which are historically viewed as systems of \emph{locally interacting Markov chains} evolving on a multidimensional grid \citep{fernandez2018overview}. A stochastic cellular automaton consists of a grid of finite-state cells updated in parallel via a local kernel, mapping a cell's neighborhood at one step to a state distribution at the next (Definition~\ref{def:ca}). In our study, the kernel is sampled afresh for each trajectory from a Dirichlet prior and hidden from the model, requiring pure in-context inference. This provides an ideal testbed: simple local rules generate essentially unbounded complexity \citep{langton1990computation}, and the dependence structure is a spatial analogue to the Markov setting. We train two-layer transformers on this task. Figure~\ref{fig:rollout} shows them continuing a trajectory from a short context on two held-out rules, the deterministic Wolfram Rule~110 and the stochastic Stavskaya rule at $p = 0.29$.

\textbf{Contributions.}
We formalize spatial in-context learning as a problem where serialization hides the geometry that the model must exploit. Our contributions are as follows:
\begin{itemize}[noitemsep,topsep=0pt,parsep=0pt,partopsep=0pt,leftmargin=1.4em]
  \item \textbf{Task \& Theoretical Framework:} We consider multidimensional cellular automata and derive the closed-form Bayes-optimal predictor (Theorem~\ref{thm:bayes}). We then introduce the \emph{spatial induction head} (Definition~\ref{def:sih}), a two-layer circuit whose first layer \emph{gathers} the spatial dependencies needed at each token, and whose second \emph{matches} the resulting configuration against precedent.
  \item \textbf{Mechanistic Constructions:} We give two explicit realizations of the gather, one using a head per offset and one using a single content-routed head with a separating MLP (Theorem~\ref{thm:gather}, parts~\ref{thm:parallel} and~\ref{thm:onehead}). Both route scattered spatial neighbors with a compact set of positional frequencies whose number is independent of grid volume and horizon (Lemma~\ref{thm:addressing}).
  \item \textbf{Bayesian Matching:} We show the matching layer performs Bayesian counting: once the gathered representation is available, a single Layer-2 head 
  computes the posterior exactly under hard attention (Proposition~\ref{thm:sink}) and arbitrarily closely otherwise (Theorem~\ref{thm:main}).
  \item \textbf{Empirical Verification:} We validate the theory on 1D and 2D cellular automata. Two-layer transformers generalize to held-out rules (Figure~\ref{fig:rollout}), surpassing prior ICL results on cellular automata (Section~\ref{sec:related}), while attention maps (Figures~\ref{fig:gridattn} and~\ref{fig:attn_mass}) and layerwise probes (Figure~\ref{fig:probes}) give strong evidence for the predicted gather-and-match computation.
\end{itemize}

\section{Related work}
\label{sec:related}

\textbf{Transformers for automata and dynamical systems.}
Transformers have been studied on discrete automata both as diagnostic tasks and as synthetic sequence models. \citet{vafa2024evaluating} use deterministic finite automata to probe limitations of large models, while \citet{liu2022transformers,berkovich2025automatagpt,burtsev2024learning} train smaller transformers on cellular automata. Only \citet{burtsev2024learning} studies in-context generalization to unseen rules; \citet{liu2022transformers} train per automaton and \citet{berkovich2025automatagpt} provide the rule in the prompt. Related work finds in-context prediction in discretized chaotic systems and hidden Markov processes \citep{liu2024llms,daipre}, but does not identify a circuit for recovering spatially scattered dependencies.

\textbf{Induction heads and Markov-chain ICL.}
Induction heads originate in mechanistic studies of language models \citep{elhage2021mathematical,olsson2022context,conmy2023towards}. Theory has since characterized their operation on $k$-order Markov chains \citep{bietti2023birth,makkuva2024local,makkuva2024attention,nichani2024transformers,rajaraman2024transformers}, including statistical induction heads \citep{edelman2024evolution} and a two-layer construction for general $k$ \citep{ekbote2025one}. Our setting extends this line: probabilistic cellular automata are precisely systems of locally interacting Markov chains \citep{fernandez2018overview}. Serialization scatters the $k$ relevant neighbors through the token stream, so the model must first gather them before performing the usual 
induction head computation.

\textbf{Bayesian and algorithmic views of in-context learning.}
A complementary account interprets ICL as approximate posterior inference under a prior induced by pretraining \citep{xie2022explanation,muller2022transformers}. In our stochastic cellular-automaton setting, the Bayes predictor is a smoothed count, and we show that the matching head represents it. Another line shows that transformers can implement optimization procedures such as gradient descent and higher-order variants \citep{von2023transformers,zhang2024trained,ahn2023transformers,giannou2024well}; \citet{cole2025context} extends this view to linear dynamical systems. These works characterize an implemented learning algorithm, without considering spatial structure, whereas we characterize a circuit and the estimator it computes.

\textbf{Spatial structure in transformers.}
Architectures for grid-structured data often supply geometry explicitly, as in weather models using coordinate-dependent attention or mesh graphs \citep{bi2023accurate,lam2023learning}. Vision transformers instead flatten patches and can use learned one-dimensional absolute position embeddings, with learned embeddings recovering row--column structure \citep{dosovitskiy2020image}. We adopt this geometry-free setting and ask how a standard transformer uses learned positions to route a scattered neighborhood and perform spatial in-context inference.

\section{The task and the routing problem}
\label{sec:task}

In this section, we cover four key pieces of background. First, we
need well specified geometric data, so we consider the setting of stochastic
cellular automata (\S\ref{sec:task:ca}). Then, we describe how we flatten the
grid into a one-dimensional token stream, and we identify the set of offsets each prediction must reach
(\S\ref{sec:task:routing}). Furthermore, we describe the two-layer transformer
we study, which is given no geometric prior beyond a learned absolute positional
table (\S\ref{sec:task:model}). Finally, we identify the queries on which
spatial in-context learning is well posed, and give the Bayes-optimal
predictor for them (\S\ref{sec:task:correct}).

\subsection{Cellular automata}
\label{sec:task:ca}

\begin{definition}[Stochastic cellular automaton]
\label{def:ca}
Fix a torus $\Grid = [W_1] \times \cdots \times [W_d]$, writing $[n] \coloneqq \{0,\dots,n-1\}$, of volume $\lvert\Grid\rvert \coloneqq \prod_{j=1}^{d} W_j$, with $W_j \geq 2$. Cells of $\Grid$ are written $z = (z_1,\dots,z_d)$. Let $\States = [V]$ ($V \geq 2$) be a state set and $\delta_1,\dots,\delta_k \in \mathbb{Z}^d$ a \emph{neighborhood} of $k \geq 1$ offsets, distinct modulo $(W_1,\dots,W_d)$. A \emph{local kernel} $f : \States^{k} \to \Delta^{V-1}$, with $\Delta^{V-1}$ the probability simplex on $\States$, defines a trajectory over horizon $T \geq 2$ from initial states $(x_{0,z})_{z \in \Grid}$, drawn i.i.d.\ uniformly on $\States$, via transitions for $t \in \{1,\dots,T-1\}$, conditionally independent across $z$ given step $t-1$:
\[
  x_{t,z} \;\sim\; f(\,\cdot \mid P_{t,z}), \qquad
  P_{t,z} \;\coloneqq\; \big(x_{t-1,\,z+\delta_1},\,\dots,\, x_{t-1,\,z+\delta_k}\big) \in \States^{k},
\]
where addition is modulo $(W_1,\dots,W_d)$. We write $(t,z)$ for the cell $z$ at step $t$. $P_{t,z}$ is its \emph{parent configuration}, whose entries are ordered by $\delta_1,\dots,\delta_k$, and we write $D_K \coloneqq \{\delta_1,\dots,\delta_k\}$ for the set of \emph{key offsets}, so $\lvert D_K\rvert = k$. The kernel $f$ is drawn from the prior $\bigotimes_{c \in \States^{k}} \mathrm{Dir}(\alpha\mathbf{1}_V)$ with $\alpha > 0$ and $\mathbf{1}_V$ the all-ones vector, so the rows $f(\cdot \mid c)$, one per configuration $c$, are independent.
\end{definition}

This is a standard probabilistic cellular automaton on a torus
\citep{fernandez2018overview}. The neighborhood
$D_K$ is fixed, and we ask whether the model can recover it. The kernel $f$ is
drawn afresh for every trajectory and is never shown to the model, whose
parameters are fixed at test time. It observes only the flattened trajectory,
and must infer $f$ from earlier occurrences of the same neighborhood
configuration within that one trajectory. We also study purely deterministic
rules, which are the $\alpha \to 0$ limit of the prior, where every row
$f(\cdot \mid c)$ is a point mass.

\begin{exmp}
The classical \emph{elementary cellular automata} \citep{wolfram2003new} have $d = 1$, $V = 2$ and
$k = 3$, with offsets $(-1,0,1)$. In $d = 2$ the \emph{von Neumann}
neighborhood has $k = 5$, the offsets $(0,0)$, $(\pm 1, 0)$ and $(0, \pm 1)$,
while the \emph{Moore} neighborhood used in Conway's Game of Life
\citep{conway1970conway} has $k = 9$, the offsets $(a,b)$ with
$a, b \in \{-1,0,1\}$.
A neighborhood of size $k$ over $V$ states admits $V^{k}$ configurations and
therefore $V^{(V^{k})}$ deterministic rules, which is $256$ in the elementary
case, $2^{32}$ for von Neumann and $2^{512}$ for Moore. Stochastic rules  assign a
distribution over $\States$, so they form a continuum of
dimension $V^{k}(V-1)$.
\end{exmp}

\subsection{Sequence flattening and spatial routing}
\label{sec:task:routing}

A trajectory is presented to the model as a single token stream. Within a time
step the cells are emitted in row-major order, with the last coordinate varying
fastest, and a \texttt{ROW\_SEP} closes each innermost row, except that the last
row of a step is closed by a \texttt{TIME\_SEP} in place of its
\texttt{ROW\_SEP}. One step therefore reads $\lvert\Grid\rvert$ cell tokens and $\lvert\Grid\rvert/W_d$
separators:
\[
    \underbrace{x_{t,(0,0)},\,\dots,\,x_{t,(0,W_d-1)}}_{\text{row } 0},\;
    \texttt{ROW\_SEP},\;\dots,\;
    \underbrace{x_{t,(m,0)},\,\dots,\,x_{t,(m,W_d-1)}}_{\text{row } m = \lvert\Grid\rvert/W_d-1},\;
    \texttt{TIME\_SEP},
\]
for $t = 0,\dots,T-1$ in order, so the stream has length
$N = T(\lvert\Grid\rvert + \lvert\Grid\rvert/W_d)$. In $d = 1$ each time step is a single row, so the only
separator is the \texttt{TIME\_SEP}. The layout is an order-preserving injection
$\rho : [T] \times \Grid \to [N]$, and we give the stream indexing in full in
Appendix~\ref{app:serialization}.
The flattening preserves adjacency along the innermost
axis while obscuring adjacency along the others;
the $k$ tokens in the parent configuration need not form a contiguous block.

In the transformer architecture, the cell token $(t,z)$ predicts the cell that follows it
in the stream, one step along the row. We call such a token a \emph{query cell}
and the cell it predicts its \emph{target cell}. We write $\usucc$ for this step
and call it the \emph{successor offset}; it is the $d$-th coordinate vector of
the torus $\Grid$.
A token therefore plays two roles. To make its \emph{current prediction}, it
must access the $k$ parents of its target cell $z+\usucc$, which lie at offsets
$D_K+\usucc$ from $z$. To serve as a \emph{precedent} for later predictions,
it must also represent its own $k$ parents, which lie at offsets $D_K$.
We denote the former set by $D_Q$.

\begin{definition}[Routing set]
\label{def:routing}
For a query cell $(t,z)$, with target cell $z+\usucc$, the \emph{query offsets}
are $D_Q \coloneqq D_K+\usucc$, and the \emph{routing set} is
$
    \Uset
    \;\coloneqq\;
    D_K \cup D_Q
    \;=\;
    D_K \cup (D_K+\usucc)
    \;\subseteq\; \Grid 
$.
\end{definition}

\begin{exmp}
On grids large enough to avoid wraparound, the elementary $1$D neighborhood has $D_K = \{-1,0,1\}$ and $k = 3$, so
$D_Q = \{0,1,2\}$ and
$\Uset = \{-1,0,1,2\}$, with $\lvert\Uset\rvert = 4$.
The $2$D von Neumann neighborhood has $k = 5$ and $\lvert\Uset\rvert = 8$;
the $2$D Moore neighborhood has $k = 9$ and
$\Uset = \{-1,0,1\}\times\{-1,0,1,2\}$,
with $\lvert\Uset\rvert = 12$. 
\end{exmp}

Generally,
$k \leq \lvert\Uset\rvert \leq \min(2k, \lvert\Grid\rvert)$.
Figure~\ref{fig:gridattn} shows both offset sets in a trained model.
In these examples $D_K$ is symmetric but $\Uset$ is not, because $D_Q$ is
shifted by $\usucc$. Along the innermost axis, $\Uset$ extends one step further
in the positive direction than $D_K$, but is the same along other axes.
\subsection{The model}
\label{sec:task:model}

We employ a two-layer causal decoder-only transformer
\citep{vaswani2017attention} with learned absolute positional embeddings.
Position $n$ of the token stream has residual-stream vector $x^{(0)}_n = W_E e_n + W_P p_n \in \mathbb{R}^{\dmodel}$, where $e_n$ is the one-hot of the token at position $n$ and $W_P$ embeds the learned absolute position $p_n \in \mathbb{R}^{\dpos}$ (column $n$ of $P \in \mathbb{R}^{\dpos \times N}$) into the residual stream. Attention weights
\[
  \alpha^{(l)}_{n,n'} \;=\; \operatorname{softmax}_{n' \leq n}\!\Bigl(\bigl\langle W^{(l)}_Q x^{(l-1)}_n,\; W^{(l)}_K x^{(l-1)}_{n'} \bigr\rangle\Bigr).
\]
at layer $l \in \{1,2\}$ determine the output $\sum_{n' \leq n} \alpha^{(l)}_{n,n'} W^{(l)}_V x^{(l-1)}_{n'}$ which is projected back into the stream (optionally followed by a two-layer ReLU MLP of hidden width $\dmlp$). 
For simplicity of exposition, we absorb scaling factors into $W_Q$ and $W_K$.

Our analysis focuses on the learned positional table $P$, 
such as the one used in the original
Vision Transformer \citep{dosovitskiy2020image}. Its columns are free
parameters and supply no geometry of their own, so any spatial structure in a
trained model must have been inferred from data.
Appendix~\ref{app:serialization:seqrpe} 
compares alternative positional embedding schemes.

\subsection{Resolvable queries and the Bayes target}
\label{sec:task:correct}

The kernel is drawn from the Dirichlet prior of Definition~\ref{def:ca}, so the
best possible prediction at a token is based on the posterior, given
everything earlier in the stream. It has a closed form: count the earlier
occurrences of the query configuration, and how often each state followed them. Denote
\begin{figure}[t]
\centering
\includegraphics[width=\linewidth]
{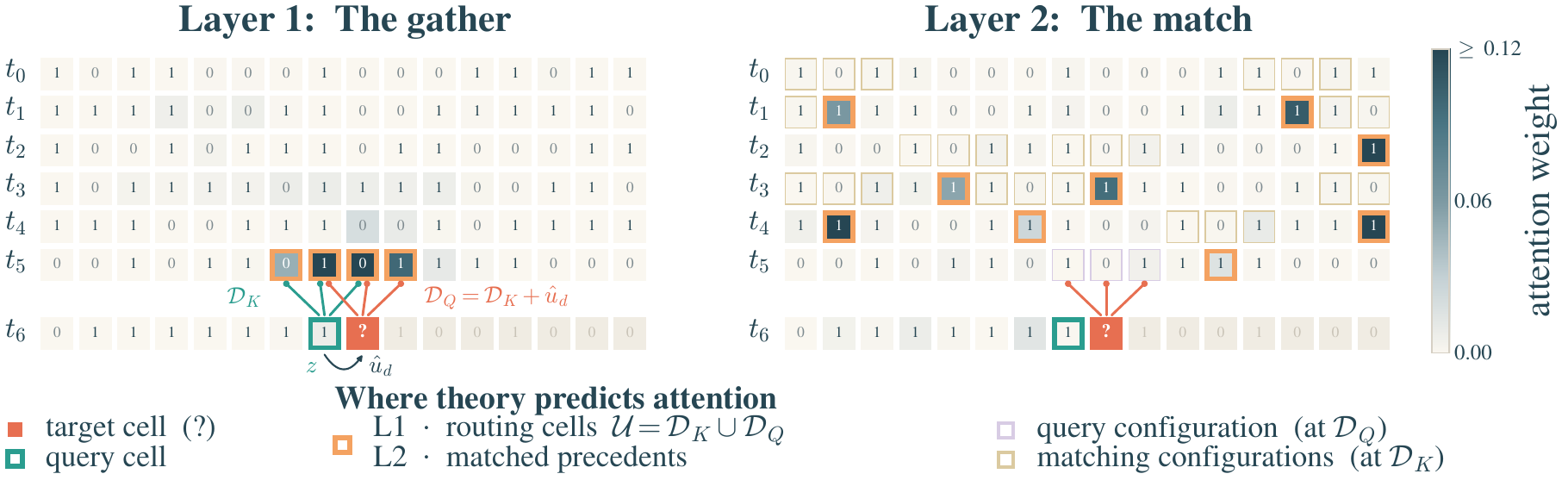}
\caption{\textbf{Layer~1 gathers the routing cells; Layer~2 retrieves matching precedents.}
Attention weights from the shared-head model on 1D $V{=}2, k{=}3$ (elementary cellular automata), Wolfram Rule~110.
The query cell at $z$ predicts the target cell at $z+\hat u_d$ (marked ``?'').
Layer~1 (left) concentrates on the four routing cells in the previous timestep, at offsets $\Uset=D_K\cup D_Q$ relative to $z$, where $D_Q=D_K+\hat u_d$ gives the target cell's parent offsets.
Layer~2 (right) concentrates on matching precedents: earlier cells whose parent configurations (at offsets $D_K$) match the query configuration (at offsets $D_Q$). Their states are used to predict the target cell.
}
\label{fig:gridattn}
\end{figure}

\[
  n_c \coloneqq \#\{(t',z') : t' \geq 1,\ \rho(t',z') \leq \rho(t,z),\ P_{t',z'} = c\}, \qquad
  n_{c,y} \coloneqq \#\{\text{those with } x_{t',z'} = y\}.
\]

\begin{theorem}[Bayes posterior predictive]
\label{thm:bayes}
For the Dirichlet prior and serialization (\S\ref{sec:task:routing}), the optimal predictive distribution at a token with query configuration $c$, given all earlier observations, is
\[
  \widehat{P}_\alpha(y \mid c) \;=\; \frac{n_{c,y} + \alpha}{n_c + V\alpha},\quad y \in \States.
\]
\end{theorem}
The proof and further discussion are provided in Appendix~\ref{app:bayes}. 
\begin{definition}[Resolvable query, coverage]
  \label{def:resolvable}
  Configuration $c \in \States^{k}$ is \emph{covered} at $(t,z)$ if $P_{t',z'} = c$ for
  some $(t',z')$
  satisfying $\rho(t',z') \leq \rho(t,z)$, i.e. if $n_c>0$.
  A query cell $(t,z)$ defines the \emph{query configuration} $Q_{t,z} \coloneqq
  P_{t,\,z+\usucc}$, the parent configuration of its target cell,
  which determines the target state distribution $f(\cdot \mid Q_{t,z})$.
  A query is \emph{resolvable} if $Q_{t,z}$ is covered
  (Appendix~\ref{app:bayes:coverage}).
  \end{definition}

If a query
configuration $c$ has not occurred earlier in the stream, then $n_c = 0$ and
$\widehat{P}_\alpha(\cdot \mid c)$ reduces to the prior mean $1/V$, since there are no
observations on $f(\cdot \mid c)$. 
In other words, the context only determines the target distribution if resolvability holds.
In Section~\ref{sec:sih}, we focus on resolvable queries.

\section{Spatial Induction Heads: Construction and Optimality}
\label{sec:sih}

\subsection{The spatial induction head}
\label{sec:sih:def}

For spatial data, the model must construct the
relevant neighborhood at the query cell before matching. We call the resulting two-stage
circuit a \emph{spatial induction head}.

\begin{definition}[Spatial induction head]
\label{def:sih}
A two-layer circuit is a \emph{spatial induction head}, if at every resolvable query, it predicts in two stages:
\begin{enumerate}
[label=(\roman*),noitemsep,topsep=0pt,parsep=0pt,partopsep=0pt,leftmargin=2.2em]
  \item \emph{The gather.} Layer 1 assembles the spatial neighborhood
  information needed at the query cell.
  \item \emph{The match.} Layer 2 compares the resulting configuration
  against earlier configurations in the context and averages over matching
  precedents to predict the next state.
\end{enumerate}
\end{definition}

In our cellular-automaton setting, the gather makes the states of all
$\lvert\Uset\rvert$ routing cells, at offsets $\Uset = D_K \cup D_Q$, available at
the query cell. These states serve two roles in Layer~2:
the query reads the offsets $D_Q$, which form the query configuration, the
parent configuration of the target cell, while the key at an earlier token
reads the offsets $D_K$, which form that token's own parent configuration.
Layer~2 then performs the induction-style matching operation: the query
configuration is compared against earlier parent configurations. Those that
agree with it are the \emph{matching configurations}, the cells carrying them
are the \emph{matching precedents}, and the states that followed those
precedents are aggregated to form the prediction. 
Figure~\ref{fig:gridattn} shows the gather and match stages in a trained
model.

\subsection{Two constructions of the gather}
\label{sec:sih:constructions}

A standard transformer can realize the gather in two ways: it can use one
attention head per routed offset, so that attention itself keeps the
gathered states separate, or it can gather all offsets with a single head
and use an MLP to separate them afterward. Both constructions rely on positional addressing to locate the required cells using learned absolute positional embeddings.

\begin{lemma}[Addressing]
\label{thm:addressing}
For any nonempty target offset set $\mathcal N\subseteq \Grid$, a Layer-1
positional score which ranks the previous-step cells at offsets $\mathcal N$
from the query above all other visible cells can be realized with
$d_p=d_S+2$, where the spatial addressing width is
$d_S=O(|\mathcal N|\,d\log|\mathcal N|)$ in general or $d_S=2d$ when $|\mathcal N|=1$.
The positional dimension is therefore independent of $\lvert \Grid \rvert$ and $T$.
\end{lemma}
The two gather constructions use this addressing primitive differently:
Construction~I applies it separately to each singleton offset
$\{\delta\}$, while Construction~II applies it once to the full routing
set $\Uset$.

\begin{theorem}[Spatial gather constructions]
  \label{thm:gather}
    The Layer-1 gather of $\Uset$ has two realizations, both at widths
    independent of $\lvert\Grid\rvert$ and $T$:
    \begin{enumerate}[label=(\Roman*),leftmargin=2.2em,noitemsep,topsep=0pt,parsep=0pt,partopsep=0pt]
      \item \textbf{Dedicated-heads construction.} \label{thm:parallel}
      $\lvert\Uset\rvert$ heads, one per offset, and no Layer-1 MLP, with
      $\dpos = 2d+3$ and $\dmodel = (\lvert\Uset\rvert+2)V + 2d+6$.
      It recovers each routed state to any accuracy.
      \item \textbf{Shared-head construction.} \label{thm:onehead}
      One head and an MLP which separates the routed states by offset, with
      $\dmlp \le 2(V-1)\lvert\Uset\rvert^{3} + 3$ and
      $\dpos \le O(\lvert\Uset\rvert\,d\log\lvert\Uset\rvert) + \lvert\Uset\rvert^{3} + 3$.
      It recovers them exactly.
    \end{enumerate}
  \end{theorem}

A related head/nonlinearity tradeoff appears for sequential $k$-gram induction
\citep{ekbote2025one}; here, the additional difficulty is preserving the
identities of spatially separated offsets after routing.

\paragraph{Proof sketch.}
In Construction~I, head $\delta$ spatially addresses the cell at offset
$\delta\in\Uset$ and writes its state directly into a dedicated
$V$-dimensional block of the residual stream. Since the heads write into
disjoint blocks, the offsets are already separated after attention and no
Layer-1 MLP is needed.
Construction~II uses nonlinear capacity instead of separate heads. A single
head attends jointly to the cells at all offsets in $\Uset$. The cell state is encoded in the attention weight, while the value carries a positional label identifying the routing cell. A single global labeling guarantees these $\lvert\Uset\rvert$ routing cells remain mutually distinguishable at any query.
Finally, the Layer-1 MLP uses the query's positional information to decode the aggregate into one state block per offset. 
The Fourier construction underlying the addressing bound,
the positional labeling scheme for the
  shared-head construction,
  and the exact
  transformer parameter constructions are given in Appendices~\ref{app:addressing},~\ref{app:colouring}, and~\ref{app:constructions}.

\subsection{Matching and Bayes-optimal prediction}
\label{sec:sih:matching}

The gathered routing set feeds into
a Layer-2 matching head. The query selects the gathered blocks
indexed by $D_Q$, while each earlier key selects the blocks indexed by
$D_K$. Their inner product, scaled by $\beta>0$, measures agreement between the
query configuration and an earlier parent configuration.

\begin{proposition}[Hard-attention matching as Bayesian counting]
  \label{thm:sink}
  Suppose a single Layer-2 attention head assigns the same score $a$ to every
  precedent whose parent configuration matches the current query, and $-\infty$ to
  all other positions. With no Layer-2 MLP, the head outputs the empirical
  conditional transition law $(n_{c,y}/n_c)_{y\in\States}$. Augmenting the context
  with a single \emph{prior token}, of value $\tfrac{1}{V}\mathbf 1$ and score
  $a+\log(V\alpha)$, makes the output exactly
  $\widehat P_\alpha(y\mid c)
    =
    \frac{n_{c,y}+\alpha}{n_c+V\alpha}$, the 
  posterior of Theorem~\ref{thm:bayes}.
  \end{proposition}

\paragraph{Proof sketch.}
For one-hot state blocks, the query--key inner product counts matching
offsets, so an exact configuration match attains the maximal score. If all
matches receive the same score and nonmatches are excluded, softmax is
uniform over the $n_c$ matching precedents. Since each value carries the
state that followed its configuration, their average is exactly the
empirical conditional transition law. The prior token contributes total
unnormalized weight $V\alpha$ with uniform value
$\frac{1}{V}\mathbf 1$, adding $\alpha$ to every state count and recovering
the posterior of Theorem~\ref{thm:bayes}.
\hfill \qed

A finite softmax need not realize this ideal matching exactly: nonmatching
positions retain some attention mass, and approximation error from Layer~1
can make the scores of matching positions slightly unequal. The
end-to-end result below controls both effects.

\begin{theorem}[End-to-end spatial induction; formally Theorem~\ref{thm:main-app}]
  \label{thm:main}
  Let the Layer-1 gather construction have per-block error $\bar\epsilon$, let
  $\beta$ denote the Layer-2 query--key scale, and write
  $\Delta \coloneqq \beta\bigl(1-(2k+2)\bar\epsilon\bigr)$ for the resulting
  match--nonmatch separation. Suppose the Layer-2 head also discounts
  separators and initial-step cells. At every resolvable query with $n_c$
  matching precedents:
  \begin{enumerate}[label=(\roman*),leftmargin=2.2em,itemsep=1pt,topsep=2pt]
    \item \emph{Stochastic rules.} Fix $\kappa\in(0,1]$, assume
    $N\ge V\alpha$, and include the prior token of
    Proposition~\ref{thm:sink}. If $\Delta \ge \log\bigl(4N/((n_c+V\alpha)\kappa)\bigr)$ and
    $\bar\epsilon \le \kappa/(16k\Delta)$, then the Layer-2 attention output
    $p$, read directly as a distribution, satisfies
    $\KL\bigl(\widehat P_\alpha(\cdot\mid c)\,\big\|\,p\bigr) \le \kappa$.

    \item \emph{Deterministic rules.} In the deterministic limit $\alpha\to0$, if
    $\bar\epsilon < 1/(2k+2)$ and $\Delta > \log\bigl(2\sqrt{2}\,(N-n_c)/n_c\bigr)$,
    then the predicted next state is exactly correct.
  \end{enumerate}
  \end{theorem}

The two cases place different demands on the matching scores. For a
deterministic rule, every matching precedent has the same next state, so a
sufficient gap between matches and nonmatches guarantees the correct
prediction. For a stochastic rule, the relative weights \emph{among}
matching precedents also matter: their uniform average is what recovers the
transition counts. Layer-1 error perturbs these matched scores by at most
$2\beta k\bar\epsilon$, so $\beta$ must be large enough to suppress
nonmatches without amplifying this dispersion too much. The two conditions
in Theorem~\ref{thm:main} balance these effects.

\section{Experiments}
\label{sec:experiments}

\begin{table}[!b]
\vspace{-5mm}
\centering
\small
\caption{\textbf{Test-set prediction accuracy on held-out rules.}
Deterministic metrics use the ground-truth history (teacher forcing, TF) or greedy autoregressive (AR) rollout.
Stochastic AR cell accuracy is reported for the model and the $k$-gram sampling baseline.
Stochastic AR KL is $D_{\mathrm{KL}}(q_{\text{kgram}} \| q_{\text{model}})$ between the $k$-gram and model predictive distributions on the model's rollout.
}
\begin{tabular}{@{}ll cc cc@{}}
\toprule
& & \multicolumn{2}{c}{Deterministic} & \multicolumn{2}{c}{Stochastic} \\
\cmidrule(lr){3-4} \cmidrule(lr){5-6}
Task & Arch.
  & TF cell acc & AR cell acc
  & AR cell acc ($k$-gram) & AR KL \\
\midrule
1D $V{=}2,\;k{=}3$
  & Shared-head & 1.0000 & 0.9999 & 0.5496\;(0.5485) & 0.0019 \\
  & Dedicated-heads & 1.0000 & 1.0000 & 0.5488\;(0.5485) & 0.0044 \\
\midrule
1D $V{=}3,\;k{=}3$
  & Shared-head & 1.0000 & 1.0000 & 0.3476\;(0.3475) & 0.0026 \\
  & Dedicated-heads & 1.0000 & 1.0000 & 0.3475\;(0.3475) & 0.0004 \\
\midrule
2D VN $V{=}2,\;k{=}5$
  & Shared-head & 1.0000 & 1.0000 & 0.5201\;(0.5197) & 0.0013 \\
  & Dedicated-heads & 1.0000 & 0.9999 & 0.5191\;(0.5197) & 0.0038 \\
\bottomrule

\end{tabular}

\label{tab:results}
\vspace{-3mm}
\end{table}

\subsection{Experimental setup}
\label{sec:setup}

We train two layer transformers, then evaluate them on rules held out from training to measure their ability for spatial in-context learning. We also examine whether the learned internal representations are consistent with the two-stage structure of the theoretical constructions.
We use three settings: 1D $V{=}2, k{=}3$ (elementary cellular automata), 1D $V{=}3, k{=}3$, and 2D von Neumann (VN) $V{=}2, k{=}5$.
The model receives only the first $T_0$ time steps of a trajectory generated under an unknown rule (not the rule table), and its weights are frozen at test time.
Predictions are scored on cells from time step $t = T_0$ through $t = T{-}1$.
Every evaluation trajectory is rejection-sampled so that all $V^k$ neighborhood configurations appear in the $T_0$ context rows, so every scored query is resolvable.
Full details of data generation, training, and evaluation are in Appendix~\ref{app:expdetails}.

\textbf{Rules.}
(1) \emph{Deterministic}: each trajectory draws one rule uniformly from the corresponding split's rule pool; training, validation, and test pools are mutually disjoint, so every val or test trajectory is generated by a rule the model has never seen during training.
(2) \emph{Stochastic}: each trajectory draws a fresh rule by independently sampling a $\text{Dirichlet}(1, \ldots, 1)$ transition row for every neighborhood configuration; because the Dirichlet is continuous, no two trajectories share the same rule.

\textbf{Architectures.}
We instantiate the two gather constructions of Theorem~\ref{thm:gather}: the \emph{Dedicated-heads} construction ($|\mathcal{U}|$ L1 attention heads with no L1 MLP) and the \emph{Shared-head} construction (one L1 attention head with L1 MLP).
Both share a single L2 attention head and a two-layer L2 MLP.

\textbf{Metrics.}
For deterministic rules, we report \emph{TF cell accuracy} (each predicted cell conditioned on the ground-truth history) and \emph{AR cell accuracy} (greedy autoregressive rollout scored against the ground-truth trajectory).
For stochastic rules, \emph{AR cell accuracy} is the fraction of cells in a temperature-1 rollout that match the ground truth. The ceiling is well below~1;
we also report a posterior-predictive sampling baseline, an independent rollout which samples each cell from the add-one-smoothed $k$-gram posterior predictive (its mode maximizes expected one-step accuracy).
We additionally report \emph{AR $D_{\mathrm{KL}}(q_{\text{kgram}} \| q_{\text{model}})$}, computed at each cell during autoregressive generation between the model's predictive distribution and the causal $k$-gram distribution on the model's own rollout.

\subsection{Results}
\label{sec:results} \label{sec:mechanism}

\textbf{Can the model recover unseen rules?} Table~\ref{tab:results} reports test-set performance across three settings and both architectures.
On deterministic rules, all models reach near-perfect cell accuracy under both teacher forcing and autoregressive rollout.
On stochastic rules, both architectures match the $k$-gram baseline in AR cell accuracy, and all AR KL values fall below $0.005$ nats.
Both architectures thus perform accurate in-context prediction for previously unseen rules across all three settings.

\begin{figure}[!b]

\centering
\includegraphics[width=\linewidth]{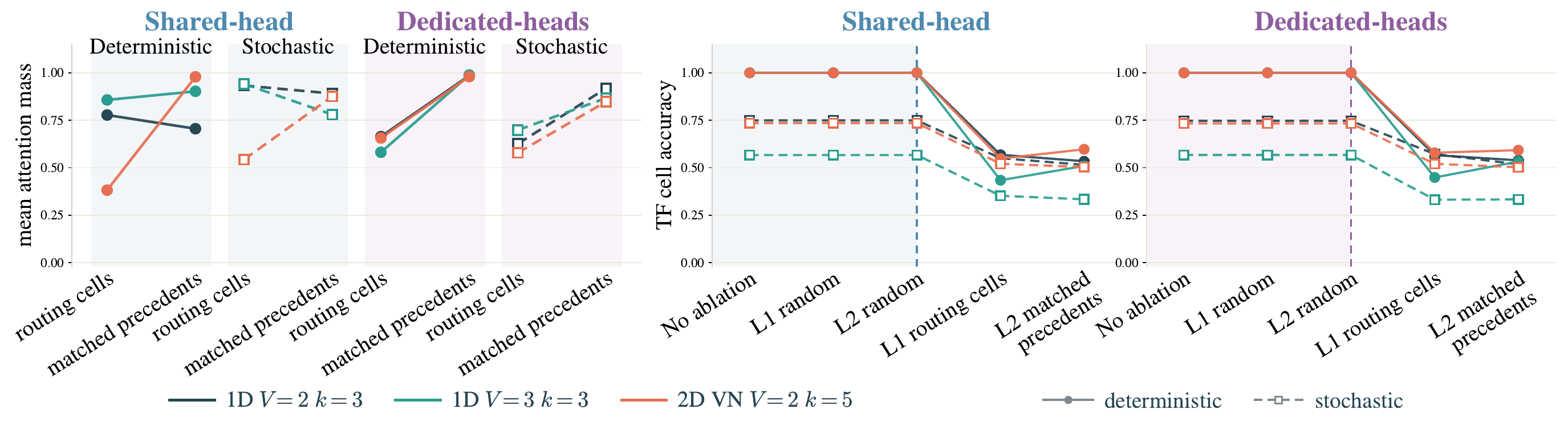}
\caption{\textbf{Attention alignment and ablation.}
\emph{Left:} Mean attention mass on the positions predicted by each construction: L1 routing cells at offsets $\mathcal{U}=D_K\cup D_Q$, and L2 matching precedents, earlier cells whose parent configurations match the query configuration.
\emph{Right:} Cell accuracy when attention to the L1 routing cells or L2 matching precedents is masked at evaluation time, compared with no ablation and controls that mask attention to the same number of randomly selected cells.}
\label{fig:attn_mass}
\end{figure}

\textbf{Do attention heads attend where each construction predicts?}
We next ask whether the model recovers the rule using the constructed two stage gather-and-match. 
We first examine the attention mass.
Figure~\ref{fig:attn_mass} (left) reports the mean attention mass on the construction-predicted positions.
Both architectures place substantial mass on these positions across all three settings.
To further test whether internal model computations rely on these positions, we examine what happens when we mask attention to these cells.
For both Shared-head and Dedicated-heads models, we mask L1 attention to routing cells or L2 attention to matching precedents at evaluation time, using an equal number of randomly selected cells as controls (Figure~\ref{fig:attn_mass}, right).
Across all three settings, masking attention to the theory-predicted cells reduces cell accuracy, while the random controls have little effect.
This supports the role of both the routing and matching stages in the trained models' predictions.

\textbf{Do trained models exhibit the predicted layerwise structure?}
We train linear probes on frozen post-residual representations after embedding, L1 attention, L1 MLP, L2 attention, and L2 MLP. For two targets: (1) the \emph{neighborhood configuration}, the $k$-cell neighborhood that determines the next cell value; (2) the \emph{$k$-gram predictive distribution}, the conditional distribution over the next cell value computed from causal $k$-gram counts in the trajectory prefix with add-one smoothing (the probe minimizes $D_{\mathrm{KL}}(q_{\text{kgram}} \| q_{\text{probe}})$).
Probe training, validation, and test rules are disjoint.
The constructions make the configuration linearly readable after the L1 MLP in Shared-head models and after L1 attention in Dedicated-heads models, and the predictive distribution after L2 attention in both via the simplex readout.
Our logistic probes are an empirical diagnostic of this structure, and both architectures match it across all three settings (Figure~\ref{fig:probes}). Figure~\ref{fig:tsne} visualizes the layerwise structure identified by probes, using t-SNE on the Shared-head model on deterministic 1D $V{=}2, k{=}3$.
After L1 MLP, representations cluster by neighborhood configuration; after L2 attention, by output value.

\begin{figure}[t]
\centering
\includegraphics[width=\linewidth]{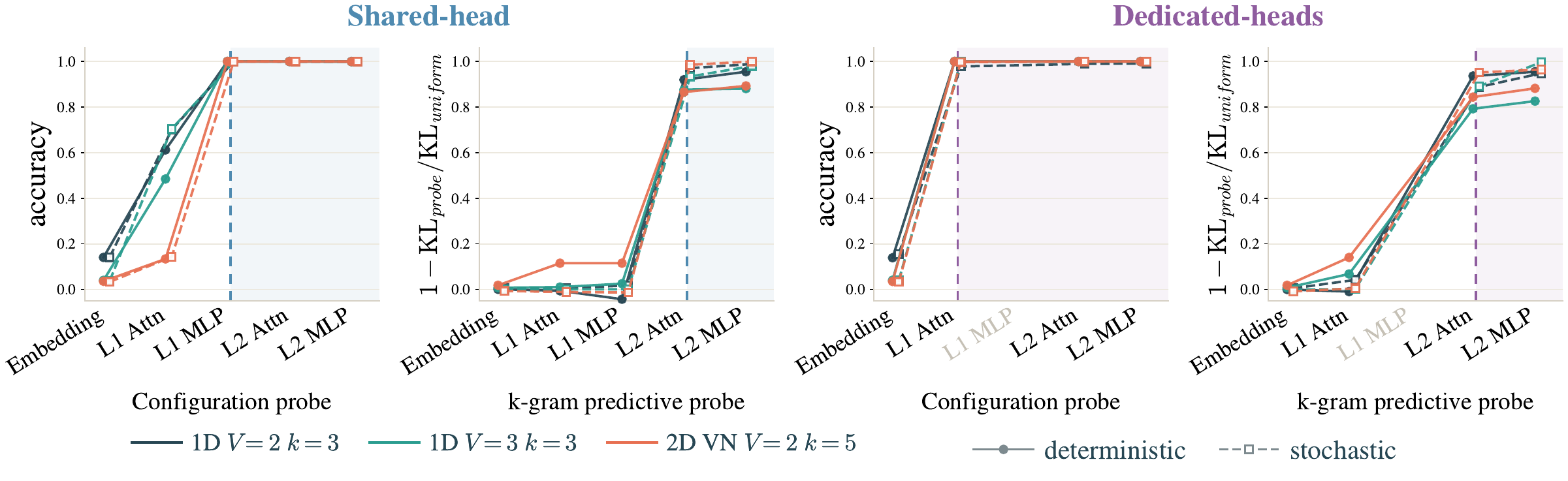}
\caption{\textbf{Linear probe results.} Higher is better.
The configuration probe measures classification accuracy for the $k$-cell neighborhood; the k-gram predictive probe measures $1 - \overline{D_{\mathrm{KL}}(q_{\text{kgram}} \| q_{\text{probe}})} \,/\, D_{\mathrm{KL}}(q_{\text{kgram}} \| \text{Uniform})$, which the axis label abbreviates as $1-\mathrm{KL}_{\mathrm{probe}}/\mathrm{KL}_{\mathrm{uniform}}$.
Solid lines: deterministic; dashed: stochastic.
Vertical lines mark the stage at which the theoretical construction places each target.
}
\vspace{-6mm}
\label{fig:probes}
\end{figure}

\begin{figure}[b]
\vspace{-4mm}
\centering
\includegraphics[width=0.9\linewidth]{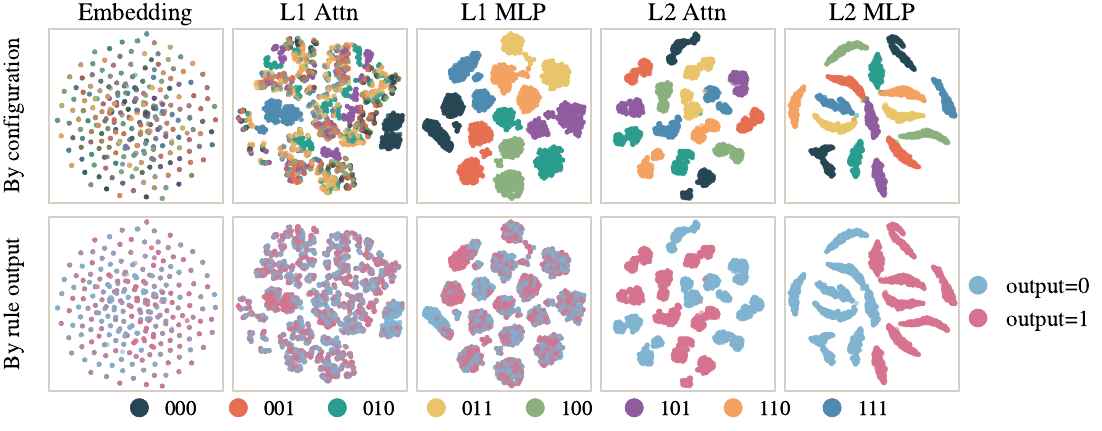}
\caption{\textbf{t-SNE visualization of frozen representations} from the Shared-head model on deterministic 1D $V{=}2, k{=}3$.
Top row (colored by neighborhood configuration): configurations cluster after the L1 MLP.
Bottom row (colored by output value): outputs cluster after L2 attention.
}
\label{fig:tsne}
\end{figure}

\section{Conclusion}
\label{sec:conclusion}

We studied how transformers perform in-context learning when the dependencies
required for prediction are spatial. We introduced the \emph{spatial induction
head}, a two-layer gather-and-match circuit in which the first layer assembles
the relevant spatial context from a serialized sequence and the second matches
that context against precedent to predict the next state. We gave two explicit
realizations of the gather and showed that the positional width required for
spatial routing depends only on the local neighborhood and spatial dimension,
not on the total grid volume or trajectory horizon. We further showed that the matching computation realizes Bayesian counting: the end-to-end circuit can approach the Bayes posterior predictive
arbitrarily closely for stochastic rules and predicts exactly for deterministic
rules. Experiments on 1D and 2D cellular automata demonstrate generalization to
held-out rules and provide mechanistic evidence consistent with the gather-and-match computation.

Our analysis focuses on cellular automata, whose finite-state local dynamics
provide a particularly tractable setting for isolating spatial induction. 
More broadly, by separating spatial routing
from precedent matching, our construction provides a framework for
understanding induction mechanisms in more general types of data with spatial
structure.

\subsubsection*{Acknowledgments}
We thank Kilian Weinberger for his support. S.D. was partly supported by NSF CCF-2312774, NSF OAC-2311521, NSF IIS-2442137, and an AI2050 Early Career Fellowship program at Schmidt Sciences. This work was also supported by the National Science Foundation (NSF) through the AI Research Institutes program Award No.~\mbox{DMR-2433348}.

\bibliographystyle{iclr2027_conference}
\bibliography{references}

\appendix
\numberwithin{theorem}{section}
\setcounter{theorem}{0}
\addtocontents{toc}{\protect\setcounter{tocdepth}{2}}
\tableofcontents
\section{Roadmap and reading paths}
\label{app:roadmap}

The appendices are organized thematically. Part~I formalizes the task, detailing grid serialization, prediction correctness, and the limitations of standard positional schemes. Part~II establishes grid-based mathematical foundations, yielding two interface corollaries. Parts~III and~IV describe Layers~1 and~2 of the architecture, respectively. Part~V outlines experimental methods.

\begin{itemize}[leftmargin=1.6em,itemsep=2pt,topsep=3pt]
  \item \textbf{Notation} (Appendix~\ref{app:notation}). Symbol reference and running example.
  \item \textbf{Serialization and limits of sequence-relative routing} (Appendix~\ref{app:serialization}), \emph{Part~I}. Formalizes flattening (Section~\ref{sec:task:routing}) and separator functions. Proves that sequence-relative biases (Definition~\ref{def:seqrpe}) cannot route local neighborhoods on a torus in a single layer, motivating the spatial addressing in Part~II.
  \item \textbf{Coverage and Bayes-optimal prediction} (Appendix~\ref{app:bayes}), \emph{Part~I}. Defines resolvable queries and the optimal predictor for empirical accuracy.
  \item \textbf{Spatial addressing with grid-independent width} (Appendix~\ref{app:addressing}), \emph{Part~II}. Demonstrates that $O(\lvert\Nset\rvert d\log\lvert\Nset\rvert)$ coordinates suffice to score routing cells at offsets $\Uset$ above other visible keys, independent of grid size.
  \item \textbf{Colorings for Shared-head decoding} (Appendix~\ref{app:colouring}), \emph{Part~II}. Formulates an injective grid labeling scheme for translates $z+\Uset$, corresponding to a Cayley graph proper coloring.
  \item \textbf{Layer 1: spatial gather constructions} (Appendix~\ref{app:constructions}), \emph{Part~III}. Details the forward pass, residual stream, and weights. Gathers are performed using either $\lvert\Uset\rvert$ Dedicated-heads or a single Shared-head with a separating MLP.
  \item \textbf{Layer 2: Bayesian matching and end-to-end guarantees} (Appendix~\ref{app:synthesis}), \emph{Part~IV}. Constructs the matching mechanism. Shows that softmax over an equally scored set implements a counter, yielding the Bayes predictive distribution. Bounds the required separation $\Delta$ against Layer~1's dispersion error $\eta$.
  \item \textbf{Experimental methods and additional results} (Appendix~\ref{app:expdetails}), \emph{Part~V}. Details data generation, training, and evaluation protocols, including complete metric tables and mechanistic analyses.
\end{itemize}

\begin{table}[t]
\centering
\caption{Dependencies between construction stages.}
\small
\begin{tabular}{@{}lp{0.60\linewidth}@{}}
\toprule
Stage result & Downstream usage \\
\midrule
Corollary~\ref{cor:head} & Positional features for offset selection in previous time steps \\
Corollary~\ref{cor:colour-interface} & Distinct color slots and storage costs (Construction~II) \\
Corollary~\ref{cor:gather-interface} & State blocks (approximate or exact) per routed offset \\
Theorem~\ref{thm:bayesmatch} & Equal-weight matching for empirical counts and Dirichlet smoothing \\
Lemmas~\ref{lem:noise}, \ref{lem:dispersion}, \ref{lem:biasbudget} & Bounds on separation, dispersion, and predictive cost \\
Theorem~\ref{thm:main-app} & Finite-parameter prediction guarantee \\
\bottomrule
\end{tabular}
\label{tab:dependencies}
\end{table}

\paragraph{Reading paths.}
The shortest sequence comprising the main result involves six statements. Read the Appendix~\ref{app:constructions} introduction, followed by the two gather theorems (Theorems~\ref{thm:parallel-app} and~\ref{thm:onehead-app}) and their Layer~2 interface (Corollary~\ref{cor:gather-interface}). Continue to the Appendix~\ref{app:synthesis} introduction, the counting identity (Theorem~\ref{thm:bayesmatch}), and the end-to-end guarantee (Theorem~\ref{thm:main-app}). Two prerequisites are referenced on demand: Corollary~\ref{cor:head} (positional addressing) and Corollary~\ref{cor:colour-interface} (colorings for Theorem~\ref{thm:onehead-app}).

Suggested reading paths by topic:

\begin{itemize}[leftmargin=1.6em,itemsep=2pt,topsep=3pt]
  \item \emph{Circuit construction:} Appendices~\ref{app:constructions} and~\ref{app:synthesis} (Parts~III and~IV).
  \item \emph{Mathematical foundations:} Part~II.
  \item \emph{Formal problem setting:} Part~I.
  \item \emph{Symbol reference:} Appendix~\ref{app:notation}.
\end{itemize}
\section{Notation and conventions}
\label{app:notation}

This section catalogs symbols defined throughout the text; Table~\ref{tab:notation} provides their definitions and points of introduction. Section~\ref{app:not:collisions} clarifies overloaded symbols, and Section~\ref{app:not:instance} instantiates task parameters for the three experimental settings.

\subsection{Conventions}
\label{app:not:conventions}

Coordinates, states, time steps, sequence positions, and residual blocks ($B_0,\dots,B_4$) are $0$-indexed; dimensions, offsets, and heads are $1$-indexed. Let $[n] \coloneqq \{0,\dots,n-1\}$. Grid coordinate arithmetic is modulo the respective side length, making $\Grid$ a torus such that $z + \delta$ always identifies a valid cell.

Grid side lengths are denoted $W_1,\dots,W_d$, and its \emph{volume} is $\lvert\Grid\rvert = \prod_j W_j$.

A \emph{row} designates an innermost-axis grid row, comprising the $W_d$ cells followed by a \texttt{ROW\_SEP}. A \emph{time step} $t$ corresponds to one automaton layer; a query at step $t$ accesses step $t-1$. The term \emph{row} strictly retains its spatial meaning.

Vectors are column vectors. The notation $\lvert\cdot\rvert$ denotes set cardinality or the distance to $0$ for group elements. The one-hot vector for coordinate $v$ is $e_v$; for a state $x$, this is written $e(x)$ to distinguish it from the dual-group generators $e_1,\dots,e_d$.

\subsection{Symbols}
\label{app:not:symbols}

\begin{longtable}{@{}l@{\hspace{1.2em}}p{0.46\linewidth}@{\hspace{1.2em}}p{0.24\linewidth}@{}}
\caption{Symbols, and where each is fixed.}
\label{tab:notation}\\
\toprule
symbol & meaning & fixed at \\
\midrule
\endfirsthead
\multicolumn{3}{@{}l}{\small\itshape Table~\ref{tab:notation}, continued}\\
\toprule
symbol & meaning & fixed at \\
\midrule
\endhead
\midrule
\multicolumn{3}{r@{}}{\small\itshape continued on the next page}\\
\endfoot
\bottomrule
\endlastfoot

\multicolumn{3}{@{}l}{\emph{Grid and dual group}}\\
$d$ & number of spatial dimensions & Def.~\ref{def:ca} \\
$W_j$ & side length of axis $j$ & Def.~\ref{def:ca} \\
$\Wmax$ & largest side length, $\max_j W_j$ & Lem.~\ref{lem:degreeone} \\
$\Grid$, $\lvert\Grid\rvert$ & the torus $[W_1]\times\cdots\times[W_d]$; its
  volume $\prod_j W_j$ & Def.~\ref{def:ca} \\
$z$, $z_j$, $\delta$ & a cell of $\Grid$; its $j$-th coordinate; an offset &
  Def.~\ref{def:ca} \\
$\Gd$ & dual group, the characters of $\Grid$ & App.~\ref{app:addr:fourier} \\
$\omega$, $e_j$ & a character; the $j$-th generator of $\Gd$ &
  App.~\ref{app:addr:fourier} \\
$\varphi_\omega$, $\chi_\omega$ & phase $2\pi\sum_j\omega_j\delta_j/W_j$; the
  character $e^{i\varphi_\omega}$ & App.~\ref{app:addr:fourier} \\
$\supp\hat g$ & spectrum of $g$, the $\omega$ with $\hat g(\omega)\neq0$ &
  App.~\ref{app:addr:fourier} \\
\addlinespace
\multicolumn{3}{@{}l}{\emph{States, rules and the prior}}\\
$\States$, $V$ & state set $[V]$, and its size & Def.~\ref{def:ca} \\
$T$, $t$ & horizon; a time step & Def.~\ref{def:ca} \\
$T_0$ & context length, the number of time steps shown before scoring
  begins & \S\ref{sec:experiments} \\
$(t,z)$ & the cell $z$ at step $t$ & Def.~\ref{def:ca} \\
$x_{t,z}$ & state of cell $z$ at step $t$ & Def.~\ref{def:ca} \\
$f$ & the local kernel $\States^k\to\Delta^{V-1}$ & Def.~\ref{def:ca} \\
$\alpha$ & Dirichlet concentration of the prior on $f$ & Def.~\ref{def:ca} \\
$P_{t,z}$ & parent configuration of $(t,z)$ & Def.~\ref{def:ca} \\
$Q_{t,z}$ & query configuration, $P_{t,z+\usucc}$ & Def.~\ref{def:resolvable} \\
$c$ & a \emph{matching configuration}: an earlier $P_{t',z'}$ equal to
  $Q_{t,z}$; its cell is a \emph{matching precedent} &
  Def.~\ref{def:resolvable} \\
$\widehat P_\alpha$ & Bayes posterior predictive & Thm.~\ref{thm:bayes} \\
\addlinespace
\multicolumn{3}{@{}l}{\emph{Neighborhood and routing}}\\
$(\delta_1,\dots,\delta_k$), $k$ & the neighborhood, a list of distinct offsets,
  and its size & Def.~\ref{def:ca} \\
$D_K$, $D_Q$ & key offsets; query offsets $D_K+\usucc$ &
  Def.~\ref{def:ca}, \ref{def:routing} \\
$\usucc$ & successor offset, the $d$-th coordinate vector &
  Def.~\ref{def:serialization} \\
$\Uset$ & routing set $D_K\cup D_Q$; the cells at these offsets are the
  \emph{routing cells} & Def.~\ref{def:routing} \\
$\Nset$ & a generic set of offsets to be isolated & Thm.~\ref{thm:isolation} \\
\addlinespace
\multicolumn{3}{@{}l}{\emph{Serialization}}\\
$\rho_\Grid$, $r(z)$ & row-major index of $z$; its row &
  Def.~\ref{def:serialization} \\
$\rho(t,z)$ & sequence position of the cell $(t,z)$ &
  Def.~\ref{def:serialization} \\
$N$ & stream length $T(\lvert\Grid\rvert+\lvert\Grid\rvert/W_d)$ &
  Def.~\ref{def:serialization} \\
$n$, $n'$ & a query position; a key position & \S\ref{sec:task:model} \\
$\tau(n)$, $\kappa(n)$ & virtual time step and virtual cell of position $n$ &
  App.~\ref{app:serialization:formal} \\
query cell & a cell token which poses a query & Def.~\ref{def:routing} \\
target cell & the cell $\kappa(n)+\usucc$ which $n$ predicts &
  Lem.~\ref{lem:successor} \\
\texttt{ROW\_SEP} & separator closing an innermost row &
  Def.~\ref{def:serialization} \\
\texttt{TIME\_SEP} & separator closing a time step &
  Def.~\ref{def:serialization} \\
\addlinespace
\multicolumn{3}{@{}l}{\emph{The model}}\\
$W_E$, $e_n$, $P$, $p_n$ & token embedding; one-hot of the token at $n$;
  positional table; its column $n$ & \S\ref{sec:task:model} \\
$W_Q^{(l)},W_K^{(l)},W_V^{(l)}$ & query, key and value projections at layer $l$
  & \S\ref{sec:task:model} \\
$g$, $\gamma$ & the positional kernel on $\Grid$; its isolation margin &
  Def.~\ref{def:isolating} \\
$\Lambda$ & Layer-1 logit scale & Cor.~\ref{cor:scale} \\
$\Gamma_1$ & Layer-1 separator gate strength &
  Prop.~\ref{prop:sepgate-l1} \\
$\Gamma_2$ & Layer-2 key-validity gate strength &
  Prop.~\ref{prop:keygate-l2} \\
$\nu_n$ & key-validity indicator $s_n+\zeta_n$ &
  Def.~\ref{def:l2match} \\
$\beta$ & Layer-2 query scaling & Def.~\ref{def:l2match} \\
\addlinespace
\multicolumn{3}{@{}l}{\emph{Widths}}\\
$\dmodel$ & residual-stream width & \S\ref{sec:task:model} \\
$\dpos$ & width of the positional table, $\dS+3$ for Construction~I
  & Def.~\ref{def:stream} \\
$\dS$ & addressing dimension, $\lvert\supp\hat g\setminus\{0\}\rvert$ &
  \S\ref{app:addr:fourier} \\
$\dmlp$ & hidden width of the Layer-1 MLP & \S\ref{sec:task:model} \\
\addlinespace
\multicolumn{3}{@{}l}{\emph{Residual blocks and projections}}\\
$B_0$ & separator indicator, auxiliary-token marker, and the constant $1$ & Def.~\ref{def:stream} \\
$\mu_n$ & auxiliary-token marker $\mathbf 1[n=\star]$, a coordinate of $B_0$
  & Def.~\ref{def:stream} \\
$\sigma_\star$ & Layer-2 score of the auxiliary token $\star$ &
  Def.~\ref{def:l2match} \\
$B_1$ & own state one-hot $e(x_{t,z})$ & Def.~\ref{def:stream} \\
$B_2$ & positional table entry $p_{t,z}$ & Def.~\ref{def:stream} \\
$\zeta_n$ & initial-step indicator $\mathbf 1[\tau(n)=0]$, a coordinate of $B_2$
  & Def.~\ref{def:stream} \\
$B_2^{\mathrm{col}}$ & own-color block of $B_2$, zero at separators
  & Def.~\ref{def:stream} \\
$B_3$ & pattern scratch, $B_3^{\mathrm{pat}} \oplus B_3^{\mathrm{agg}}$
  & Def.~\ref{def:stream} \\
$B_4$ & output slot, written by Layer 2 & Def.~\ref{def:stream} \\
$\Pi_{B_r}$ & projection onto block $B_r$ & Def.~\ref{def:stream} \\
$\iota_\delta$ & embedding of $\mathbb{R}^V$ into block $\delta$ of $B_3^{\mathrm{pat}}$ &
  Thm.~\ref{thm:parallel-app} \\
$\Sigma_Q$, $\Sigma_K$ & Layer-2 selection projections onto $\mathbb{R}^{kV}$ &
  Def.~\ref{def:l2match} \\
\addlinespace
\multicolumn{3}{@{}l}{\emph{Coloring}}\\
$c$ & a coloring $\Grid\to[M]$ & Def.~\ref{def:colouring} \\
$M$ & number of colors used & Def.~\ref{def:colouring} \\
$M^\star$ & chromatic number of $\Gamma_\Uset$ & Def.~\ref{def:colouring} \\
$\Gamma_\Uset$ & Cayley graph on $(\Uset-\Uset)\setminus\{0\}$ &
  Def.~\ref{def:colouring} \\
$\chi$, $K$, $M_K$ & periodic coloring; its subgroup; its color count
  $\lvert\Grid/K\rvert$ & Def.~\ref{def:colouring},
  Lem.~\ref{lem:periodic} \\
$n_{\max}$ & most visible cells sharing one color &
  Thm.~\ref{thm:onehead-app} \\
\addlinespace
\multicolumn{3}{@{}l}{\emph{Budgets, gaps and counts}}\\
$\eblock$ & per-block error, the routing budget &
  Def.~\ref{def:eclose}; discharged by Lem.~\ref{lem:blockerror} \\
$\eret$ & retrieval error at score gap $\Delta$ &
  Cor.~\ref{lem:det-retrieval} \\
$\Delta$ & Layer-2 separation $\beta(1-(2k+2)\eblock)$, the gap between
  matching and non-matching keys & Lem.~\ref{lem:noise} \\
$\eta$ & dispersion within the matched set, $2\beta k\eblock$ &
  Lem.~\ref{lem:dispersion} \\
$\kappa$ & divergence budget of the main theorem &
  Thm.~\ref{thm:main-app} \\
$\KL$ & Kullback--Leibler divergence & Thm.~\ref{thm:main} \\
$n_c$, $n_{c,y}$ & times $c$ has been seen; times it was followed by $y$ &
  Def.~\ref{def:resolvable} \\
$\mathcal{M}_{t,z}$ & the matched set & \S\ref{app:con:l2match} \\
\end{longtable}

\subsection{Overloaded notation}
\label{app:not:collisions}

This work integrates harmonic analysis on finite abelian groups, graph coloring, and transformer architectures. To adhere to established conventions in these fields, eight symbols are overloaded. The surrounding context unambiguously determines the active definition, and no single equation mixes distinct definitions of the same symbol.

\smallskip
\noindent\begin{tabular}{@{}lp{0.78\linewidth}@{}}
\toprule
symbol & readings \\
\midrule
$c$
  & a \emph{configuration} in $\States^{k}$ (Def.~\ref{def:resolvable}); a
    \emph{coloring} $\Grid\to[M]$ (Def.~\ref{def:colouring}). \\
$\kappa$
  & the \emph{virtual cell index} $\kappa(n)$, always applied to a position
    (App.~\ref{app:serialization:formal}); the \emph{divergence budget}
    $\kappa\in(0,1]$, always bare (Thm.~\ref{thm:main-app}). \\
$B$
  & a \emph{residual block} $B_0,\dots,B_4$, always subscripted
    (Def.~\ref{def:stream}); the \emph{bounding box} of a target set
    (Lem.~\ref{lem:degreeone}); the MLP \emph{gate constant}
    $B_\delta = G_\delta$ of \eqref{eq:gate}, subscripted by an offset. \\
$P$
  & the \emph{positional table} (\S\ref{sec:task:model}); the \emph{parent
    configuration} $P_{t,z}$ (Def.~\ref{def:ca}); the \emph{axis factor} $P_j$
    of a box kernel (Lem.~\ref{lem:degreeone}). \\
$Q$
  & the \emph{bilinear form} realizing a kernel (Lem.~\ref{lem:realization});
    the \emph{box kernel} $Q = \sum_j P_j$ (Lem.~\ref{lem:degreeone}); the
    \emph{query configuration} $Q_{t,z}$ (Def.~\ref{def:resolvable}). \\
$M$
  & the \emph{number of colors} of an admissible coloring
    (Def.~\ref{def:colouring}), subscripted $M_K$ for a periodic one and
    starred $M^\star$ for the chromatic number; the \emph{score matrix} in the
    proof of Lem.~\ref{lem:realization}, in an appendix which never colors
    anything; the \emph{matched set} $\mathcal{M}_{t,z}$, calligraphic and
    subscripted (\S\ref{app:con:l2match}). \\
$\Delta$
  & the \emph{logit separation} (Lem.~\ref{lem:noise}); the \emph{simplex}
    $\Delta^{V-1}$, always superscripted (Def.~\ref{def:ca}); the \emph{maximum
    degree} in the Hajnal--Szemer\'edi bound
    (proof of Lem.~\ref{lem:Mbounds}). \\
$G$
  & the \emph{spatiotemporal kernel} $G(\delta,s)$, always applied to an
    offset and a step displacement (proof of Cor.~\ref{cor:head}, Stage~6);
    the grid $\Grid$ and the Cayley graph $\Gamma_\Uset$, both distinguished
    by font; the MLP \emph{gain} $G_\delta$, always subscripted (proof of
    Thm.~\ref{thm:onehead-app}). \\
\bottomrule
\end{tabular}
\smallskip

\noindent
Additionally, $\alpha$ denotes the Dirichlet concentration (Definition~\ref{def:ca}), except when written as $\alpha^{(l)}_{n,n'}$, where it represents the attention weight (\S\ref{sec:task:model}) with a layer superscript and positional subscripts.

\subsection{The experimental settings}
\label{app:not:instance}

The experiments (Appendix~\ref{app:expdetails}) cover three neighborhood types: binary $1$D ($V=2,k=3$), ternary $1$D ($V=3,k=3$), and binary $2$D von Neumann ($V=2,k=5$). Table~\ref{tab:instance-task} specifies the deterministic instances considered below, on $\mathbb{Z}_{16}$, $\mathbb{Z}_{32}$, and $\mathbb{Z}_6^2$, respectively. The stochastic experiments use the same neighborhood types but have different task parameters where indicated in Table~\ref{tab:settings}; in particular, the stochastic ternary $1$D setting uses $\mathbb{Z}_{16}$. Table~\ref{tab:instance-colour} gives the colorings for the instances in Table~\ref{tab:instance-task} alongside the bounds from Appendix~\ref{app:colouring}. Table~\ref{tab:instance-sizes} (Appendix~\ref{app:constructions}) details the widths and scales required by the two constructions.

\begin{table}[t]
\centering
\caption{The three experimental settings as tasks.}
\label{tab:instance-task}
\small
\begin{tabular}{@{}lccc@{}}
\toprule
& 1D, $V{=}2$ & 1D, $V{=}3$ & 2D von Neumann \\
\midrule
$\Grid$                                 & $\mathbb{Z}_{16}$ & $\mathbb{Z}_{32}$ & $\mathbb{Z}_6^2$ \\
$d$, $\Wmax$, $\lvert\Grid\rvert$        & $1$, $16$, $16$ & $1$, $32$, $32$ & $2$, $6$, $36$ \\
$V$, $T$                                & $2$, $10$ & $3$, $12$ & $2$, $12$ \\
$D_K$, $k$                              & $\{-1,0,1\}$, $3$ & $\{-1,0,1\}$, $3$ & von Neumann, $5$ \\
$\usucc$, $\lvert\Uset\rvert$           & $1$, $4$ & $1$, $4$ & $(0,1)$, $8$ \\
tokens per step, $N$                    & $17$, $170$ & $33$, $396$ & $42$, $504$ \\
\bottomrule
\end{tabular}
\end{table}

\apppart{Part I}{The task}
\section{Serialization and limits of sequence-relative routing}
\label{app:serialization}
\label{app:seqrpe}

This section comprises three parts. Section~\ref{app:serialization:formal} specifies the mapping from cells to sequence positions, extending it to separators via a virtual cell index. Section~\ref{app:ser:separators} addresses boundary conditions introduced by row separators. Section~\ref{app:serialization:seqrpe} proves that sequence-relative positional biases cannot strictly isolate required routing cells. Appendix~\ref{app:addressing} constructs an absolute positional table for spatial addressing.

\subsection{Formal serialization and indexing}
\label{app:serialization:formal}

\begin{definition}[Row-terminated serialization]
\label{def:serialization}
Cells are emitted in row-major order:
\[
  \rho_{\Grid}(z)
  \;\coloneqq\;
  \sum_{j=1}^{d} z_j \prod_{l>j} W_l
  \;\in\; [\lvert\Grid\rvert],
  \qquad
  r(z)
  \;\coloneqq\;
  \bigl\lfloor \rho_{\Grid}(z)/W_d \bigr\rfloor
  \;\in\; [\lvert\Grid\rvert/W_d].
\]
Here, $\rho_{\Grid}(z)$ is the flattened grid index of cell $z$, and $r(z)$ is the global index of the innermost row containing $z$, such that $\rho_{\Grid}(z)=r(z)W_d+z_d$. Let
\[
  \usucc \;\coloneqq\; (0,\dots,0,1)\in\Grid
\]
denote the \emph{successor offset}, which corresponds to advancing one position within a row.

Every innermost row concludes with a separator: a \texttt{TIME\_SEP} for the final row of a time step, and a \texttt{ROW\_SEP} otherwise. Letting $(r,c)$ denote the cell in row $r$ with last coordinate $c$, one time step is serialized as:
\begin{equation}
\label{eq:stream}
  \underbrace{
    \underbrace{x_{t,(0,0)},\,\dots,\,x_{t,(0,W_d-1)}}_{\text{row }0},\;
    \texttt{ROW\_SEP},\;\dots,\;
    \underbrace{x_{t,(m,0)},\,\dots,\,x_{t,(m,W_d-1)}}_{\text{row }m=\lvert\Grid\rvert/W_d-1},\;
    \texttt{TIME\_SEP}
  }_{\lvert\Grid\rvert+\lvert\Grid\rvert/W_d\text{ tokens}}.
\end{equation}
The complete stream comprises \eqref{eq:stream} for $t=0,\dots,T-1$, with total length
\[
  N \;\coloneqq\; T(\lvert\Grid\rvert+\lvert\Grid\rvert/W_d).
\]
The cell $(t,z)$ occupies sequence position
\begin{equation}
\label{eq:rho}
  \rho(t,z)
  \;=\;
  t(\lvert\Grid\rvert+\lvert\Grid\rvert/W_d)+\rho_{\Grid}(z)+r(z),
\end{equation}
where the terms represent tokens in prior time steps, preceding cells in the current step, and preceding row separators. The map $\rho: [T]\times\Grid\to[N]$ is thus an order-preserving injection.
\end{definition}

\begin{exmp}[One dimension]
In $d=1$, each time step is one row, so $r(z)=0$ and $\lvert\Grid\rvert/W_d=1$. Writing
$W=W_1$, one time step is $W$ cell tokens followed by one
\texttt{TIME\_SEP}, and
\[
  \usucc=1,
  \qquad
  \rho(t,z)=(W+1)t+z.
\]
\end{exmp}

\begin{exmp}[Two dimensions]
In $d=2$, with grid dimensions $W_1 \times W_2$, the innermost row length is $W_2$. For a cell at $z = (z_1, z_2)$:
\[
  \rho_{\Grid}(z) = z_1 W_2 + z_2,
  \qquad
  r(z) = z_1.
\]
Here, $r(z)$ extracts the row coordinate $z_1$. One time step contains $W_1$ rows of length $W_2$, requiring $W_1$ separators. The sequence position is therefore:
\[
  \rho(t,z) = t(W_1 W_2 + W_1) + (z_1 W_2 + z_2) + z_1.
\]
\end{exmp}

\paragraph{Causal masking.}
The model employs a standard causal mask; a query at sequence position $n$ attends only to positions $n'\le n$. For cell tokens, $(t',z')$ is visible to $(t,z)$ if and only if $\rho(t',z')\le\rho(t,z)$.

\paragraph{Virtual cell index.}
We assign a virtual spatial index and time step to every sequence position except the final \texttt{TIME\_SEP}. For a cell token $n=\rho(t,z)$, define
\[
  (\tau(n),\kappa(n))\coloneqq(t,z).
\]
For a separator $n$ whose successor cell at $n+1$ is $(t',z')$, define
\[
  (\tau(n),\kappa(n))\coloneqq(t',z'-\usucc).
\]
Thus, a \texttt{ROW\_SEP} inherits its row's time step, while a \texttt{TIME\_SEP} assumes the subsequent time step. The final \texttt{TIME\_SEP}, at $n=N-1$, has no successor cell, so $\tau(N-1)$ and $\kappa(N-1)$ are left undefined. It poses no query, and as the last position it is visible as a key only to itself. We therefore exclude it from every requirement below which quantifies over virtual addresses, including the gather, and give its column of the positional table arbitrary features; no scored prediction depends on it.

\begin{lemma}[Constant successor offset]
\label{lem:successor}
Under Definition~\ref{def:serialization}, a cell token $(t,z)$ is followed by a
cell token exactly when $z_d<W_d-1$, in which case that successor is
$(t,z+\usucc)$. Every separator other than the final \texttt{TIME\_SEP} is also
followed by a cell token. Whenever position $n$ is followed by a cell token,
that token, which we call the \emph{target cell} of $n$, is
\[
  (\tau(n),\kappa(n)+\usucc).
\]
Hence every position which predicts a cell does so at the same spatial offset
$\usucc$ from its virtual cell index.
\end{lemma}

\begin{proof}
Within an innermost row, a non-final cell $z$ is followed by $z+\usucc$, and a row-final cell by its separator. Non-final separators are followed by the first cell of the subsequent row or time step. For a separator with successor $(t',z')$, the virtual index yields $\tau(n)=t'$ and $\kappa(n)=z'-\usucc$; its target remains $\kappa(n)+\usucc=z'$.
\end{proof}

\begin{rmk}[Size of the routing set]
\label{rmk:usize}
Because $\Uset = D_K \cup (D_K+\usucc)$ is a union of two sets of size $k$, and
because the offsets are distinct modulo $(W_1,\dots,W_d)$ by
Definition~\ref{def:ca},
\[
  k \;\leq\; \lvert\Uset\rvert \;\leq\; \min(2k,\lvert\Grid\rvert),
\]
which is the bound quoted in \S\ref{sec:task:routing}. The lower bound is
attained exactly when $D_K+\usucc = D_K$, in which case $\Uset = D_K$. This
degenerate case is possible: on $\mathbb{Z}_3^2$ with
$D_K = \{(0,0),(0,1),(0,2)\}$ and $\usucc = (0,1)$, translation by $\usucc$
permutes $D_K$, so $k = \lvert\Uset\rvert = 3$. It requires $D_K$ to be a union
of cosets of $\langle\usucc\rangle$, and so to wrap the innermost axis
completely; every neighborhood we consider has $\lvert\Uset\rvert \geq k+1$.
\end{rmk}

\subsection{Boundary and separator queries}
\label{app:ser:separators}
\label{app:separators}
\label{app:sep:rowfinal}
\label{app:sep:sepquery}

Serialization introduces two boundary cases: a \emph{row-final} cell ($z_d=W_d-1$) predicts a separator, and a separator predicts the first cell of the next row. These cases, fixed by emission order, are summarized in Table~\ref{tab:duty}.

\begin{table}[t]
\centering
\caption{Prediction duties under Definition~\ref{def:serialization}. For a
non-row-final cell and for a separator, the predicted token is a cell and its
spatial location is expressed through the constant successor offset
$\usucc$. A row-final cell instead predicts a separator whose identity is fixed
by position.}
\label{tab:duty}
\small
\begin{tabular}{lll}
\toprule
token & predicts & spatial target \\
\midrule
cell $(t,z)$, $z_d<W_d-1$
  & cell $(t,z+\usucc)$
  & $z+\usucc$ \\
cell $(t,z)$, $z_d=W_d-1$
  & a separator
  & none \\
separator
  & first cell of the next row
  & $\kappa(n)+\usucc$ \\
\bottomrule
\end{tabular}
\end{table}

\begin{lemma}[Row-final targets are fixed by position]
\label{lem:rowfinal}
Let $n=\rho(t,z)$ hold a cell token with $z_d=W_d-1$ and $n<N-1$. Then the
token at $n+1$ is a \texttt{TIME\_SEP} if $r(z)=\lvert\Grid\rvert/W_d-1$ and a
\texttt{ROW\_SEP} otherwise. In particular, conditioned on the position $n$,
the next token is deterministic and independent of the cellular-automaton
kernel $f$.
\end{lemma}

\begin{proof}
The emission order is fixed before any cell state is drawn. Every innermost row
is followed by exactly one separator, and the separator is a
\texttt{TIME\_SEP} exactly for the last row of a time step. Equation~\eqref{eq:rho}
therefore determines the next token type from $n$ alone.
\end{proof}

The experimental evaluation (Section~\ref{sec:experiments}) scores cell-value targets, excluding separators. Scored predictions include both non-row-final cell queries and separator queries targeting the prediction region. Row-final cell tokens (fraction $1/W_d$) predict separators whose identities are positionally fixed and independent of the rule $f$. However, they remain valid \emph{keys} for Layer~2, as their states are valid observations of $f$ (Remark~\ref{rmk:gather-scope}).

The remaining case is a separator followed by a cell. We chose the virtual
index so that this case has the same spatial target convention as an ordinary
query.

\begin{lemma}[Separator queries]
\label{lem:sepsucc}
\label{cor:sepvisible}
Let $n$ hold a separator with $n<N-1$, and let $(t',z')$ be the cell token at
$n+1$.
\begin{enumerate}[label=(\roman*),leftmargin=2.2em,itemsep=2pt,topsep=3pt]
  \item If $n$ is a \texttt{ROW\_SEP} closing row $r$ of time step $t$, then
        $t'=t$, $r(z')=r+1$, and $z'_d=0$. If $n$ is a
        \texttt{TIME\_SEP} closing time step $t$, then
        $(t',z')=(t+1,0)$. In either case,
        \[
          \tau(n)=t',
          \qquad
          \kappa(n)=z'-\usucc,
          \qquad
          \kappa(n)+\usucc=z'.
        \]
  \item If $t'\ge1$, then every cell in time step $t'-1=\tau(n)-1$ precedes
        $n$ in the stream. Hence all $k$ parents of the target cell $z'$ are
        causally visible when the separator predicts it.
  \item If $t'\ge1$, the parents of $z'$ occur at offsets $D_K$ from $z'$,
        equivalently at offsets
        \[
          D_K+\usucc=D_Q\subseteq\Uset
        \]
        from the virtual index $\kappa(n)$. Thus separator queries whose
        targets are CA-generated cells require no spatial offsets beyond the
        routing set already defined in Definition~\ref{def:routing}.
\end{enumerate}
\end{lemma}

\begin{proof}
For part~(i), the token following the separator of row $r$ is the first cell of
row $r+1$, while the token following a \texttt{TIME\_SEP} is the first cell of
the next time step. The identities involving $(\tau,\kappa)$ then follow from
the definition of the virtual index.

For part~(ii), if $n$ is a \texttt{ROW\_SEP}, then $t'=t$, and all of time
step $t-1$ occurs before every token of time step $t$. If $n$ is a
\texttt{TIME\_SEP}, then $t'=t+1$, and every cell of time step $t$ appears
before its closing separator. The condition $t'\ge1$ excludes separators
inside the initial condition, whose successor cells have no CA parents.

For part~(iii), when $t'\ge1$, the parent set of the target
$z'=\kappa(n)+\usucc$ is
$z'+D_K=\kappa(n)+(D_K+\usucc)=\kappa(n)+D_Q$.
\end{proof}

The two separator types differ only in how their stream position relates to the target time step: a \texttt{ROW\_SEP} predicts a cell in its own time step (reading parents from $t-1$), whereas a \texttt{TIME\_SEP} predicts the first cell of the next step (reading parents from $t$). Indexing by $\tau$ unifies these cases. Initial-condition separators predicting initial states fall outside the transition-prediction task. The transformer constructions (Appendix~\ref{app:constructions}) use separators as queries when predicting CA cells, gating them out as keys via token content; this requires no extra spatial offsets or positional coordinates.

\subsection{Limits of displacement-only positional biases}
\label{app:serialization:seqrpe}

In this section, we bound what a positional bias depending only on sequence displacement can do. No fixed scalar bias scores every desired routing cell strictly above all other causally visible cells at every query position; periodic wraparound prevents this one-layer isolation (Theorem~\ref{thm:seqrpe}).

The absolute positional table (Section~\ref{sec:task:model}) lacks an explicit coordinate system, requiring spatial structure to be learned from data. A simpler alternative is a scalar additive bias depending solely on sequence displacement \citep{raffel2020exploring,press2022train}.

On an open grid, a fixed spatial offset yields a constant sequence displacement between any two valid cells, provided no boundary truncation or padding changes the row length. On a torus, wraparound disrupts this consistency. Theorem~\ref{thm:seqrpe} identifies two queries where conflicting displacements prevent any fixed displacement-only bias from isolating the desired cells.

For the formulas below, write
\[
  R\coloneqq W_d+1,
  \qquad
  S\coloneqq (\lvert\Grid\rvert/W_d)R,
  \qquad
  r(z)\coloneqq\sum_{j<d} z_j\prod_{j<l<d}W_l.
\]
Here $R$ is the row stride, $S$ is the time-step stride, and $r(z)$ agrees with
Definition~\ref{def:serialization}. Equation~\eqref{eq:rho} becomes
\begin{equation}
\label{eq:rho-strides}
  \rho(t,z)=tS+r(z)R+z_d.
\end{equation}

\begin{definition}[Sequence-relative bias]
\label{def:seqrpe}
A \emph{sequence-relative bias} is a function
\[
  b:\{0,\dots,N-1\}\to\mathbb{R}
\]
whose contribution to the attention score between a query at position $n$ and
a key at position $n'\le n$ depends only on the displacement $n-n'$. Separator
keys may additionally be gated out by token content.

For the isolation property below, ordinary cell keys are ranked solely by the
positional score $b(n-n')$. For a non-row-final cell query $(t,z)$ with
$1\le t<T$ and $z_d<W_d-1$, such a bias \emph{isolates} an offset set
$H\subseteq\Grid$ if every desired cell $(t-1,z+\delta)$ with $\delta\in H$
receives a strictly larger positional score than every causally visible cell
outside this desired set.
\end{definition}

Allowing separator keys to be gated only strengthens the class in
Definition~\ref{def:seqrpe}. The witnesses below compare cell tokens with cell
tokens and never use a separator as a competing key.

\begin{exmp}[The elementary one-dimensional contradiction]
\label{exmp:seqrpe-elementary}
For the elementary neighborhood,
$D_K=\{-1,0,1\}$ and
$\Uset=\{-1,0,1,2\}$. Let $W\ge4$. At the query $(2,0)$, the desired cell at
offset $-1$ wraps to the right edge of the previous time step and lies at
sequence displacement $2$. The cell at displacement $2W-1$ is visible but
lies two time steps back, so correct isolation requires
\[
  b(2)>b(2W-1).
\]
At the query $(t,W-2)$, the desired cell at offset $2$ wraps to the left edge
of the previous time step and lies at displacement $2W-1$, while displacement
$2$ points to a visible cell in the current time step. Correct isolation now
requires
\[
  b(2W-1)>b(2).
\]
The two inequalities are incompatible. Theorem~\ref{thm:seqrpe} generalizes
this argument.
\end{exmp}

\begin{rmk}[What the theorem does not cover]
\label{rmk:seqrpe-scope}
A bias which may also read the query position modulo the row stride is already
grid-indexed: $n\bmod R=z_d$ on cell tokens, which exposes exactly the boundary
information used by the contradiction. Standard one-dimensional rotary
embeddings \citep{su2024roformer} enter the logits through a
position-dependent bilinear form, so they too fall outside
Definition~\ref{def:seqrpe} and the theorem. We do not
analyze either class here.
\end{rmk}

\begin{theorem}[Sequence-relative obstruction]
\label{thm:seqrpe}
Fix the serialization of Definition~\ref{def:serialization} with $T\ge3$.
\begin{enumerate}[label=(\roman*),leftmargin=2.2em]
  \item Suppose $d=1$, $W\ge4$, and $\Uset$ contains some negative offset
        $-m$ with $1\le m\le W-3$ and some positive offset $p$ with
        $2\le p\le W-2$. Then no single sequence-relative bias isolates
        $\Uset$ for every query cell $z<W-1$.
  \item Suppose $d\ge2$, $W_d\ge2$, and some $\delta\in D_K$ has a nonzero
        component modulo the corresponding side length on an outer axis;
        that is, $\delta_j\not\equiv0\pmod{W_j}$ for some $j<d$. Then no
        partition
        \[
          \Uset=H_1\sqcup\cdots\sqcup H_h
        \]
        together with sequence-relative biases $b_1,\dots,b_h$ can make each
        $b_a$ isolate $H_a$ at every query. In particular, no single
        sequence-relative bias isolates $\Uset$.
\end{enumerate}
\end{theorem}

Part~(i) includes the elementary neighborhood by taking $m=1$ and $p=2$.
Part~(ii) applies whenever the neighborhood is not confined to one innermost
row, including the usual von Neumann and Moore neighborhoods in dimension at
least two. The result concerns one-layer isolation in the sense of
Definition~\ref{def:seqrpe}; it does not rule out a model which assembles the
routing set through several layers or which is supplied with additional
spatial information.

\begin{proof}
\emph{Part (i).}
Write $W=W_1$. Then $R=S=W+1$ and
$\rho(t,z)=(W+1)t+z$. For a query $(t,z)$, the displacement to the desired cell
$(t-1,z+\delta)$ is
\[
  \rho(t,z)-\rho(t-1,z+\delta)
  =
  \begin{cases}
    (W+1)-\delta,
      & 0\le z+\delta<W,\\
    (2W+1)-\delta,
      & z+\delta\ge W,\\
    1-\delta,
      & z+\delta<0,
  \end{cases}
\]
where the last two cases use toroidal wraparound.

Take first the query $(2,0)$, which $T\ge3$ and $W\ge4$ make a non-row-final
cell query in the sense of Definition~\ref{def:seqrpe}. The desired cell at
offset $-m$ wraps below and
therefore lies at displacement $m+1$. The position at displacement
$2W+1-p$ is
\[
  \rho(2,0)-(2W+1-p)=p+1,
\]
which is the visible cell $(0,p+1)$ because $p\le W-2$. It lies in time step
$0$, not in the desired parent step $1$, so it is not one of the routed cells.
Isolation therefore requires
\begin{equation}
\label{eq:seqrpe-1d-first}
  b(m+1)>b(2W+1-p).
\end{equation}

Now take the query $(t,W-2)$ for any $1\le t<T$, which is a non-row-final cell
query because $W\ge4$. The desired cell at offset $p$ wraps above and
lies at displacement $2W+1-p$. The position at displacement $m+1$ is
\[
  \rho(t,W-2)-(m+1)=\rho(t,W-3-m),
\]
which is a cell in the current time step because $m\le W-3$. It precedes the
query and is therefore visible, but it does not belong to the previous time
step. Hence isolation requires
\begin{equation}
\label{eq:seqrpe-1d-second}
  b(2W+1-p)>b(m+1),
\end{equation}
contradicting \eqref{eq:seqrpe-1d-first}. This proves part~(i).

\emph{Part (ii).}
Let $M_{\mathrm{row}}\coloneqq \lvert\Grid\rvert/W_d$ be the number of rows in one time step, so $S=M_{\mathrm{row}}R$, and
write
\[
  P_j\coloneqq\prod_{j<l<d}W_l,
  \qquad
  r(z)=\sum_{j<d}z_jP_j.
\]
Choose $\delta\in D_K$ such that $\delta_j\not\equiv0\pmod{W_j}$ for some
$j<d$, and represent each coordinate by its residue
$\delta_j\in\{0,\dots,W_j-1\}$. Let $i<d$ be the smallest
index with $\delta_i\ne0$, and define
\[
  C\coloneqq\sum_{j<d}\delta_jP_j,
  \qquad
  C'\coloneqq\sum_{i<j<d}\delta_jP_j,
  \qquad
  G\coloneqq W_iP_iR,
  \qquad
  A\coloneqq S-CR-\delta_d.
\]
Because $i$ is the first nonzero outer coordinate, mixed-radix bounds give
$C'<P_i$ and hence
\[
  P_i\le C<W_iP_i\le M_{\mathrm{row}}.
\]

Consider first the query $(2,0)$, which $T\ge3$ and $W_d\ge2$ make a
non-row-final cell query. The desired routed cell at offset $\delta$ is
$(1,\delta)$, whose sequence displacement is
\[
  \rho(2,0)-\rho(1,\delta)
  =2S-(S+CR+\delta_d)
  =A.
\]
The position at displacement $A+G$ is
\[
  2S-A-G
  =S+(C-W_iP_i)R+\delta_d
  =(M_{\mathrm{row}}+C-W_iP_i)R+\delta_d.
\]
Since $0\le M_{\mathrm{row}}+C-W_iP_i<M_{\mathrm{row}}$, this is a cell in time step $0$. It is visible, but
all desired routed cells for the query lie in time step $1$, so it is not a
desired key. Therefore the bias assigned to any part containing $\delta$ must
satisfy
\begin{equation}
\label{eq:seqrpe-highd-first}
  b(A)>b(A+G).
\end{equation}

For the reverse inequality, take
\[
  z^\ast
  \coloneqq
  (W_1-1,\dots,W_{i-1}-1,\,W_i-\delta_i,\,0,\dots,0)
\]
at time $t=1$. Since $z_d^\ast=0$ and $W_d\ge2$, this is a non-row-final cell
query. Since
\[
  \sum_{j<i}(W_j-1)P_j=M_{\mathrm{row}}-W_iP_i,
\]
we have $r(z^\ast)=M_{\mathrm{row}}-\delta_iP_i$. Adding $\delta$ wraps coordinate $i$ to
zero and gives
\[
  r(z^\ast+\delta)=M_{\mathrm{row}}-W_iP_i+C'.
\]
Thus the desired cell $(0,z^\ast+\delta)$ lies at displacement
\begin{align*}
  \rho(1,z^\ast)-\rho(0,z^\ast+\delta)
  &=S+\bigl(W_iP_i-\delta_iP_i-C'\bigr)R-\delta_d\\
  &=A+G.
\end{align*}
By contrast, the position at displacement $A$ is
\[
  \rho(1,z^\ast)-A=S+C'R+\delta_d,
\]
a cell in time step $1$, in row $C'$ and column $\delta_d$. It precedes the
query because
\[
  C'R+\delta_d
  <(C'+1)R
  \le P_iR
  \le(M_{\mathrm{row}}-\delta_iP_i)R,
\]
and it is not a desired routed cell because the desired cells lie in time step
$0$. Hence the same bias must satisfy
\begin{equation}
\label{eq:seqrpe-highd-second}
  b(A+G)>b(A),
\end{equation}
contradicting \eqref{eq:seqrpe-highd-first}.

This contradiction applies to whichever partition part contains $\delta$.
Therefore no family of sequence-relative biases, one per part of a partition
of $\Uset$, can isolate all parts. This proves part~(ii).
\end{proof}

\paragraph{Spatial relative positions and learned routing.}
A spatial relative-position bias in $d$ dimensions, augmented with temporal
displacement, can avoid this obstruction. For ordinary cell tokens, such a bias
can take the form
\[
  b_{\mathrm{sp}}(t-t',\,z'-z),
\]
where $z'-z\in\Grid$ is computed coordinatewise modulo $(W_1,\dots,W_d)$.
Assigning a higher score when $t-t'=1$ and $z'-z\in\Uset$ than otherwise
isolates the desired routing cells. The same spatial offset therefore remains
identifiable across periodic boundaries.

This parameterization explicitly supplies the mapping from sequence positions to
spatial and temporal coordinates, together with the periodic boundary structure,
even when the bias values themselves are learned. Our experimental models
instead use freely learned absolute position embeddings indexed only by sequence
position, without an explicit coordinate-based parameterization. The
serialization and separators supply structural cues, but the model must learn
how to use them for spatial routing. Our constructions establish that such
routing can be represented, and our experiments provide evidence that
gradient-based training learns routing behavior consistent with these
constructions.
\section{Coverage and Bayes-optimal prediction}
\label{app:bayes}

All distributional guarantees apply to \emph{resolvable} queries and are measured against the \emph{Bayes-optimal predictor}. Section~\ref{app:bayes:coverage} characterizes coverage and specifies the queries excluded by resolvability, where optimal prediction relies solely on position and the prior. Section~\ref{app:bayes:optimal} derives the Bayes-optimal predictor (Theorem~\ref{thm:bayes}). These results rely exclusively on Definition~\ref{def:ca} and the serialization from Appendix~\ref{app:serialization}, without reference to transformer architectures.

\subsection{Coverage and resolvable queries}
\label{app:bayes:coverage}

We adapt Definition~\ref{def:resolvable} to the serialized stream to address two issues: the parent configuration $P_{t,z}$ (Definition~\ref{def:ca}) is defined only for $t \geq 1$ (as the initial step is drawn independently), and the first cell of every row is predicted at a separator token (Lemma~\ref{lem:sepsucc}), which is scored alongside interior cells (Section~\ref{sec:experiments}).

\paragraph{Querying positions.}
A position $n$ \emph{poses a query} when it has a target cell in the sense of
Lemma~\ref{lem:successor} and that target cell lies in a time step
$\tau(n) \geq 1$. Two families qualify, and they are the first and third rows of
Table~\ref{tab:duty}:
\begin{enumerate}[label=(\alph*),leftmargin=2.2em,itemsep=1pt,topsep=2pt]
  \item a cell token $(t,z)$ with $t \geq 1$ and $z_d < W_d - 1$, whose target
        cell is $z + \usucc$ in its own time step;
  \item a separator at $n < N-1$ with $\tau(n) \geq 1$, whose target cell is
        $\kappa(n) + \usucc$, the first cell of the next row, again in step
        $\tau(n)$.
\end{enumerate}
In both cases the query configuration is
\[
  Q_n \;\coloneqq\; P_{\tau(n),\,\kappa(n)+\usucc} ,
\]
which is defined because $\tau(n) \geq 1$, and the target is a draw from
$f(\cdot \mid Q_n)$. For a cell token this is the $Q_{t,z}$ of
Definition~\ref{def:resolvable}. Theorem~\ref{thm:main-app} gives the
separator case of the guarantee, under the same definition.

\paragraph{Counting positions.}
The counts range over the cells generated by $f$:
\begin{equation}
\label{eq:counts}
  n_c \;\coloneqq\;
  \#\bigl\{(t',z') \;:\; t' \geq 1,\ \rho(t',z') \leq n,\ P_{t',z'} = c \bigr\},
  \qquad
  n_{c,y} \;\coloneqq\; \#\{\text{those with } x_{t',z'} = y\} .
\end{equation}
Above, $c \in \States^{k}$ is a configuration and $n$ is the sequence position
posing the query. We need the restriction $t' \geq 1$ for two reasons:
$P_{0,z'}$ is undefined, and $x_{0,z'}$ is an initial state, so it says nothing
about the kernel. A configuration $c$ is \emph{covered} at $n$ when
$n_c \geq 1$, and the query at $n$ is \emph{resolvable} when its configuration
$Q_n$ is covered.

Coverage is a property of the stream alone, and it is monotone in
position.

\begin{lemma}[Coverage is monotone]
\label{lem:coverage}
Fix a trajectory and a configuration $c \in \States^{k}$.
\begin{enumerate}[label=(\roman*),leftmargin=2.2em,itemsep=1pt,topsep=2pt]
  \item $n_c$, evaluated at $(t,z)$, is non-decreasing in $\rho(t,z)$.
  \item If a query posing $c$ is resolvable at $(t,z)$, then every query posing
        $c$ at a position $(t'',z'')$ with $\rho(t'',z'') \geq \rho(t,z)$ is
        resolvable.
  \item A query is resolvable if and only if $n_c \geq 1$ for its own
        configuration $c = Q_{t,z}$.
\end{enumerate}
\end{lemma}

\begin{proof}
Since $\rho$ is an order-preserving injection (Definition~\ref{def:serialization}), the sets $A(n) \coloneqq \{(t',z') : t' \geq 1,\ \rho(t',z') \leq n\}$ form a chain under inclusion. The condition $t' \geq 1$ excludes the initial time step uniformly, preserving the chain structure. Because $n_c$ at $(t,z)$ counts elements of $A(\rho(t,z))$ with parent configuration $c$, it is non-decreasing along this chain, proving (i). Statement (ii) follows from (i) and (iii). For (iii), $c$ is covered at $(t,z)$ if and only if $A(\rho(t,z))$ contains an element with parent configuration $c$, meaning $n_c \geq 1$.
\end{proof}

\noindent
Part (iii) gives a nonempty matched set. Both the deterministic and the
stochastic analysis divide by $n_c$, so both need resolvability.

Evaluating only on resolvable queries excludes three classes of cell token.

\begin{lemma}[What resolvability excludes]
\label{lem:coverage-excluded}
A cell token either poses a resolvable query or falls into exactly one of the
following three classes, tested in this order.
\begin{enumerate}[label=(\roman*),leftmargin=2.2em,itemsep=1pt,topsep=2pt]
  \item \emph{Initial-step cells}, those with $t = 0$. Their target is a
        separator when $z_d = W_d - 1$ and an initial state otherwise. The
        separator is determined by the position $n$, and the initial state is
        drawn from the initial-state distribution of Definition~\ref{def:ca},
        so neither depends on $f$.
  \item \emph{Row-final cells}, those with $t \geq 1$ and $z_d = W_d - 1$.
        Their target is a separator, determined by the position $n$ alone, so
        its conditional law given everything observed is a point mass and
        carries zero mutual information with $f$.
  \item \emph{Novel configurations}, those with $t \geq 1$ and $z_d < W_d-1$
        whose query configuration has $n_c = 0$. There the Bayes posterior
        predictive of Theorem~\ref{thm:bayes-app} is the uniform distribution
        $\widehat P_\alpha(y \mid c) = 1/V$ for every $y$, whatever $f$ is.
\end{enumerate}
In all three classes the optimal prediction depends on the position and the
prior alone, and the trajectory adds nothing.
\end{lemma}

\begin{proof}
Tokens with $t = 0$ fall into class (i); those with $t \geq 1$ and $z_d = W_d - 1$ into class (ii); and those with $t \geq 1$, $z_d < W_d - 1$, and an uncovered configuration into class (iii) (Lemma~\ref{lem:coverage}(iii)). The three tests are applied in that order, so the classes are disjoint, and all other cell tokens pose resolvable queries. For (i), an initial-step row-final cell predicts a positionally fixed separator (Lemma~\ref{lem:rowfinal}) and every other initial-step cell predicts an initial state, drawn independently of $f$. For (ii), a row-final cell predicts a deterministically positioned separator (Lemma~\ref{lem:rowfinal}), conveying zero mutual information about $f$. For (iii), substituting $n_c = n_{c,y} = 0$ into Theorem~\ref{thm:bayes-app} yields $\widehat P_\alpha(y \mid c) = 1/V$, the mean of the $\mathrm{Dir}(\alpha\mathbf 1_V)$ prior, independent of the trajectory.
\end{proof}

\begin{rmk}
The excluded classes differ in frequency and rationale. Initial-step cells (fraction $1/T$) and novel configurations are excluded because their Bayes-optimal prediction is uninformative about the inferred rule, scoring at $1/V$. Row-final cells (fraction $1/W_d$) are universally excluded because they deterministically predict a separator; scoring them would artificially inflate accuracy. The novel configurations become rare as $T$ grows. Once an uncovered configuration is queried, its target is observed and becomes a precedent, and by Lemma~\ref{lem:coverage}(ii) it stays covered. Each of the $V^k$ configurations therefore yields at most one uncovered query in the whole stream, so for fixed $V$, $k$ and $\Grid$ their number is bounded independently of $T$ and their fraction vanishes. Consequently, we rejection-sample evaluation trajectories to ensure every scored query is resolvable (Appendix~\ref{app:exp:eval}).
\end{rmk}

\subsection{Bayes-optimal posterior prediction}
\label{app:bayes:optimal}

Under the prior of Definition~\ref{def:ca}, no predictor attains lower expected log loss than the posterior predictive, which relies on counts (Theorem~\ref{thm:bayes-app}). We measure divergences against this Bayes-optimal predictor, and Theorem~\ref{thm:main} bounds the circuit's divergence from it. For next-cell accuracy the Bayes act is the $\arg\max$ of the posterior predictive. Sampling from it is a different predictor: at predictive distribution $(0.9,0.1)$, a sample matches the target with probability $0.82$, while the mode matches with probability $0.9$.

\begin{theorem}[Bayes posterior predictive; Theorem~\ref{thm:bayes} of the body]
\label{thm:bayes-app}
Let $f$ be drawn as in Definition~\ref{def:ca} and let a trajectory be
serialized as in Appendix~\ref{app:serialization}. For every token with
resolvable query $c$ and every $y \in \States$, the posterior predictive
distribution of the target given everything observed up to that token is exactly
\[
  \widehat{P}_\alpha(y \mid c) \;=\; \frac{n_{c,y} + \alpha}{n_c + V\alpha},
\]
and no predictor attains smaller expected log-loss.
\end{theorem}

\begin{proof}
By Definition~\ref{def:ca}, cells at $t \geq 1$ are conditionally independent across $z$ given the previous step, and the initial time step is drawn independently of $f$. The trajectory likelihood factorizes into an initial-step term and one multinomial factor per configuration $c$, with counts $n_{c,y}$. The initial-step factor $p(x_{0,\cdot})$ is independent of $f$ and cancels from the posterior. The independent Dirichlet prior is conjugate, yielding the posterior $\bigotimes_c \mathrm{Dir}(\alpha\mathbf{1} + n_{c,\cdot})$. The mean of a $\mathrm{Dir}(\beta)$ distribution is $\beta_y/\sum_v \beta_v$, producing the stated formula. Under log-loss, the Bayes act is exactly the posterior predictive.
\end{proof}

\begin{rmk}
The proof does not use resolvability, and the formula remains valid at
$n_c = 0$, where Lemma~\ref{lem:coverage-excluded}(iii) shows it returns the
prior mean. We state resolvability in the theorem because it is the regime
where the formula is informative, and because every later result which uses
this predictor divides by $n_c$.
\end{rmk}

\begin{corollary}[Empirical-frequency limit and deterministic lookup]
\label{cor:bayes-lookup}
As $\alpha \to 0$ the predictor of Theorem~\ref{thm:bayes-app} tends to the
empirical frequency $n_{c,y}/n_c$. If in addition every $f(\cdot \mid c)$ is a
point mass, meaning the deterministic corner of Definition~\ref{def:ca}, then a
single observation of $c$ determines its image, so the predictor is the lookup
table and is \emph{exact} on every resolvable query.
\end{corollary}

\begin{proof}
The limit is immediate from the displayed formula. Under a point mass
$f(\cdot\mid c) = \mathbf{1}_{y^\star}$ every observation of $c$ has target
$y^\star$, so $n_{c,y}/n_c = \mathbf{1}[y = y^\star]$ as soon as
$n_c \geq 1$, which Lemma~\ref{lem:coverage}(iii) supplies at a resolvable
query.
\end{proof}

\begin{rmk}
Corollary~\ref{cor:bayes-lookup} characterizes the \emph{predictor}. In the
deterministic setting the Bayes-optimal predictor is exact, so exact-match accuracy is a
meaningful metric and $100\%$ is attainable.
Theorem~\ref{thm:main-app}\emph{(ii)} characterizes the \emph{circuit}, and it
shows that the two-layer construction attains that floor once the score gap is
large enough.

Outside the deterministic regime the Bayes-optimal predictor sits strictly below
$100\%$ accuracy, because the query configuration determines only a distribution over
the target. Table~\ref{tab:results} therefore reports the Kullback--Leibler divergence from
the $k$-gram estimate of $\widehat P_\alpha$, and Theorem~\ref{thm:main} bounds
the divergence from $\widehat P_\alpha$ itself. A predictor which matches the
Bayes-optimal predictor still errs on a strictly positive fraction of stochastic
queries.
\end{rmk}

\apppart{Part II}{Spatial routing primitives}
\section{Spatial addressing with grid-independent width}
\label{app:addressing}

Appendix~\ref{app:seqrpe} demonstrates that a sequence-relative additive bias (Definition~\ref{def:seqrpe}) cannot route a toroidal neighborhood. Consequently, the Layer-1 positional table is indexed spatially. This section derives the required table width. Lemma~\ref{thm:addressing} establishes that a positional bias can isolate any target offset set using a grid-independent number of frequencies; the table width is this frequency count plus two.

The section is structured as follows. Section~\ref{app:addr:fourier} shows the coordinate dimension of a translation-invariant bias equals its Fourier spectral support size. Section~\ref{app:addr:isolation} bounds this support for isolating kernels, which strictly order a target set above its complement. Section~\ref{app:addr:scale} translates this ordering into an attention mass bound. Finally, Section~\ref{app:addr:interface} constructs the positional table and its projections, establishing the interface for Appendix~\ref{app:constructions}, while Section~\ref{app:addr:instance} provides a numerical example.

\subsection{Fourier representation of spatial attention scores}
\label{app:addr:fourier}

We apply harmonic analysis on finite abelian groups \citep[see][\S1.2 and \S1.10]{terras1999fourier} without proof.

The grid $\Grid = [W_1]\times\cdots\times[W_d]$ is the product of the cyclic
groups $\mathbb{Z}_{W_1},\dots,\mathbb{Z}_{W_d}$ under coordinatewise addition,
and has order $\lvert\Grid\rvert = \prod_j W_j$. Its dual $\Gd$ is the group of characters,
indexed by the same set: for $\omega \in \Gd$ and $\delta \in \Grid$ put
\begin{equation}
\label{eq:phase}
  \varphi_\omega(\delta) \;\coloneqq\; 2\pi\sum_{j=1}^{d}
    \frac{\omega_j\delta_j}{W_j},
  \qquad
  \chi_\omega(\delta) \;\coloneqq\; e^{i\varphi_\omega(\delta)} .
\end{equation}
Two consequences of \eqref{eq:phase} are frequently used. First, the phase is additive,
\begin{equation}
\label{eq:additive}
  \varphi_\omega(\delta + \delta') \;\equiv\;
  \varphi_\omega(\delta) + \varphi_\omega(\delta') \pmod{2\pi},
  \qquad\text{equivalently}\qquad
  \chi_\omega(\delta+\delta') \;=\;
  \chi_\omega(\delta)\,\chi_\omega(\delta') ,
\end{equation}
meaning a displacement in the argument rotates the value. The congruence is
needed because $\delta+\delta'$ is reduced modulo the side lengths; the
character identity, and every sine and cosine computed from it, hold exactly. Second, characters are orthonormal:
\begin{equation}
\label{eq:orthogonality}
  \mathop{\mathbb{E}}_{\omega \sim \mathrm{Unif}(\Gd)}\;
  \chi_\omega(y)
  \;=\;
  \frac{1}{\lvert\Grid\rvert}\sum_{\omega \in \Gd} \chi_\omega(y)
  \;=\;
  \begin{cases}
    1 & y = 0, \\
    0 & y \neq 0,
  \end{cases}
  \qquad y \in \Grid .
\end{equation}
Since the characters form an orthonormal basis of $\mathbb{C}^{\Grid}$, every
$g : \Grid \to \mathbb{C}$ has a unique expansion
$g = \sum_{\omega} \hat g(\omega)\chi_\omega$. The \emph{spectrum} of $g$ is
$\supp \hat g \coloneqq \{\omega : \hat g(\omega) \neq 0\}$, and it is
conjugate-symmetric, $\hat g(-\omega) = \overline{\hat g(\omega)}$, when $g$ is
real. We write $e_1,\dots,e_d$ for the standard generators of $\Gd$, and
\[
  \dS \;\coloneqq\; \bigl\lvert \supp\hat g \setminus \{0\} \bigr\rvert
\]
for the number of nontrivial characters in the spectrum, the \emph{addressing
dimension} of $g$.

A displacement-dependent positional bias $z - z'$ induces a $\Grid$-circulant score matrix. Lemma~\ref{lem:realization} establishes its minimal rank factorization, with transformer implications discussed in \S\ref{app:addr:interface}.

\begin{lemma}[Realization of translation-invariant spatial biases]
\label{lem:realization}
Let $g : \Grid \to \mathbb{R}$ and $\Omega \coloneqq \supp\hat g$. The least $n$
for which there exist $p : \Grid \to \mathbb{R}^n$ and $Q \in
\mathbb{R}^{n\times n}$ with
\[
  p_z^\top Q\, p_{z'} \;=\; g(z-z')
  \qquad\text{for all } z,z' \in \Grid
\]
is exactly $\lvert\Omega\rvert$. A $Q$ attaining it may be taken block diagonal
with blocks of size $1$ and $2$, the $2\times2$ blocks being scaled rotations.
\end{lemma}

\begin{proof}
\emph{The score matrix and its eigenvalues.}
Let $M$ be the matrix indexed by $\Grid\times\Grid$ with entries
$M_{z,z'} \coloneqq g(z-z')$. It is constant along the orbits of the diagonal
action of $\Grid$ on itself, making it a $\Grid$-circulant matrix, the
$d$-dimensional analogue of a circulant. Every character is an eigenvector of
$M$: substituting
$y \coloneqq z - z'$ and using \eqref{eq:additive},
\begin{equation}
\label{eq:eigen}
\begin{split}
  (M\chi_\omega)(z)
  &\;=\; \sum_{z' \in \Grid} g(z-z')\,\chi_\omega(z')
   \;=\; \sum_{y \in \Grid} g(y)\,\chi_\omega(z-y) \\[2pt]
  &\;=\; \chi_\omega(z) \sum_{y \in \Grid} g(y)\,\overline{\chi_\omega(y)}
   \;=\; \lvert\Grid\rvert\,\hat g(\omega)\,\chi_\omega(z) .
\end{split}
\end{equation}
The $\lvert\Grid\rvert$ characters are a basis, so \eqref{eq:eigen} diagonalizes $M$ and its
eigenvalues are $\{\lvert\Grid\rvert\hat g(\omega)\}_{\omega \in \Gd}$. Exactly
$\lvert\Omega\rvert$ of them are nonzero, so $\rank M = \lvert\Omega\rvert$.

\emph{Lower bound.}
Suppose $p$ and $Q$ are as in the statement and let $P$ be the $n \times \lvert\Grid\rvert$
matrix whose columns are the vectors $p_z$. Then $M = P^\top Q P$, so
\[
  \lvert\Omega\rvert \;=\; \rank M \;\leq\; \rank Q \;\leq\; n .
\]

\emph{Construction.}
For the converse, partition $\Omega$ into orbits of the involution
$\omega \mapsto -\omega$. Because $g$ is real its spectrum is
conjugate-symmetric, so collecting each orbit into a single real term gives
\begin{equation}
\label{eq:realform}
  g(\delta) \;=\; \sum_{\text{orbits } \{\omega,-\omega\}}
    \bigl[\alpha_\omega\cos\varphi_\omega(\delta)
        + \beta_\omega\sin\varphi_\omega(\delta)\bigr] ,
\end{equation}
with $\alpha_\omega, \beta_\omega \in \mathbb{R}$. An orbit is
\emph{self-paired} when $\omega = -\omega$, that is $2\omega = 0$; there
$\varphi_\omega$ takes values in $\{0,\pi\}$, so $\sin\varphi_\omega \equiv 0$
and $\beta_\omega = 0$.

Assign coordinates as follows: a self-paired orbit, which contributes one
element to $\Omega$, receives the single coordinate
$\cos\varphi_\omega(z)$; a conjugate pair, which contributes two, receives the
two coordinates
\begin{equation}
\label{eq:blockcoord}
  e_\omega(z) \;\coloneqq\;
  \begin{bmatrix}
    \cos\varphi_\omega(z) \\[2pt] \sin\varphi_\omega(z)
  \end{bmatrix}
  \;\in\; \mathbb{R}^2 .
\end{equation}
The total is $\lvert\Omega\rvert$ coordinates. Take $Q$ block diagonal with the
blocks
\begin{equation}
\label{eq:block}
  \bigl[\,\alpha_\omega\,\bigr]
  \quad\text{on a self-paired orbit,}
  \qquad
  M_\omega \;\coloneqq\;
  \begin{bmatrix}
    \alpha_\omega & -\beta_\omega \\[2pt]
    \beta_\omega & \phantom{-}\alpha_\omega
  \end{bmatrix}
  \quad\text{on a conjugate pair.}
\end{equation}

It remains to check that each block contributes its term of \eqref{eq:realform}.
Write $A \coloneqq \varphi_\omega(z)$ and $B \coloneqq \varphi_\omega(z')$.
Expanding the quadratic form and applying the angle-difference identities
$\cos(A-B) = \cos A\cos B + \sin A \sin B$ and
$\sin(A-B) = \sin A\cos B - \cos A\sin B$,
\begin{align}
  e_\omega(z)^\top M_\omega\, e_\omega(z')
  &=
  \begin{bmatrix} \cos A & \sin A \end{bmatrix}
  \begin{bmatrix}
    \alpha_\omega & -\beta_\omega \\ \beta_\omega & \phantom{-}\alpha_\omega
  \end{bmatrix}
  \begin{bmatrix} \cos B \\ \sin B \end{bmatrix}
  \notag\\[4pt]
  &= \alpha_\omega\bigl(\cos A\cos B + \sin A \sin B\bigr)
   + \beta_\omega \bigl(\sin A\cos B - \cos A \sin B\bigr)
  \notag\\[4pt]
  &= \alpha_\omega\cos(A-B) + \beta_\omega\sin(A-B) .
  \label{eq:blockexpand}
\end{align}
By \eqref{eq:additive}, $A - B = \varphi_\omega(z) - \varphi_\omega(z') =
\varphi_\omega(z-z')$, so \eqref{eq:blockexpand} is exactly the
$\omega$-term of \eqref{eq:realform} evaluated at $z-z'$. The self-paired blocks
reduce to the same computation with the sine coordinate absent. Summing over the
orbits gives $p_z^\top Q\, p_{z'} = g(z-z')$.

Finally, $M_\omega = r_\omega R(\theta_\omega)$ with
$r_\omega = \sqrt{\alpha_\omega^2 + \beta_\omega^2}$,
$\theta_\omega = \operatorname{atan2}(\beta_\omega, \alpha_\omega)$ and
$R(\theta)$ the rotation by $\theta$, so the $2\times2$ blocks are scaled
rotations, as claimed.
\end{proof}

\begin{rmk}
The lower bound establishes the addressing dimension $\dS$ (\S\ref{app:addr:fourier}) as an intrinsic property of the positional bias. In \S\ref{app:addr:interface}, the scaled-rotation construction \eqref{eq:block} allows an attention head to select its routing offset via a linear map applied to a shared positional table, minimizing parameter count.
\end{rmk}

\subsection{Isolating a set of offsets}
\label{app:addr:isolation}

\begin{definition}[Isolation]
\label{def:isolating}
Let $\emptyset \neq \Nset \subsetneq \Grid$ and $\gamma > 0$. A function
$g : \Grid \to \mathbb{R}$ \emph{isolates $\Nset$ with margin $\gamma$} if
\[
  \min_{\delta\in\Nset} g(\delta)
  \;-\; \max_{\delta\notin\Nset} g(\delta)
  \;\geq\; \gamma .
\]
\end{definition}

Isolation is an ordering condition. Since additive constants leave the softmax invariant, $g$ isolates $\Nset$ if and only if $g + \mathrm{const}$ does. We therefore discard the trivial character $\omega = 0$ and report counts for $\supp\hat g\setminus\{0\}$.

The elementary construction multiplies one factor per value the target set assumes on each axis. It vanishes exactly on the bounding box of $\Nset$ and is strictly positive elsewhere, reducing the problem to the box. When the box coincides with the target set, it yields an exact solution (e.g., Construction~I in Appendix~\ref{app:constructions}).

\begin{lemma}[Box kernels]
\label{lem:degreeone}
Let $\emptyset \neq \Nset \subseteq \Grid$, let
$V_j \coloneqq \{u_j : u \in \Nset\}$ be the projection of $\Nset$ to axis $j$,
and put
\[
  P_j(x) \;\coloneqq\; \prod_{v \in V_j}
    \Bigl(1 - \cos\bigl(2\pi (x_j - v)/W_j\bigr)\Bigr),
  \qquad
  Q \;\coloneqq\; \sum_{j=1}^{d} P_j ,
  \qquad
  B \;\coloneqq\; V_1 \times \dots \times V_d .
\]
Then $Q \geq 0$ vanishes exactly on $B$, which contains $\Nset$ and satisfies
$\lvert B\rvert \leq \lvert\Nset\rvert^{d}$, and
\[
  \supp\hat Q \setminus\{0\} \;\subseteq\;
  \{c\,e_j : j \leq d,\; 1 \leq \lvert c\rvert \leq \lvert\Nset\rvert\},
\]
at most $2d\lvert\Nset\rvert$ characters. Consequently:
\begin{enumerate}[label=(\roman*),leftmargin=2.2em,itemsep=1pt,topsep=2pt]
  \item if $B = \Nset \neq \Grid$, then $-Q$ isolates $\Nset$ with margin
        $\varepsilon \coloneqq \min_{x \notin B} Q(x) > 0$ at no more than
        $2d\lvert\Nset\rvert$ nontrivial frequencies;
  \item a single offset, $\Nset = \{\delta\}$, is that case: $B = \{\delta\}$,
        $Q(x) = \sum_j\bigl(1-\cos(2\pi(x_j-\delta_j)/W_j)\bigr)$ is of degree
        one on each axis, and $\{\delta\}$ is isolated at \emph{at most} $2d$
        frequencies, the pair $(\cos,\sin)$ on each axis, with margin
        $1 - \cos(2\pi/\Wmax)$, where $\Wmax \coloneqq \max_j W_j$ is the
        largest side length. The count is exactly $2d$ when every
        $W_j \geq 3$; an axis with $W_j = 2$ contributes one character instead
        of two, since there $e_j = -e_j$ and the orbit is self-paired.
\end{enumerate}
\end{lemma}

\begin{proof}
\emph{Non-negativity and the zero set.}
Each factor of $P_j$ has the form $1 - \cos\theta \in [0,2]$ and vanishes exactly
when $\theta \in 2\pi\mathbb{Z}$, that is when $x_j = v$ in $\mathbb{Z}_{W_j}$.
So $P_j \geq 0$, and $P_j(x) = 0$ if and only if $x_j \in V_j$. A sum of
non-negative terms vanishes only where every term does, so $Q \geq 0$ and
\[
  Q(x) = 0
  \quad\Longleftrightarrow\quad
  x_j \in V_j \text{ for every } j
  \quad\Longleftrightarrow\quad
  x \in B .
\]

\emph{The box contains the target set, and is not much larger.}
If $u \in \Nset$ then $u_j \in V_j$ for every $j$ by the definition of $V_j$, so
$\Nset \subseteq B$. Each $V_j$ is the image of $\Nset$ under a coordinate
projection, so $\lvert V_j\rvert \leq \lvert\Nset\rvert$ and
$\lvert B\rvert = \prod_j \lvert V_j\rvert \leq \lvert\Nset\rvert^{d}$.

\emph{The spectrum.}
Fix $j$. Each factor of $P_j$ is a function of $x_j$ alone whose spectrum lies
in $\{0, \pm e_j\}$, a trigonometric polynomial of degree $1$ in $x_j$. The
product-to-sum identity
\[
  \cos(a)\cos(b) \;=\; \tfrac12\bigl[\cos(a-b) + \cos(a+b)\bigr]
\]
and its two companions show that multiplying a trigonometric polynomial of
degree $m$ by one of degree $1$ yields one of degree at most $m+1$. By induction
on the $\lvert V_j\rvert$ factors, $P_j$ is a trigonometric polynomial in $x_j$
of degree at most $\lvert V_j\rvert \leq \lvert\Nset\rvert$, so
\[
  \supp \hat P_j \;\subseteq\;
  \{c\,e_j : \lvert c\rvert \leq \lvert\Nset\rvert\} .
\]
Summing over $j$ gives the stated inclusion for $Q$. Discarding $\omega = 0$,
each axis contributes at most the $2\lvert\Nset\rvert$ characters
$\pm e_j, \pm 2e_j, \dots, \pm\lvert\Nset\rvert e_j$, for at most
$2d\lvert\Nset\rvert$ in all.

\emph{Part (i).}
Suppose $B = \Nset$. Then $Q$ vanishes on $\Nset$, and off $\Nset$ it is
positive; since $\Grid$ is finite, $\varepsilon = \min_{x\notin B} Q(x)$ is a
minimum over a finite set on which $Q$ does not vanish, hence
$\varepsilon > 0$. Therefore
\[
  -Q(x) = 0 \quad (x \in \Nset),
  \qquad
  -Q(x) \leq -\varepsilon \quad (x \notin \Nset),
\]
which is isolation with margin $\varepsilon$. The frequency count is the one
just established, and negation does not change a spectrum.

\emph{Part (ii).}
Let $\Nset = \{\delta\}$. Then $V_j = \{\delta_j\}$ is a single value for each
$j$, so each $P_j$ consists of one factor, $B = V_1\times\cdots\times V_d =
\{\delta\}$, and
\[
  Q(x) \;=\; \sum_{j=1}^{d}\Bigl(1 - \cos\bigl(2\pi(x_j-\delta_j)/W_j\bigr)\Bigr).
\]
Each summand has spectrum $\{0, \pm e_j\}$, so $\supp\hat Q\setminus\{0\}$ is
contained in $\{\pm e_j\}_{j\leq d}$, one $(\cos,\sin)$ pair per axis and hence
at most $2d$ elements. The containment is an equality when every $W_j \geq 3$.
If $W_j = 2$ then $2e_j = 0$ in $\Gd$, so $e_j$ and $-e_j$ are the same
character, $\varphi_{e_j}$ takes values in $\{0,\pi\}$, and that axis
contributes a single self-paired character, so the count drops by one.
Lemma~\ref{lem:realization} then assigns that axis one
coordinate, so the realized width matches the reduced count. Padding to
$2d$ recovers a single uniform formula where one is preferred.
Since $B = \Nset$, part~(i) applies, and it remains to bound $\varepsilon$ from
below. Let $x \neq \delta$. Then $x_{j} \neq \delta_{j}$ for some axis $j$, so
$x_{j} - \delta_{j}$ is a nonzero residue modulo $W_{j}$ and its phase
$2\pi(x_{j}-\delta_{j})/W_{j}$ is at angular distance at least $2\pi/W_{j}$ from
$0$. Cosine is decreasing in the distance to $0$ on $[0,\pi]$, so that summand
is at least $1 - \cos(2\pi/W_{j})$, and the remaining summands are non-negative:
\[
  Q(x) \;\geq\; 1 - \cos\bigl(2\pi/W_{j}\bigr)
        \;\geq\; 1 - \cos\bigl(2\pi/\Wmax\bigr) ,
\]
the second inequality because $W_{j} \leq \Wmax$ makes the angle $2\pi/W_{j}$ no
smaller than $2\pi/\Wmax$, both lying in $(0,\pi]$ whenever $W_j \geq 2$. Hence
$\varepsilon \geq 1 - \cos(2\pi/\Wmax)$.
\end{proof}

\begin{exmp}
Lemma~\ref{lem:degreeone} provides a worst-case upper bound. Numerical optimization yields much tighter results for standard neighborhoods. For example, on $\mathbb{Z}_6^2$, the von Neumann, Moore, and eight-offset routing sets $\Uset$ can each be isolated using only four characters $\{\pm e_1, \pm e_2\}$, significantly below the guaranteed $32$ characters. This setting is analyzed thoroughly in \S\ref{app:addr:instance}.
\end{exmp}

If the bounding box strictly contains the target set, the box kernel constitutes the first stage of a two-stage probabilistic construction.

\begin{theorem}[Grid-independent spatial isolation]
\label{thm:isolation}
Let $\emptyset \neq \Nset \subsetneq \Grid$ and $\gamma > 0$. There is a
$g : \Grid \to \mathbb{R}$ isolating $\Nset$ with margin $\gamma$ and
\[
  \lvert\supp\hat g \setminus \{0\}\rvert
  \;\leq\; 2d\,\lvert\Nset\rvert \;+\; 2\nu ,
  \qquad
  \nu \;\coloneqq\;
  \bigl\lceil 40\,\lvert\Nset\rvert\,(d\ln\lvert\Nset\rvert + 1)\bigr\rceil ,
\]
that is $O\bigl(\lvert\Nset\rvert\, d \log \lvert\Nset\rvert\bigr)$ frequencies.
The count is independent of $T$ and $\lvert \Grid \rvert$.
\end{theorem}

\begin{proof}
\emph{Stage 1: a box around the target set.}
Let $Q$ and $B$ be the box kernel and box of Lemma~\ref{lem:degreeone}. Then
$Q \geq 0$ vanishes exactly on $B$, we have $\Nset \subseteq B$ and
$\lvert B\rvert \leq \lvert\Nset\rvert^{d}$, and $Q$ has at most
$2d\lvert\Nset\rvert$ nontrivial characters. It remains to separate $\Nset$
from $B \setminus \Nset$, a set of at most $\lvert\Nset\rvert^{d}$ points; this
is the only step at which the size of the grid could have entered, and Stage 3
is where it is kept out.

\emph{Stage 2: one random character.}
Fix $x \in B$, let $\omega$ be uniform on $\Gd$, and put
\begin{equation}
\label{eq:Zdef}
  Z(\omega) \;\coloneqq\; \sum_{u \in \Nset} \cos\varphi_\omega(x-u)
  \;=\; \Re \sum_{u\in\Nset} \chi_\omega(x-u) .
\end{equation}
By \eqref{eq:orthogonality} the term indexed by $u$ has expectation $1$ if
$x = u$ and $0$ otherwise, so
\begin{equation}
\label{eq:Zmean}
  \mathbb{E}Z \;=\; \lvert\{u \in \Nset : u = x\}\rvert
  \;=\;
  \begin{cases}
    1 & x \in \Nset, \\
    0 & x \in B \setminus \Nset .
  \end{cases}
\end{equation}
The two cases differ by $1$, and the rest of the proof turns that gap into a
deterministic one.

For the second moment, apply the product-to-sum identity
$\cos a \cos b = \tfrac12[\cos(a-b) + \cos(a+b)]$ to each pair of terms of
\eqref{eq:Zdef} and use $\varphi_\omega(x-u) \pm \varphi_\omega(x-v) =
\varphi_\omega\bigl((x-u)\pm(x-v)\bigr)$ from \eqref{eq:additive}:
\begin{equation}
\label{eq:Zsecond}
  \mathbb{E}[Z^{2}]
  \;=\;
  \tfrac12\sum_{u,v \in \Nset}
    \mathbb{E}\cos\varphi_\omega(v-u)
  \;+\;
  \tfrac12\sum_{u,v \in \Nset}
    \mathbb{E}\cos\varphi_\omega(2x-u-v) .
\end{equation}
Each expectation in \eqref{eq:Zsecond} is $1$ or $0$, again by
\eqref{eq:orthogonality}. In the first sum the surviving pairs are those with
$v = u$, of which there are $\lvert\Nset\rvert$. In the second the surviving
pairs are those with $u + v = 2x$; here $u$ determines $v$, so at most
$\lvert\Nset\rvert$ pairs survive. Hence
\begin{equation}
\label{eq:Zvar}
  \operatorname{Var} Z \;\leq\; \mathbb{E}[Z^{2}]
  \;\leq\; \tfrac12\lvert\Nset\rvert + \tfrac12\lvert\Nset\rvert
  \;=\; \lvert\Nset\rvert .
\end{equation}
Finally $\lvert Z\rvert \leq \lvert\Nset\rvert$ termwise, so with
\eqref{eq:Zmean},
\begin{equation}
\label{eq:Zrange}
  \lvert Z - \mathbb{E}Z\rvert \;\leq\; \lvert\Nset\rvert + 1
  \;\leq\; 2\lvert\Nset\rvert .
\end{equation}

\emph{Stage 3: sampling, and why the grid does not enter.}
Draw $\omega^{(1)},\dots,\omega^{(\nu)}$ independent and uniform on $\Gd$ and
put $R \coloneqq \nu^{-1}\sum_{i\leq\nu} Z(\omega^{(i)})$. Bernstein's
inequality for independent, mean-zero summands with variance proxy $\sigma^2$
and range bound $M$ reads
\[
  \mathbb{P}\bigl(\lvert R - \mathbb{E}R\rvert \geq t\bigr)
  \;\leq\;
  2\exp\!\left(-\frac{\nu t^{2}}{2\bigl(\sigma^{2} + Mt/3\bigr)}\right) .
\]
Substituting $\sigma^{2} = \lvert\Nset\rvert$ from \eqref{eq:Zvar},
$M = 2\lvert\Nset\rvert$ from \eqref{eq:Zrange} and $t = \tfrac14$,
\begin{equation}
\label{eq:bernstein}
\begin{split}
  \mathbb{P}\bigl(\lvert R - \mathbb{E}R\rvert \geq \tfrac14\bigr)
  &\;\leq\;
  2\exp\!\left(-\frac{\nu/16}
    {2\bigl(\lvert\Nset\rvert + \lvert\Nset\rvert/6\bigr)}\right)
  \;=\;
  2\exp\!\left(-\frac{\nu/16}{\tfrac73\lvert\Nset\rvert}\right) \\[4pt]
  &\;=\;
  2\exp\!\left(-\frac{3\nu}{112\,\lvert\Nset\rvert}\right).
\end{split}
\end{equation}
This bound holds for each fixed $x \in B$. Taking a union bound over the at most
$\lvert\Nset\rvert^{d}$ points of $B$, the probability that
$\lvert R - \mathbb{E}R\rvert \geq \tfrac14$ somewhere on $B$ is less than $1$
as soon as
\[
  2\,\lvert\Nset\rvert^{d}
  \exp\!\left(-\frac{3\nu}{112\lvert\Nset\rvert}\right) < 1
  \qquad\Longleftrightarrow\qquad
  \nu \;>\; \tfrac{112}{3}\,\lvert\Nset\rvert
    \bigl(d\ln\lvert\Nset\rvert + \ln 2\bigr) .
\]
The $\nu$ of the statement satisfies this, since $40 > 112/3 = 37.33\ldots$ and
$1 > \ln 2$. The grid enters neither \eqref{eq:bernstein}, whose
constants come from \eqref{eq:Zvar} and \eqref{eq:Zrange} alone, nor the union
bound, which ranges over $B$ alone; the final count is therefore independent of
$W_1,\dots,W_d$.

Some realization of $\omega^{(1)},\dots,\omega^{(\nu)}$ therefore has
\begin{equation}
\label{eq:Rprops}
  R \geq \tfrac34 \text{ on } \Nset,
  \qquad
  R \leq \tfrac14 \text{ on } B\setminus\Nset ,
\end{equation}
by \eqref{eq:Zmean}; fix it. Being an average of the real functions
\eqref{eq:Zdef}, $R$ is real, and $\supp\hat R \setminus\{0\}$ has at most
$2\nu$ elements: the sampled frequencies and their conjugates. Termwise,
$\lvert R\rvert \leq \lvert\Nset\rvert$ everywhere on $\Grid$.

\emph{Stage 4: synthesis.}
Let $\varepsilon \coloneqq \min_{x \notin B} Q(x) > 0$, a minimum over a finite
set on which $Q$ does not vanish, and put
\begin{equation}
\label{eq:g0}
  g_0 \;\coloneqq\; R \;-\; \frac{\lvert\Nset\rvert}{\varepsilon}\,Q
       \;-\; \tfrac12 .
\end{equation}
The three regions of $\Grid$ are handled in turn.
\begin{enumerate}[label=(\alph*),leftmargin=2.2em,itemsep=2pt,topsep=3pt]
  \item $x \in \Nset$. Then $Q(x) = 0$ because $\Nset \subseteq B$, so
        $g_0(x) = R(x) - \tfrac12 \geq \tfrac34 - \tfrac12 = \tfrac14$
        by \eqref{eq:Rprops}.
  \item $x \in B\setminus\Nset$. Again $Q(x) = 0$, so
        $g_0(x) = R(x) - \tfrac12 \leq \tfrac14 - \tfrac12 = -\tfrac14$.
  \item $x \notin B$. Then $Q(x) \geq \varepsilon$, so
        $(\lvert\Nset\rvert/\varepsilon)Q(x) \geq \lvert\Nset\rvert$, while
        $R(x) \leq \lvert\Nset\rvert$; hence
        $g_0(x) \leq \lvert\Nset\rvert - \lvert\Nset\rvert - \tfrac12
        = -\tfrac12 < -\tfrac14$.
\end{enumerate}
So $g_0 \geq \tfrac14$ on $\Nset$ and $g_0 \leq -\tfrac14$ off it, and
$g \coloneqq 4\gamma\,g_0$ isolates $\Nset$ with margin
$4\gamma(\tfrac14+\tfrac14) = 2\gamma \geq \gamma$. Its spectrum is contained in
$\supp\hat R \cup \supp\hat Q \cup \{0\}$, so
$\lvert\supp\hat g\setminus\{0\}\rvert \leq 2\nu + 2d\lvert\Nset\rvert$, which
is the stated count.

Two degenerate cases are worth recording. If $B = \Grid$, case (c) is vacuous,
no $x$ lies outside $B$, and any $\varepsilon > 0$ serves in \eqref{eq:g0}. If
$B = \Nset$, case (b) is vacuous and the random part is unnecessary altogether:
Lemma~\ref{lem:degreeone}(i) already isolates $\Nset$, at
$2d\lvert\Nset\rvert$ frequencies and with no $\nu$ term.
\end{proof}

\begin{rmk}
Because spectral support cannot exceed $\lvert\Grid\rvert$, Theorem~\ref{thm:isolation} is informative when $\lvert\Nset\rvert\,d\log\lvert\Nset\rvert \ll \lvert\Grid\rvert$, the regime where grid independence is critical. The theorem establishes a grid-independent scaling law irrespective of offset topology.
\end{rmk}

\subsection{Attention concentration from a score gap}
\label{app:addr:scale}

While isolation strictly orders target offsets above others, securing a quantitative guarantee on attention mass requires scaling. This is the only point in the analysis where the absolute scale of the positional kernel matters; previous results depend solely on the Fourier spectral support.

\begin{corollary}[Softmax concentration]
\label{cor:scale}
Let a softmax be taken over a finite set of keys scored $\Lambda g$, where
$\Lambda > 0$ and $g$ isolates a target subset of the keys with margin $\gamma$.
Suppose at least one target key is present and at most $N_{\neq}$ non-target
keys are. Then the attention mass placed off the target subset is at most
$N_{\neq}e^{-\Lambda\gamma}$, and is therefore at most $\varepsilon$ as soon as
\begin{equation}
\label{eq:scale}
  \Lambda\gamma \;\geq\; \log(N_{\neq}/\varepsilon),
  \qquad\text{equivalently}\qquad
  \Lambda \;\geq\; \gamma^{-1}\log(N_{\neq}/\varepsilon) .
\end{equation}
\end{corollary}

\begin{proof}
Write $s_\star \coloneqq \min_{\text{target}} g$ and
$s_{\neq} \coloneqq \max_{\text{non-target}} g$, so that isolation reads
$s_\star - s_{\neq} \geq \gamma$ and, after scaling,
$\Lambda s_\star - \Lambda s_{\neq} \geq \Lambda\gamma$. Split the partition
function into its target and non-target parts,
\[
  A \;\coloneqq\; \sum_{\text{target keys}} e^{\Lambda g},
  \qquad
  B \;\coloneqq\; \sum_{\text{non-target keys}} e^{\Lambda g} .
\]
At least one target key is present, so $A \geq e^{\Lambda s_\star}$. Every
non-target key contributes at most $e^{\Lambda s_{\neq}}$, and there are at most
$N_{\neq}$ of them, so
\[
\begin{aligned}
  B \;&\leq\; N_{\neq}e^{\Lambda s_{\neq}}
    \;=\; N_{\neq}e^{\Lambda s_\star}e^{-\Lambda(s_\star - s_{\neq})} \\[2pt]
    \;&\leq\; N_{\neq}e^{-\Lambda\gamma}\,e^{\Lambda s_\star}
    \;\leq\; N_{\neq}e^{-\Lambda\gamma}A .
\end{aligned}
\]
The mass off the target subset is $B/(A+B)$, and since $A, B \geq 0$,
\[
  \frac{B}{A+B} \;\leq\; \frac{B}{A} \;\leq\; N_{\neq}e^{-\Lambda\gamma} .
\]
This is at most $\varepsilon$ exactly when $e^{\Lambda\gamma} \geq
N_{\neq}/\varepsilon$, which is \eqref{eq:scale}.
\end{proof}

\subsection{Selecting spatial offsets in the preceding time step}
\label{app:addr:interface}

This section details the spatial positional block $B_2$ (Definition~\ref{def:stream}), establishing the interface for Appendix~\ref{app:constructions}.
Lemma~\ref{lem:realization} connects spectral counts to table width. A translation-invariant positional bias yields a score matrix $P^\top W_Q^{(1)\top}W_K^{(1)} P$ (where $P$ is the positional table). Setting $Q = W_Q^{(1)\top}W_K^{(1)}$ confirms the table requires exactly $\lvert\supp\hat g\rvert$ coordinates. Discarding the trivial character, this is the addressing dimension $\dS$. An attention head selects its routed offset by rotating the $2\times2$ blocks in \eqref{eq:block}.

\begin{corollary}[Spatiotemporal addressing interface]
\label{cor:head}
Let $\emptyset \neq \Nset \subseteq \Grid$ be a set of offsets and let a
Layer-1 head be required to score, from its query's own cell $\kappa(n)$, the
cells at $\kappa(n) + \delta$ for $\delta \in \Nset$, in the time step before
the target cell's, strictly above every other causally visible cell key. There
is a positional table $p_n \in \mathbb{R}^{\dpos}$, indexed
by sequence position, and projections $W_Q, W_K$ whose pre-softmax score does
so with a margin $\gamma > 0$, where
\[
  \dpos \;=\; \dS + 2,
  \qquad
  \dS \;\leq\; 2d\lvert\Nset\rvert + 2\nu
  \;=\; O\bigl(\lvert\Nset\rvert\,d\log\lvert\Nset\rvert\bigr),
\]
and $\dS = 2d$ when $\Nset$ is a single offset. The width is independent of the
volume and of the horizon.

The two temporal coordinates carry a scale $\lambda$, which must satisfy
\begin{equation}
\label{eq:lambda}
\begin{gathered}
  \lambda \;\geq\; \frac{R + \gamma}{\gamma_T},
  \\[4pt]
  R \;\coloneqq\; \max_{\delta\in\Grid} g(\delta)
                  - \min_{\delta\in\usucc-\Nset} g(\delta) \;\geq\; 0,
  \qquad
  \gamma_T \;\coloneqq\; 1 - \cos(2\pi/T) ,
\end{gathered}
\end{equation}
where $g$ is the spatial kernel. The scale $\lambda$ enters no width.
\end{corollary}

\begin{proof}
\emph{Stage 1: the group and the target set.}
A query at sequence position $n$ has virtual time step $\tau(n)$ and virtual
cell $\kappa(n)$, both defined in Appendix~\ref{app:serialization:formal}, and
its target cell is $\kappa(n) + \usucc$
(Lemma~\ref{lem:successor}). It must score the keys $n'$ with
\[
  \kappa(n') - \kappa(n) \in \Nset
  \qquad\text{and}\qquad
  \tau(n) - \tau(n') = 1
\]
above all others. Stage 5 realizes the spatial score in the
\emph{target-minus-key} variable
\[
  \delta \;\coloneqq\; \kappa(n) + \usucc - \kappa(n') ,
\]
and a key at $\kappa(n') = \kappa(n) + \delta'$ has $\delta = \usucc -
\delta'$, so the first condition reads $\delta \in \usucc - \Nset$. Above,
$\usucc - \Nset \coloneqq \{\usucc - \delta' : \delta' \in \Nset\}$, and
$\delta' \mapsto \usucc - \delta'$ is a bijection of $\Grid$, so
$\lvert\usucc - \Nset\rvert = \lvert\Nset\rvert$ and every width bound of the
statement, which is stated in $\lvert\Nset\rvert$, is unaffected. Each
condition constrains a difference, so the ambient group is
$\Grid \times \mathbb{Z}_T$ and the target set is
$(\usucc - \Nset) \times \{1\}$. We build a kernel on that product, one block
per factor.

\emph{Stage 2: the spatial block.}
Let $g$ isolate $\usucc - \Nset$ in $\Grid$ with margin $\gamma$, by
Theorem~\ref{thm:isolation}, and let $\Omega \coloneqq \supp\hat g\setminus\{0\}$
with $\lvert\Omega\rvert = \dS$. Removing the trivial character replaces $g$ by
the centered kernel $g - \hat g(0)$, which isolates the same set with the same
margin (Definition~\ref{def:isolating}); we write $g$ for it from here on. Let
$\Omega_+ \subseteq \Omega$ contain one representative of each orbit
$\{\omega,-\omega\}$. Write $g$ in the real form
\eqref{eq:realform} with coefficients $\alpha_\omega, \beta_\omega$, and take the
coordinates and blocks Lemma~\ref{lem:realization} supplies, indexed by the
representatives:
\begin{equation}
\label{eq:spatialtable}
  \phi(z) \;\coloneqq\;
  \bigl(e_\omega(z)\bigr)_{\omega \in \Omega_+}
  \;=\;
  \Bigl(
    \cos\varphi_\omega(z),\; \sin\varphi_\omega(z)
  \Bigr)_{\omega \in \Omega_+}
  \;\in\; \mathbb{R}^{\dS},
  \qquad
  M_\omega \;=\;
  \begin{bmatrix}
    \alpha_\omega & -\beta_\omega \\[2pt]
    \beta_\omega & \phantom{-}\alpha_\omega
  \end{bmatrix} ,
\end{equation}
with the sine coordinate omitted for a self-paired representative. This gives
one $2\times2$ block per conjugate pair and one $1\times1$ block per self-paired
character, for $\lvert\Omega\rvert=\dS$ coordinates in total. By \eqref{eq:blockexpand},
$\phi(z)^\top \operatorname{blkdiag}(M_\omega)_\omega\, \phi(z') = g(z - z')$.

\emph{Stage 3: the temporal block.}
Step selection is isolation of the single displacement $1$ in $\mathbb{Z}_T$,
which Lemma~\ref{lem:degreeone}(ii) covers with $d = 1$, $W_1 = T$ and
$\Nset = \{1\}$. Discarding the constant term, which
Definition~\ref{def:isolating} permits, the kernel is
\begin{equation}
\label{eq:timekernel}
  h(s) \;\coloneqq\; \cos\bigl(2\pi(s-1)/T\bigr),
  \qquad
  h(1) = 1,
  \qquad
  h(s) \leq \cos(2\pi/T) = 1 - \gamma_T \ \ (s \neq 1),
\end{equation}
so $h$ isolates $\{1\}$ with margin $\gamma_T = 1 - \cos(2\pi/T) > 0$, whatever
$T \geq 2$ is. Its spectrum is $\{\pm 1\} \subseteq \widehat{\mathbb{Z}_T}$,
at most two characters (one when $T=2$, where $1=-1$), so
Lemma~\ref{lem:realization} realizes it in at most two coordinates. We always
use two, padding with the identically zero sine coordinate when $T=2$:
\begin{equation}
\label{eq:timetable}
  \psi(\tau) \;\coloneqq\;
  \begin{bmatrix}
    \cos(2\pi\tau/T) \\[2pt] \sin(2\pi\tau/T)
  \end{bmatrix}
  \in \mathbb{R}^2,
  \qquad
  M_{\mathrm{time}} \;\coloneqq\;
  \begin{bmatrix}
    \cos(2\pi/T) & -\sin(2\pi/T) \\[2pt]
    \sin(2\pi/T) & \phantom{-}\cos(2\pi/T)
  \end{bmatrix}
  \;=\; R(2\pi/T),
\end{equation}
the rotation by one time step: reading \eqref{eq:blockexpand} with
$\alpha_\omega = \cos(2\pi/T)$ and $\beta_\omega = \sin(2\pi/T)$ gives
$\psi(\tau)^\top M_{\mathrm{time}}\psi(\tau') = \cos\bigl(2\pi(\tau-\tau'-1)/T\bigr)
= h(\tau-\tau')$.

\emph{Stage 4: the table and the bilinear form.}
Concatenate the two blocks. The table entry at sequence position $n$ is
\begin{equation}
\label{eq:table}
  p_n \;\coloneqq\;
  \begin{bmatrix}
    \psi(\tau(n)) \\[2pt] \phi(\kappa(n))
  \end{bmatrix}
  \;\in\; \mathbb{R}^{\dS + 2},
  \qquad
  W_Q^\top W_K \;\coloneqq\;
  \operatorname{blkdiag}\bigl(
    \lambda M_{\mathrm{time}},\,
    (M_\omega)_{\omega\in\Omega_+}
  \bigr),
\end{equation}
evaluated at the virtual time step and virtual cell of the position, which
Appendix~\ref{app:serialization:formal} defines for cell tokens and separators
alike. Because $W_Q^\top W_K$ is block diagonal with the blocks aligned to the
blocks of $p_n$, the quadratic form is the sum of the per-block forms of Stages
2 and 3:
\begin{equation}
\label{eq:score}
\begin{split}
  p_n^\top \bigl(W_Q^\top W_K\bigr) p_{n'}
  &\;=\;
  \lambda\, \psi(\tau(n))^\top M_{\mathrm{time}}\, \psi(\tau(n'))
  \;+\;
  \sum_{\omega\in\Omega_+} e_\omega(\kappa(n))^\top M_\omega\, e_\omega(\kappa(n'))
  \\[4pt]
  &\;=\;
  g\bigl(\kappa(n)-\kappa(n')\bigr)
  \;+\; \lambda\, h\bigl(\tau(n)-\tau(n')\bigr).
\end{split}
\end{equation}
This is the form Appendix~\ref{app:separators} quotes as
$p_n = (\phi(\kappa(n)), \psi(\tau(n)))$.

\emph{Stage 5: the successor shift costs no coordinate.}
The spatial argument used in Stage 1 is
$\kappa(n) + \usucc - \kappa(n')$, whereas \eqref{eq:score} delivers
$g(\kappa(n) - \kappa(n'))$. The discrepancy is the constant $\usucc$, and by
\eqref{eq:additive} a constant displacement is a fixed rotation of every
character:
\[
  \chi_\omega(z + \usucc) \;=\; \chi_\omega(\usucc)\,\chi_\omega(z),
  \qquad\text{that is}\qquad
  \varphi_\omega(z + \usucc) = \varphi_\omega(z) + \theta_\omega,
  \quad \theta_\omega \coloneqq \varphi_\omega(\usucc) .
\]
The product $M_\omega R(-\theta_\omega)$ is again a scaled rotation, by the
last paragraph of the proof of Lemma~\ref{lem:realization}. Replacing each
block $M_\omega$ of \eqref{eq:table} by it turns \eqref{eq:blockexpand}
into $\alpha_\omega\cos(A - B + \theta_\omega) + \beta_\omega\sin(A - B +
\theta_\omega)$, so the spatial term of \eqref{eq:score} becomes
$g(\kappa(n) + \usucc - \kappa(n'))$ as required. The table itself is unchanged
and no coordinate is added. By Lemma~\ref{lem:successor} the offset is the same
constant $\usucc$ for every position which has a target cell, separators included, so
one choice of blocks serves all of them; the separator case is recorded after
this proof.

\emph{Stage 6: the margin.}
Write the score of key $n'$ at query $n$ as $\Lambda\,G(\delta, s)$ with
$\delta \coloneqq \kappa(n) + \usucc - \kappa(n')$,
$s \coloneqq \tau(n) - \tau(n')$ and
\[
  G(\delta, s) \;\coloneqq\; g(\delta) \;+\; \lambda\,h(s) ,
\]
which Stages 4 and 5 realize. Put $m \coloneqq \min_{\delta\in\usucc-\Nset}
g(\delta)$, so that $\max_{\delta\notin\usucc-\Nset} g \leq m - \gamma$ by isolation
and $\max_{\delta\in\Grid} g = m + R$ by the definition of $R$ in
\eqref{eq:lambda}. Three cases exhaust the keys.
\begin{enumerate}[label=(\alph*),leftmargin=2.2em,itemsep=2pt,topsep=3pt]
  \item \emph{Target keys}, $\delta \in \usucc-\Nset$ and $s = 1$. Then $h(s) = 1$ and
        $G \geq m + \lambda$.
  \item \emph{Right step, wrong offset}, $\delta \notin \usucc-\Nset$ and $s = 1$. Then
        $G \leq (m - \gamma) + \lambda$, which is below case (a) by $\gamma$.
  \item \emph{Wrong step}, $s \neq 1$. Then $h(s) \leq 1 - \gamma_T$ by
        \eqref{eq:timekernel} and $g(\delta) \leq m + R$, so
        \[
          G \;\leq\; m + R + \lambda(1 - \gamma_T)
            \;=\; m + \lambda + \bigl(R - \lambda\gamma_T\bigr)
            \;\leq\; m + \lambda - \gamma ,
        \]
        the last step by the choice \eqref{eq:lambda} of $\lambda$, which is
        exactly $\lambda\gamma_T \geq R + \gamma$. Again this is below case (a)
        by $\gamma$.
\end{enumerate}
So $G$ isolates $(\usucc-\Nset)\times\{1\}$ with margin $\gamma$, and the pre-softmax
score $\Lambda G$ isolates it with margin $\Lambda\gamma$. Case (c) is the one
which forces the scaling: with $\lambda = 1$ the bound there is
$m + 1 + R - \gamma_T$, which exceeds the target floor $m + 1$ whenever
$R > \gamma_T$, that is, whenever the spatial kernel varies across $\usucc-\Nset$ by
more than one time step is worth. The causal mask is not needed for any of this,
since $h$ orders the whole of $\mathbb{Z}_T$, visible part included.

\emph{Widths.}
The two blocks of \eqref{eq:table} occupy orthogonal coordinates, so their ranks
add and $\dpos = \dS + 2$. The count $\dS \leq 2d\lvert\Nset\rvert + 2\nu$ is
Theorem~\ref{thm:isolation}, and $\dS = 2d$ for a single offset is
Lemma~\ref{lem:degreeone}(ii) on $\Grid$. Neither count involves the volume or
the horizon: the grid is absent from Theorem~\ref{thm:isolation} by Stage 3 of
its proof, and the horizon is absent because the temporal block has two
coordinates for every $T$.
The scale $\lambda$ of \eqref{eq:lambda} multiplies a block; it does not widen
one.
\end{proof}

\begin{remark}[Separator keys]
\label{rmk:head-sepkeys}
The guarantee above is stated over cell keys. A separator carries the virtual
address of the cell one step before its successor
(Lemma~\ref{lem:sepsucc}\emph{(i)}), so a separator key can hold the same
$(\tau,\kappa)$ as a target cell, and a score built from $\tau$ and $\kappa$
alone cannot order it below. Proposition~\ref{prop:sepgate-l1} adds a rank-one
gate of strength $\Gamma_1$ which places every separator key $\Lambda\gamma$
below the target floor, extending the corollary to all causally visible keys at
no positional width.
\end{remark}

\begin{rmk}[Boundary cases of the interface]
\label{rmk:head-degenerate}
We clarify three assumptions used above:
\emph{Horizon.} Stage~3 requires $T \geq 2$ so that $\gamma_T = 1-\cos(2\pi/T) > 0$, ensuring the temporal kernel has a nonempty complement. Definition~\ref{def:ca} imposes $T \geq 2$; at $T = 1$, the prediction task is empty.
\emph{Full target set.} While Definition~\ref{def:isolating} and Theorem~\ref{thm:isolation} require $\Nset \subsetneq \Grid$, the case $\Nset = \Grid$ is covered by setting $g \equiv 0$ (so $\dS = 0$ and $\dpos = 2$). The temporal block alone correctly identifies all keys in the previous time step with margin $\gamma = \lambda\gamma_T$, making the interface total in $\Nset$.
\emph{Degenerate axes.} If $W_j = 2$, Lemma~\ref{lem:degreeone}(ii) contributes one character instead of two on that axis, giving $\dS \leq 2d$ without increasing any stated width.
\end{rmk}

\noindent
On the query side, the positional table accommodates separator positions. By Lemma~\ref{lem:sepsucc}, a separator at $n$ carries virtual indices $\kappa(n)$ and $\tau(n)$, targets $\kappa(n) + \usucc$, and causally accesses all parents in step $\tau(n)-1$. Because the proof depends on the query solely through $\kappa$ and $\tau$, \eqref{eq:table} applies verbatim to separator queries without additional coordinates, maintaining the margin $\gamma$. Remark~\ref{rmk:head-sepkeys} covers the key side.

Subsequent sections depend on the spatial kernel $g$, its isolation margin $\gamma$, the positional table width $\dpos = \dS + 2$, and the separator gate of Proposition~\ref{prop:sepgate-l1}. Corollary~\ref{cor:scale} converts this margin and logit scale into an attention mass bound. A caller simply specifies the desired offsets $\Nset$ relative to the query; Stage~1 of the proof handles the translation to the kernel's variable.

\subsection{Example: addressing the 2D von Neumann routing set}
\label{app:addr:instance}

The preceding bounds are worst-case over all subsets $\Nset$. Here, we instantiate the derivation for a specific experimental setting with explicit numerical constants.

\begin{example}[The routing set on $\mathbb{Z}_6^2$]
\label{ex:union}
Let $\Grid = \mathbb{Z}_6^2$, let $D_K$ be the von Neumann neighborhood, let
$\usucc = (0,1)$ and let $T = 12$, so that the routing set
$\Uset = D_K \cup (D_K + \usucc)$ of Definition~\ref{def:routing} has the
$\lvert\Uset\rvert = 8$ elements
\[
  \Uset \;=\;
  \{(0,0),\,(0,1),\,(0,2),\,(0,5),\,(1,0),\,(1,1),\,(5,0),\,(5,1)\} .
\]
Then the anisotropic, half-shifted degree-one kernel
\begin{equation}
\label{eq:instkernel}
  g(\delta) \;=\;
  \tfrac32\cos\bigl(2\pi\delta_1/6\bigr)
  \;+\;
  \cos\bigl(2\pi(\delta_2 - \tfrac12)/6\bigr)
\end{equation}
isolates $\Uset$ with margin $\gamma = \tfrac34$ at the four characters
$\Omega = \{\pm e_1, \pm e_2\}$. Here $\usucc - \Uset = \Uset$, so this is
also the set Stage 2 of Corollary~\ref{cor:head} asks $g$ to isolate when the
corollary is invoked with $\Nset = \Uset$. Thus
\[
  \dS \;=\; 4, \qquad \dpos \;=\; 6 ,
\]
and Corollary~\ref{cor:head} holds at these widths with temporal scale
$\lambda = 13$. At the stream length $N = 504$ of this setting, a logit scale
$\Lambda = 12$ places mass at least $1 - \tfrac1{12}$ on $\Uset$.
\end{example}

\begin{proof}
\emph{The spectrum and the blocks.}
Expanding the half-step shift by the angle-difference identity, with
$2\pi\cdot\tfrac12/6 = \pi/6$,
\[
\begin{aligned}
  \cos\bigl(2\pi(\delta_2-\tfrac12)/6\bigr)
  \;&=\;
  \cos\bigl(\varphi_{e_2}(\delta) - \tfrac\pi6\bigr) \\[2pt]
  \;&=\;
  \cos\tfrac\pi6\,\cos\varphi_{e_2}(\delta)
  \;+\;
  \sin\tfrac\pi6\,\sin\varphi_{e_2}(\delta) ,
\end{aligned}
\]
so \eqref{eq:instkernel} is in the real form \eqref{eq:realform} with
\[
  \alpha_{e_1} = \tfrac32, \quad \beta_{e_1} = 0,
  \qquad
  \alpha_{e_2} = \tfrac{\sqrt3}{2}, \quad \beta_{e_2} = \tfrac12 .
\]
Both coefficients are nonzero, so $\supp\hat g\setminus\{0\} = \{\pm e_1,\pm e_2\}$
and $\dS = \lvert\Omega\rvert = 4$: one conjugate pair per axis, hence the two
$2\times2$ blocks of \eqref{eq:block},
\[
  M_{e_1} \;=\;
  \begin{bmatrix} \tfrac32 & 0 \\[2pt] 0 & \tfrac32 \end{bmatrix}
  \;=\; \tfrac32\,I_2 ,
  \qquad
  M_{e_2} \;=\;
  \begin{bmatrix}
    \tfrac{\sqrt3}{2} & -\tfrac12 \\[4pt]
    \tfrac12 & \phantom{-}\tfrac{\sqrt3}{2}
  \end{bmatrix}
  \;=\; R(\pi/6) .
\]
The shift costs no coordinate: it is absorbed into the rotation angle of the
second block, as Lemma~\ref{lem:realization} shows in general. With
the two temporal coordinates of Stage 3 of Corollary~\ref{cor:head},
$\dpos = \dS + 2 = 6$.

\emph{The margin.}
On $\mathbb{Z}_6$ the cosine takes the four values $1, \tfrac12, -\tfrac12, -1$,
so $g$ takes finitely many values which can be listed. On $\Uset$ it takes three:
\[
\begin{gathered}
  g(0,0) = g(0,1) = \tfrac32 + \tfrac{\sqrt3}{2},
  \qquad
  g(0,2) = g(0,5) = \tfrac32 ,
  \\[4pt]
  g(1,0) = g(1,1) = g(5,0) = g(5,1) = \tfrac34 + \tfrac{\sqrt3}{2} ,
\end{gathered}
\]
so $\min_{\Uset} g = \tfrac32$. Off $\Uset$ the largest value is
$\max_{\delta\notin\Uset} g = \tfrac34$, attained at the four offsets
$(\pm1, 2)$ and $(\pm1, 5)$. Hence
\[
  \gamma \;=\; \min_{\Uset} g - \max_{\notin\Uset} g
       \;=\; \tfrac32 - \tfrac34 \;=\; \tfrac34 .
\]
The half-step shift makes this work. $\Uset$ is asymmetric along the
innermost axis, containing $(0,2)$ but not $(0,4)$, and an unshifted
kernel, being even in $\delta_2$, scores those two offsets alike and cannot
separate them.

\emph{The temporal scale.}
The largest value of $g$ on the whole grid is
$\max_\Grid g = \tfrac32 + \tfrac{\sqrt3}{2}$, attained at $(0,0) \in \Uset$, so
\[
  R \;=\; \max_\Grid g - \min_{\Uset} g \;=\; \tfrac{\sqrt3}{2} \;\approx\; 0.866,
  \qquad
  \gamma_T \;=\; 1 - \cos(2\pi/12) \;=\; 1 - \tfrac{\sqrt3}{2} \;\approx\; 0.134 ,
\]
and \eqref{eq:lambda} requires
\[
  \lambda \;\geq\; \frac{R+\gamma}{\gamma_T}
         \;=\; \frac{\tfrac{\sqrt3}{2} + \tfrac34}{1 - \tfrac{\sqrt3}{2}}
         \;\approx\; 12.06 ,
\]
so $\lambda = 13$ suffices. The gap between $\gamma = \tfrac34$ and
$\gamma_T \approx 0.134$ makes the scaling necessary here: an unscaled sum of
the two kernels would score a key in the wrong step
at up to $\tfrac32 + \tfrac{\sqrt3}{2} + \cos(\pi/6) \approx 3.23$, above the
target floor $\tfrac32 + 1 = \tfrac52$.

\emph{The logit scale.}
The stream length for this setting is
$N = T(\lvert\Grid\rvert + \lvert\Grid\rvert/W_2) = 12\,(36+6) = 504$
(Definition~\ref{def:serialization}). Corollary~\ref{cor:scale} with
$N_{\neq} \leq N$ and off-target mass $\varepsilon = \tfrac1{12}$ gives
\[
  \Lambda \;\geq\; \gamma^{-1}\log(N/\varepsilon)
          \;=\; \tfrac43 \log(6048) \;\approx\; 11.61 ,
\]
so $\Lambda = 12$ places all but $\tfrac1{12}$ of the attention mass on
$\Uset$. We chose $\varepsilon = \tfrac1{12}$ to make the arithmetic concrete.
This $\Nset = \Uset$ addressing serves Construction~II, which decodes each
block exactly and so requires \eqref{eq:onehead-scale} of $\Lambda$ instead.
At $V = 2$, $M = 12$ colors and
$n_{\max} \leq T\lceil\lvert\Grid\rvert/M\rceil = 36$, that condition reads
\[
  \Lambda\gamma \;\geq\; \log\bigl(2V(n_{\max}-1)\bigr)
               \;=\; \log 140 \;\approx\; 4.94 ,
\]
so $\Lambda \geq 6.59$ and the same $\Lambda = 12$ satisfies it.
\end{proof}

\begin{rmk}
Two comparisons relate this example to the general statements.

First, compared to Lemma~\ref{lem:degreeone}: the same four characters isolate the von Neumann and Moore neighborhoods on this grid (margin $\tfrac12$) via the isotropic unshifted kernel $g(\delta) = \cos\varphi_{e_1}(\delta) + \cos\varphi_{e_2}(\delta)$ ($M_{e_1} = M_{e_2} = I_2$).

Second, the theorem becomes informative only when $\lvert\Nset\rvert d\log\lvert\Nset\rvert$ is much less than the volume, the regime where grid independence is critical. For the small grids in Table~\ref{tab:instance-task}, explicit numerical search provides tighter bounds. In $1$D, no search is required: the bounding box equals the set itself, so Lemma~\ref{lem:degreeone}(i) applies directly, avoiding Theorem~\ref{thm:isolation}'s probabilistic bounds, and yielding $\dS \leq 2d\lvert\Uset\rvert = 8$.
\end{rmk}
\section{Colorings for Shared-head decoding}
\label{app:colouring}

Construction~II (Appendix~\ref{app:constructions}) employs a shared Layer-1 head. A key at cell $z'$ projects a one-hot \emph{color code} $\varphi(z') = e_{c(z')} \in \mathbb{R}^{M}$ via a map $c : \Grid \to [M]$. The head's output at query $z$ is a vector $a \in \mathbb{R}^{M+1}$ (comprising $M$ color coordinates and one reference coordinate) which superposes all $\lvert\Uset\rvert$ gathered cells. The cell at offset $\delta$ writes to coordinate $c(z+\delta)$ with a state-dependent weight. For the MLP to recover this cell, no other gathered cell may occupy the same coordinate, requiring:
\[
  \delta \;\longmapsto\; c(z+\delta)
  \qquad\text{is injective on } \Uset, \text{ for every } z \in \Grid .
\]
This condition defines an admissible coloring (Definition~\ref{def:colouring}). This section characterizes these maps and bounds the required number of colors $M$ as a function of $\Grid$ and $\Uset$.

A cell's color is a function of its position, so the key and the query each read it from the positional table. The key reads the color of its own cell, and the query reads the color it expects at each routed offset. The two uses occupy separate coordinates, which Corollary~\ref{cor:colour-interface} counts in parts~(e) and~(b). Appendix~\ref{app:constructions} assigns those coordinates to residual blocks and wires them into the value projection and the MLP.

\subsection{Admissible colorings and the Cayley graph}
\label{app:col:admissible}

\begin{definition}[Admissible coloring]
\label{def:colouring}
A \emph{coloring} is a map $c : \Grid \to [M]$, and it is \emph{admissible for
$\Uset$} if it is injective on every translate $z + \Uset$. Equivalently, $c$ is
a proper coloring of the Cayley graph
\[
  \Gamma_{\Uset} \;\coloneqq\;
  \mathrm{Cay}\bigl(\Grid,\; (\Uset-\Uset)\setminus\{0\}\bigr) ,
\]
whose vertices are the cells and whose edges join cells differing by a nonzero
element of $\Uset - \Uset$. We write $M^\star$ for the least $M$ admitting an
admissible coloring, that is, for the chromatic number of $\Gamma_\Uset$. A
coloring is \emph{periodic} if it is the quotient map
$\chi : \Grid \to \Grid/K$ by a subgroup $K \leq \Grid$, composed with any
bijection $\Grid/K \cong [M_K]$, equivalently if its color classes are exactly
the cosets of $K$, and its number of colors is then
$M_K \coloneqq \lvert\Grid/K\rvert$. The choice of bijection is immaterial:
relabeling preserves both admissibility and the number of colors.
\end{definition}

Because a query gathers uniquely labeled cells, two cells may share a label only if their nonzero difference lies outside $\Uset - \Uset$, reducing the problem to graph coloring.

Periodic colorings are easily constructed, and their admissibility reduces to an algebraic condition:

\begin{lemma}[When a periodic coloring is admissible]
\label{lem:periodic}
A periodic coloring $\chi : \Grid \to \Grid/K$ is admissible for $\Uset$ if and
only if $K \cap (\Uset - \Uset) = \{0\}$.
\end{lemma}

\begin{proof}
The equality $\chi(z+\delta) = \chi(z+\delta')$ holds exactly when $\delta - \delta' \in K$. As $\delta,\delta'$ range over $\Uset$, injectivity on every translate $z+\Uset$ fails if and only if a nonzero element of $\Uset-\Uset$ lies in $K$.
\end{proof}

\S\ref{app:col:instance} uses this criterion to construct admissible colorings for our experimental grids; these are significantly smaller than the worst-case bound below.

\subsection{Bounds on the number of colors}
\label{app:col:bounds}

\begin{lemma}[Bounds on the number of colors]
\label{lem:Mbounds}
\[
  \lvert\Uset\rvert \;\leq\; M^\star \;\leq\;
  \bigl\lvert(\Uset-\Uset)\setminus\{0\}\bigr\rvert + 1 \;\leq\;
  \lvert\Uset\rvert^{2} - \lvert\Uset\rvert + 1 ,
\]
and $M^\star \leq M_K$ for every periodic admissible coloring. Moreover an
admissible coloring with $\lvert(\Uset-\Uset)\setminus\{0\}\rvert + 1$ colors
may be taken \emph{equitable}, its color classes differing in size by at most
one. The outermost bound $\lvert\Uset\rvert^2 - \lvert\Uset\rvert + 1$ does not
mention the grid.
\end{lemma}

\begin{remark}[Which bounds are grid-independent]
\label{rmk:degree-grid}
The middle quantity $\lvert(\Uset-\Uset)\setminus\{0\}\rvert$, the degree of $\Gamma_\Uset$, evaluates differences \emph{on the torus} and thus depends on side lengths (distinct differences in $\mathbb{Z}^d$ can collide modulo $W_j$, strictly decreasing the count). Evaluating this expression in $\mathbb{Z}^d$ yields a grid-independent upper bound; the figures $7$, $23$, and $35$ below derive from this grid-independent evaluation.

For the von Neumann routing set, the $\mathbb{Z}^2$ count is $\lvert(\Uset-\Uset)\setminus\{0\}\rvert = 22$ (degree bound $23$). On $\mathbb{Z}_6^2$, offsets $+3$ and $-3$ are identified, dropping the exact count to $21$ (degree bound $22$). While both bounds validly majorize $M^\star$, only the $\mathbb{Z}^d$ bound is grid-independent. Grid independence is a statement about widths. The loads and logit scales of Corollary~\ref{cor:colour-interface}\emph{(c)} and \eqref{eq:onehead-scale} do depend on the volume and the horizon.
\end{remark}

\begin{proof}
The lower bound holds because $\Uset$ forms a clique in $\Gamma_\Uset$. $\Gamma_\Uset$ is regular of degree $\Delta = \lvert(\Uset-\Uset)\setminus\{0\}\rvert$; a greedy coloring uses at most $\Delta+1$ labels. Every nonzero difference comes from one of the $\lvert\Uset\rvert(\lvert\Uset\rvert-1)$ ordered pairs of distinct offsets, so $\Delta \leq \lvert\Uset\rvert^2 - \lvert\Uset\rvert$. Equitability at $\Delta+1$ colors follows from the Hajnal--Szemer\'edi theorem \citep{hajnal1970proof}. Lemma~\ref{lem:periodic} confirms that a periodic admissible coloring is proper for $\Gamma_\Uset$, so $M^\star \leq M_K$.
\end{proof}

\begin{rmk}
The upper bound depends only on local neighborhood geometry. Evaluated in $\mathbb{Z}^d$ (Remark~\ref{rmk:degree-grid}), it requires at most $7$ colors for the elementary neighborhood, $23$ for von Neumann, and $35$ for Moore, independent of grid side length. Theorem~\ref{thm:onehead-app} leverages this grid-independent bound.
\end{rmk}

\subsection{Colorings for the experimental grids}
\label{app:col:instance}

Lemma~\ref{lem:periodic} supplies admissible colorings for the three
experimental settings of Appendix~\ref{app:not:instance}, and each uses far
fewer colors than the degree bound on the same line.
Table~\ref{tab:instance-colour} records them.

\begin{table}[t]
\centering
\caption{The coloring Construction~II uses in each experimental setting. Every
one is admissible by Lemma~\ref{lem:periodic}, and every one is far below the
degree bound on the same line.}
\label{tab:instance-colour}
\small
\begin{tabular}{@{}lccc@{}}
\toprule
& 1D, $V{=}2$ & 1D, $V{=}3$ & 2D von Neumann \\
\midrule
$\lvert(\Uset{-}\Uset)\setminus\{0\}\rvert$ & $6$ & $6$ & $21$ \\
degree bound on $M^\star$               & $7$ & $7$ & $22$ \\
subgroup $K$                            & $4\mathbb{Z}_{16}$ & $4\mathbb{Z}_{32}$ & $\langle(2,2)\rangle$ \\
colors $M = M_K$                        & $4$ & $4$ & $12$ \\
\bottomrule
\end{tabular}
\end{table}

Each coloring is the quotient by a subgroup intersecting $\Uset - \Uset$ only at $0$ (Lemma~\ref{lem:periodic}). For $d = 1$, $\Uset = \{-1,0,1,2\}$ yields $(\Uset-\Uset)\setminus\{0\} = \{\pm1,\pm2,\pm3\}$. Because this lacks nonzero multiples of $4$, reduction modulo $4$ is admissible on any $\mathbb{Z}_W$ where $4 \mid W$, yielding an optimal $M = 4$ (Lemma~\ref{lem:Mbounds} requires $M^\star \geq 4$). For $d = 2$, coordinatewise reduction requires $M_K = 18$, while the diagonal subgroup $\langle(2,2)\rangle = \{(0,0),(2,2),(4,4)\}$ avoids $\Uset-\Uset$ to achieve $M_K = 12$.

The degree row in Table~\ref{tab:instance-colour} counts $\lvert(\Uset-\Uset)\setminus\{0\}\rvert$ on the torus, yielding $21$ on $\mathbb{Z}_6^2$. Evaluating differences in $\mathbb{Z}^d$ yields the grid-independent count of $23$ (Lemma~\ref{lem:Mbounds}, Table~\ref{tab:colouring}). This discrepancy arises from identifying $+3$ and $-3$ on the torus (Remark~\ref{rmk:degree-grid}). Table~\ref{tab:colouring} contextualizes these grids.

\begin{table}[t]
\centering
\caption{Colorings across grids, for reference. $M_K$ is the least number of colors over periodic
admissible colorings and is arithmetic in the side length. The degree bound
$\lvert(\Uset-\Uset)\setminus\{0\}\rvert + 1$ of Lemma~\ref{lem:Mbounds} is
evaluated with the differences taken in $\mathbb{Z}^d$, so it is the same for
every grid and bounds $M^\star$ from above; the exact torus count can be
smaller, as Remark~\ref{rmk:degree-grid} shows on $\mathbb{Z}_6^2$. The bound
never degenerates and $M_K$ sometimes does: on $\mathbb{Z}_5^2$ and
$\mathbb{Z}_7^2$ every nontrivial subgroup has a nonzero intersection with $\Uset-\Uset$, so
$M_K = \lvert\Grid\rvert$ and every cell takes a color of its own.}
\label{tab:colouring}
\small
\begin{tabular}{llrrrr}
\toprule
neighborhood & grid & $\lvert\Uset\rvert$ & least periodic $M_K$ & degree bound &
$\lvert\Grid\rvert$ \\
\midrule
elementary  & $\mathbb{Z}_{5}$      &  4 &  5 &  7 &   5 \\
elementary  & $\mathbb{Z}_{16}$     &  4 &  4 &  7 &  16 \\
elementary  & $\mathbb{Z}_{32}$     &  4 &  4 &  7 &  32 \\
von Neumann & $\mathbb{Z}_{5}^{2}$  &  8 & 25 & 23 &  25 \\
von Neumann & $\mathbb{Z}_{6}^{2}$  &  8 & 12 & 23 &  36 \\
von Neumann & $\mathbb{Z}_{7}^{2}$  &  8 & 49 & 23 &  49 \\
von Neumann & $\mathbb{Z}_{8}^{2}$  &  8 &  8 & 23 &  64 \\
Moore       & $\mathbb{Z}_{6}^{2}$  & 12 & 18 & 35 &  36 \\
Moore       & $\mathbb{Z}_{8}^{2}$  & 12 & 16 & 35 &  64 \\
Moore       & $\mathbb{Z}_{12}^{2}$ & 12 & 12 & 35 & 144 \\
\bottomrule
\end{tabular}
\end{table}

\begin{rmk}
For grids in Table~\ref{tab:colouring} (excluding $\mathbb{Z}_5^2$ and $\mathbb{Z}_7^2$), periodic coloring requires fewer colors and is preferred in practice. The degree bound ensures grid independence (Theorem~\ref{thm:onehead-app}).
\end{rmk}

\subsection{Coloring interface for Shared-head decoding}
\label{app:col:interface}

The Shared-head construction relies on the following five guarantees. They fix the properties and the coordinate counts; Appendix~\ref{app:constructions} assigns the coordinates to residual blocks.

\begin{corollary}[Coloring interface for the Shared-head construction]
\label{cor:colour-interface}
Let $c$ be an admissible coloring of $\Grid$ for $\Uset$ with $M$ colors.
Then
\begin{enumerate}[label=(\alph*),leftmargin=2.2em,itemsep=1pt,topsep=2pt]
  \item \emph{Distinct slots.} For every $z$, the $\lvert\Uset\rvert$ cells
        $z+\delta$, $\delta \in \Uset$, carry pairwise distinct colors.
  \item \emph{Address width.} For every $z$ the family of addresses
        $\bigl(e_{c(z+\delta)}\bigr)_{\delta\in\Uset}
        \in \mathbb{R}^{\lvert\Uset\rvert M}$ is the image, under one fixed
        affine map independent of $z$, of $\lvert\Uset\rvert(M-1)+1$
        positional coordinates: one width-$M$ address per offset, each a one-hot
        and so determined by $M-1$ of its coordinates given a shared constant.
  \item \emph{Load.} The number $n_{\max}$ of causally visible cells
        sharing one color, over a trajectory of at most $T$ time steps,
        satisfies
        \[
          n_{\max}
          \;\leq\;
          T\max_{s\in[M]}\bigl\lvert c^{-1}(s)\bigr\rvert .
        \]
        If $c$ is equitable, then
        $\max_{s}\lvert c^{-1}(s)\rvert\leq\lceil\lvert\Grid\rvert/M\rceil$,
        and consequently
        \[
          n_{\max}
          \;\leq\;
          T\Bigl\lceil \tfrac{\lvert\Grid\rvert}{M} \Bigr\rceil .
        \]
  \item \emph{Existence.} An equitable admissible $c$ exists on every grid with
        $M \leq \lvert(\Uset-\Uset)\setminus\{0\}\rvert + 1 \leq
        \lvert\Uset\rvert^2 - \lvert\Uset\rvert + 1$, where the differences may
        be taken in $\mathbb{Z}^d$, so that the final bound depends on
        $\Uset$ and on nothing else.
  \item \emph{Own-color block.} For every $z$ the own-color address
        $e_{c(z)} \in \mathbb{R}^M$ is a function of position alone and is
        carried by $M$ positional coordinates, which may be set to zero at
        separator positions.
\end{enumerate}
\end{corollary}

\begin{proof}
(a) is admissibility, which is injectivity on $z+\Uset$. For (b), the offset
$\delta$ is fixed, so $z \mapsto e_{c(z+\delta)}$ depends on position alone and
the family is carried by $\lvert\Uset\rvert$ width-$M$ positional blocks. A
one-hot of width $M$ has coordinates summing to $1$, so it lies in an affine
subspace of
dimension $M-1$ and $M-1$ coordinates determine it given the constant; the
$\lvert\Uset\rvert$ addresses therefore take
$\lvert\Uset\rvert(M-1)$ coordinates together with one constant coordinate.

For (c), a causally visible prefix contains at most $T$ copies of the grid, one
per time step, and within a single time step the cells of color $s$ number at
most $\lvert c^{-1}(s)\rvert$; summing over steps gives the general bound. If
$c$ is equitable, the $M$ color classes of $\Grid$ differ in size by at most
one, so none exceeds $\lceil\lvert\Grid\rvert/M\rceil$.
(d) restates Lemma~\ref{lem:Mbounds}. For (e), $z \mapsto e_{c(z)}$ depends on
position alone, and the positional table is indexed by sequence position, so
its entries at separator positions may be prescribed independently of its
entries at cell positions.
\end{proof}

\begin{rmk}
Condition (a) ensures correctness. Conditions (b), (c), and (d) provide quantitative bounds on width, logit scale, and grid independence. One can verify Appendix~\ref{app:constructions} treating $M$ symbolically; substituting the colorings of \S\ref{app:col:instance} (used in Section~\ref{sec:experiments}) yields concrete widths.
\end{rmk}

\apppart{Part III}{Layer 1: the spatial gather}
\section{Layer 1: spatial gather constructions}
\label{app:constructions}

This section proves the gather component of Theorem~\ref{thm:gather} by constructing the Layer-1 weights. At any position $n$ with virtual time step $\tau(n)\geq1$, Layer~1 must recover the preceding states
\[
    \bigl\{x_{\tau(n)-1,\;\kappa(n)+\delta} : \delta\in\Uset\bigr\}
\]
and track their offsets $\delta$. Localizing these states in separate residual blocks allows Layer~2 to read $D_Q$ at the current token and $D_K$ at an earlier token. Formulating this requirement via $(\tau(n),\kappa(n))$ ensures it covers both cells and separators, and retains cells with unresolved queries as necessary keys (Remark~\ref{rmk:gather-scope}).

Two challenges arise: routing the cells in $\Uset$ to the query token, and preserving their spatial identities. We provide two constructions resolving both:
\begin{enumerate}[label=(\Roman*),leftmargin=2.2em,itemsep=2pt,topsep=3pt]
    \item \textbf{Dedicated-heads.}
    Each offset $\delta\in\Uset$ receives its own attention head. The head
    directly writes the state at offset $\delta$ into a dedicated residual
    block, so attention performs both routing and separation by itself.

    \item \textbf{Shared-head.}
    A single head attends jointly to all offsets in $\Uset$. Since the head
    returns only one weighted sum, the routed states would normally be mixed
    together. We avoid this by encoding the cell state in the attention
    \emph{weight} and the routed-cell identity as a distinct basis vector in the value projection using a
    positional coloring. A one-hidden-layer MLP then separates the aggregate
    and reconstructs one state block per offset.
\end{enumerate}

Both constructions rely on the spatial-temporal addressing primitive developed above; Construction~II also requires the coloring primitive. We first recall the necessary properties.

Throughout, $\Grid$ denotes the $d$-dimensional torus of volume $\lvert\Grid\rvert$
(Definition~\ref{def:ca}), evaluated for $T$ time steps. For a stream position
$n$, $\kappa(n)$ and $\tau(n)$ denote its virtual cell and virtual time step
(Appendix~\ref{app:serialization}), and $\usucc$ denotes the successor offset.
See Table~\ref{tab:notation} for further notation.

\paragraph{Spatial-temporal addressing.}
For any set of offsets $\Nset$, Corollary~\ref{cor:head} provides a positional
table $p_n\in\mathbb{R}^{\dpos}$ and projections $W_Q^{(1)},W_K^{(1)}$ which,
at unit logit scale, score the keys at cells $\kappa(n)+\delta$,
$\delta\in\Nset$, in the previous time step strictly above every other causally
visible cell key, with a margin $\gamma>0$. Separator keys are covered by the
gate of \S\ref{app:con:gather}. Its positional score is
\[
    \bigl\langle W_Q^{(1)}p_n,\;W_K^{(1)}p_{n'}\bigr\rangle
    =
    g\bigl(\kappa(n)+\usucc-\kappa(n')\bigr)
    +\lambda h\bigl(\tau(n)-\tau(n')\bigr),
\]
where $g$ favors the requested spatial offsets and $h$ selects temporal
displacement $1$. The spatial width is $\dS=2d$ for a single offset and
$\dS=O(|\Nset|\,d\log|\Nset|)$ in general, and it is independent of the grid
volume and the horizon. After multiplying the logits by $\Lambda>0$,
Corollary~\ref{cor:scale} bounds the attention mass outside the target set by
$N_{\neq}e^{-\Lambda\gamma}$, where $N_{\neq}$ is the number of visible
non-target keys.

\paragraph{Positional colors.}
Construction~II must additionally preserve the identities of the cells routed
by the Shared-head. Corollary~\ref{cor:colour-interface} provides a coloring
$c$ of $\Grid$ which is injective on every translate $z+\Uset$, so the routing
cells $\{z+\delta:\delta\in\Uset\}$ carry pairwise distinct colors at every
query location $z$. Writing $M$ for the number of colors, such a coloring
exists with
\[
    M
    \leq
    \bigl|(\Uset-\Uset)\setminus\{0\}\bigr|+1
    \leq
    |\Uset|^2-|\Uset|+1,
\]
a bound depending on $\Uset$ and on nothing else. The same corollary provides
the positional storage for the query-side color addresses $e_{c(z+\delta)}$,
the key-side own-color block, and a bound on how many causally visible cells
share one color. Construction~II uses all four in
\S\ref{app:con:onehead}.

\subsection{Target representation and residual-stream layout}
\label{app:con:forward}

Both Layer-1 constructions produce the same output: $|\Uset|$ state blocks. For
each $\delta\in\Uset$, the $\delta$-th block should represent the state of cell
$(t-1,z+\delta)$.

Because finite softmax attention may leave mass on non-target keys, we first formalize how a block approximates the desired state.

\begin{definition}[$\eblock$-approximate one-hot]
\label{def:eclose}
For $\eblock\in[0,1)$ and $s\in\States$, a width-$V$ block $P$ is an
\emph{$\eblock$-approximate one-hot} for $e(s)$ if it is a sub-probability
vector, meaning that it is non-negative and has total mass at most $1$, and it
places mass at least $1-\eblock$ on coordinate $s$. Consequently, at most
$\eblock$ mass lies on the remaining $V-1$ coordinates. We call $\eblock$ the
\emph{per-block error}.
\end{definition}

\begin{definition}[Stream layout]
\label{def:stream}
\label{app:con:stream}
Both constructions use a residual stream
\[
    x_n=[\,B_0;B_1;B_2;B_3;B_4\,]
\]
with the following orthogonal blocks.
\begin{itemize}[leftmargin=1.6em,itemsep=2pt,topsep=3pt]
    \item $B_0\in\mathbb{R}^3$ is the \emph{gate block}. It stores the
    separator indicator $s_n$, the auxiliary-token marker
    $\mu_n\coloneqq\mathbf 1[n=\star]$, and a constant $1$.

    \item $B_1\in\mathbb{R}^V$ stores the token's own state one-hot and is
    used as the value source for state retrieval.

    \item $B_2\in\mathbb{R}^{\dpos}$ stores the positional features. Besides
    the $\dS+2$ spatial-temporal addressing coordinates of
    Corollary~\ref{cor:head}, it carries the \emph{initial-step indicator}
    \begin{equation}
    \label{eq:zeta}
        \zeta_n
        \coloneqq
        \mathbf 1\bigl[\tau(n)=0\bigr],
    \end{equation}
    one coordinate, used by the Layer-2 key-validity gate of
    Definition~\ref{def:l2match}. Construction~II augments these with the
    query-side color addresses and the own-color block
    \eqref{eq:owncolour}.

    \item $B_3$ is the Layer-1 scratch space. It contains
    \[
        B_3^{\mathrm{pat}}
        =
        \bigoplus_{\delta\in\Uset}\mathbb{R}^V,
    \]
    whose $\delta$-th block will hold the state routed from offset $\delta$.
    Construction~I writes these blocks directly and uses no additional
    scratch. Construction~II also uses
    $B_3^{\mathrm{agg}}\in\mathbb{R}^{M+1}$ for the Shared-head aggregate;
    its MLP decodes this aggregate into $B_3^{\mathrm{pat}}$ and then clears
    $B_3^{\mathrm{agg}}$.

    \item $B_4\in\mathbb{R}^V$ is reserved for the Layer-2 output and is
    initialized to zero on ordinary tokens.
\end{itemize}

For an ordinary cell token at $(t,z)$,
\[
    B_0=(0,0,1),
    \qquad
    B_1=e(x_{t,z}),
\]
while a separator has
\[
    B_0=(1,0,1),
    \qquad
    B_1=0.
\]
Blocks $B_3$ and $B_4$ are initialized to zero. When the auxiliary token
$\star$ of \S\ref{app:con:gather} is present, it has
\[
    B_0=(0,1,1),
    \qquad
    B_1=0,
    \qquad
    B_2=0,
    \qquad
    B_3=0,
    \qquad
    B_4=0,
\]
so that $s_\star=\zeta_\star=0$. All other tokens are \emph{ordinary} and have
$\mu_n=0$.
\end{definition}

\begin{remark}[The initial-step indicator]
\label{rmk:initstep}
While the positional table carries the temporal pair $\psi(\tau)=(\cos 2\pi\tau/T,\sin 2\pi\tau/T)$ (Eq.~\ref{eq:timetable}), the indicator $\zeta_n=\mathbf 1[\tau(n)=0]$ is not a linear functional of it for $T\geq3$, necessitating a dedicated coordinate. Like all positional features, it depends solely on sequence position and is readily learned.

This indicator resolves a role asymmetry: an initial-step cell carries $x_{0,z}$ (needed by Layer~1 for queries at $t=1$) but lacks a parent configuration (excluding it as a Layer-2 key, consistent with $t'\geq1$ in Eq.~\ref{eq:counts}). Consequently, Layer~1 gates only separators, whereas Layer~2 gates both separators and initial-step cells (\S\ref{app:syn:keygate}).
\end{remark}

\paragraph{Accounting for $B_2$.}
Construction~I stores only the addressing coordinates, so
\[
    \dpos=\dS+3,
    \qquad
    \dS=
    \begin{cases}
        2d, & |\Nset|=1,\\
        O\bigl(|\Nset|\,d\log|\Nset|\bigr), & \text{in general},
    \end{cases}
\]
independently of the grid volume and the horizon. Two of the three non-spectral
coordinates are the temporal pair of Corollary~\ref{cor:head}; the third is the
initial-step indicator $\zeta_n$ of \eqref{eq:zeta}, which Layer~2 needs and
which we account for here so that one positional width serves the whole
circuit.

Construction~II adds two color blocks. For every offset $\delta\in\Uset$ the
table stores the query-side address $e_{c(z+\delta)}$ of the routed cell, which
Corollary~\ref{cor:colour-interface}\emph{(b)} carries in at most
\[
    |\Uset|(M-1)+1
\]
coordinates; these addresses let the Layer-1 MLP identify which coordinate of
the Shared-head aggregate belongs to each offset. Separately, the Shared-head
needs the key-side color of the key's own cell, because its value projection
emits exactly that. By Corollary~\ref{cor:colour-interface}\emph{(e)} the table
carries it as a dedicated width-$M$ block
\begin{equation}
\label{eq:owncolour}
    B_2^{\mathrm{col}}
    \coloneqq
    (1-s_n)\,e_{c(\kappa(n))}
    \in\mathbb{R}^{M},
\end{equation}
the own-color one-hot at cell positions and the zero vector at separator
positions. We store this block separately from the $\delta=0$ address. Separate
storage makes the construction independent of whether $0\in\Uset$, and it makes
a separator's Layer-1 value exactly zero, so no separator color can enter the
Shared-head aggregate. This block does not admit the ``$M-1$ coordinates plus a
shared constant'' compression used for the query addresses, because the
reconstructed coordinate would read $1$, not $0$, at a separator; it therefore
costs a full $M$ coordinates.

We can now state our goal as an interface condition on
$B_3^{\mathrm{pat}}$.

\begin{definition}[Spatial induction head, quantitatively]
\label{def:sih-app}
Fix $\eblock\in[0,1)$. A two-layer circuit with the stream layout of
Definition~\ref{def:stream} is a \emph{spatial induction head at per-block
error $\eblock$} if:
\begin{enumerate}[label=(\roman*),leftmargin=2.2em,itemsep=2pt,topsep=3pt]
    \item \emph{The gather.}
    At \emph{every} position $n$ with $\tau(n)\geq1$, after Layer~1 the pattern
    scratch $\Pi_{B_3^{\mathrm{pat}}}x_n$ consists of $|\Uset|$ width-$V$
    blocks indexed by $\delta\in\Uset$, and the $\delta$-th block is an
    $\eblock$-approximate one-hot for
    \[
        e\bigl(x_{\tau(n)-1,\;\kappa(n)+\delta}\bigr).
    \]
    At positions with $\tau(n)=0$ only the weaker property of
    Lemma~\ref{lem:scratch-all} is claimed.

    \item \emph{The match.}
    Layer~2 compares the $k$ blocks indexed by $D_Q$ at the current query with
    the $k$ blocks indexed by $D_K$ at earlier tokens, and averages the states
    carried by the matching precedents themselves, that is their own
    $B_1$ entries, into $B_4$.
\end{enumerate}
\end{definition}

\begin{remark}[The gather covers more positions than the prediction]
\label{rmk:gather-scope}
Condition~\emph{(i)} quantifies over all positions with $\tau(n)\geq1$, a superset of the resolvable queries. Layer~2 reads $D_K$-blocks at every valid earlier key, requiring only that the cell records an observed transition. This key set properly contains the query set via two additions:
\begin{itemize}[leftmargin=1.6em,itemsep=2pt,topsep=3pt]
  \item \emph{Row-final cells} ($z_d=W_d-1, t\geq1$) pose no cell query (their successor is a separator, Lemma~\ref{lem:rowfinal}), but supply a valid observation $x_{t,z} \sim f(\cdot\mid P_{t,z})$ and must be retained as keys.
  \item \emph{Cells with unresolved queries} provide valid precedents $P_{t,z}$ regardless of whether their own query $Q_{t,z}$ is resolvable.
\end{itemize}
Conversely, initial-step cells are excluded from both sets, an exclusion the circuit explicitly implements (Remark~\ref{rmk:initstep}). Resolvability restricts the downstream \emph{prediction} (Appendix~\ref{app:synthesis}), not the \emph{representation} derived here.
\end{remark}

In this section, we prove condition~\emph{(i)}. In
Appendix~\ref{app:synthesis}, we use only this interface and construct the
common Layer-2 match.

\subsection{Separator gating and finite-softmax error}
\label{app:con:gather}

Before giving the two constructions, we isolate two issues common to both.
First, separators must be excluded from Layer-1 attention. Second, when a head
is intended to select a single routing cell, finite softmax leaves a small
amount of mass elsewhere.

Separators need excluding because their addresses duplicate cells'. By
Lemma~\ref{lem:sepsucc}\emph{(i)} a separator carries the virtual index of the
cell \emph{one step before its successor}. A \texttt{ROW\_SEP} closing row $r$
of time step $t$ has
\[
    \bigl(\tau(n),\kappa(n)\bigr)
    =
    \bigl(t,\;(r+1,\,W_d-1)\bigr),
\]
the last-column cell of the \emph{following} row, and a \texttt{TIME\_SEP}
closing step $t$ has $(\tau(n),\kappa(n))=(t+1,\,(0,\dots,0,W_d-1))$, the
last-column cell of the first row of the \emph{following} step. In either case
the separator shares its virtual spatial-temporal address with an ordinary
cell, and both are visible to any sufficiently late query, so a purely
positional score cannot tell them apart. Without a gate, a query addressing
that cell splits its mass with the duplicate separator.

\begin{proposition}[The Layer-1 separator gate]
\label{prop:sepgate-l1}
Add $-\Gamma_1 s_{n'}$ to the Layer-1 score, implemented by a rank-one term
from the constant query coordinate in $B_0$ to the separator indicator in the
key. If
\begin{equation}
\label{eq:gamma1}
    \Gamma_1
    \;\geq\;
    \Lambda\,(R+\gamma),
    \qquad
    R
    \coloneqq
    \max_{\delta\in\Grid}g(\delta)
    -
    \min_{\delta\in\usucc-\Nset}g(\delta)
    \;\geq\;0,
\end{equation}
with $R$ the quantity already fixed in \eqref{eq:lambda}, then every separator
key scores at least $\Lambda\gamma$ below every target key. Separators are then
absorbed into the same off-target bound $N_{\neq}e^{-\Lambda\gamma}$ of
Corollary~\ref{cor:scale} as the remaining non-target keys.
\end{proposition}

\begin{proof}
Write the ungated score as $\Lambda G(\delta,s)$ in the notation of Stage~6 of
Corollary~\ref{cor:head}, and put
$m\coloneqq\min_{\delta\in\usucc-\Nset}g(\delta)$, so that the target floor of
case~(a) of that stage is $\Lambda(m+\lambda)$. A separator key attains at most
\[
    \Lambda G
    \;\leq\;
    \Lambda\bigl(m+R+\lambda\bigr)
\]
before the gate: its spatial term is at most $\max_\Grid g=m+R$ by the
definition of $R$, and its temporal term is at most $\lambda$, attained when
$h=1$. After subtracting $\Gamma_1\geq\Lambda(R+\gamma)$ its score is at most
\[
    \Lambda(m+R+\lambda)-\Lambda(R+\gamma)
    \;=\;
    \Lambda(m+\lambda)-\Lambda\gamma ,
\]
which is at least $\Lambda\gamma$ below the target floor. Separators therefore
satisfy the same isolation inequality as the non-target cell keys, and
Corollary~\ref{cor:scale} applies with $N_{\neq}$ counting both.
\end{proof}

\begin{rmk}
The gate multiplies a token-embedding coordinate which is already present in
$B_0$, so it adds no positional coordinate and does not alter the width bounds
of the addressing construction. It is query independent and provides no
additional spatial resolution. The bound \eqref{eq:gamma1} is loose: it charges
every separator the worst spatial score on the whole
grid. We do not gate initial-step cells here; see
Remark~\ref{rmk:initstep}.
\end{rmk}

\begin{rmk}
The weaker choice $\Gamma_1=\Lambda$ does not support the stated bound. A
separator whose address coincides with a target cell's may also score the
target floor $\Lambda(m+\lambda)$ of the proof above, so subtracting $\Lambda$
leaves it only $\Lambda$ below that floor, for relative weight $e^{-\Lambda}$, against the
$e^{-\Lambda\gamma}$ claimed downstream. The two agree only at $\gamma=1$, and
$\gamma>1$ occurs on small grids: on the three-cell torus with $T=3$ the
singleton spatial kernel of Lemma~\ref{lem:degreeone}\emph{(ii)} has
$\gamma=1-\cos(2\pi/3)=3/2$, and at $\Lambda=10$ and $N=12$ the true block
error $e^{-10}\approx4.54\times10^{-5}$ exceeds the claimed
$Ne^{-\Lambda\gamma}\approx3.67\times10^{-6}$.
\end{rmk}

\paragraph{The auxiliary token.}
Construction~II always, and Construction~I when a Dirichlet prior is used,
prepend the auxiliary token $\star$ of Definition~\ref{def:stream} at
position~$0$. Every term involving $\star$ is either a rank-one product of the
query constant $1\in B_0$ with the key marker $\mu_{n'}$, or a value term
reading $\mu_{n'}$. Since $\mu$ vanishes at ordinary tokens, these terms leave
every ordinary score and value unchanged. Prepending $\star$ shifts every
ordinary position by one, so the layout becomes $\rho_\star(t,z)=\rho(t,z)+1$
and the stream length becomes $N_\star=N+1$. Each ordinary token keeps the
positional features it had before the shift, now stored at column
$\rho_\star$ of the table, so the virtual indices $\kappa$ and $\tau$ and every
ordinary score are unchanged. When $\star$ is present, every bound below that
involves $N$ is read with $N_\star$ in its place.

\begin{lemma}[Layer-1 role of $\star$]
\label{lem:star-l1}
At the key $\star$, the positional score vanishes because $B_2=0$, the content
score vanishes because $B_1=0$, and the separator gate vanishes because
$s_\star=0$. Hence:
\begin{enumerate}[label=(\alph*),leftmargin=2.2em,itemsep=1pt,topsep=2pt]
  \item \emph{Construction~II.} The term $\Lambda g_\star\cdot1\cdot\mu_{n'}$
  in $W_Q^{(1)\top}W_K^{(1)}$ gives $\star$ the logit $\Lambda g_\star$ at every
  query. The value term $e_\star\mu_{n'}$, with $e_\star$ the last coordinate of
  $B_3^{\mathrm{agg}}\in\mathbb{R}^{M+1}$, gives $\star$ the value $e_\star$;
  ordinary keys contribute nothing to that coordinate.
  \item \emph{Construction~I.} Let $\Lambda(m_\delta+\lambda)$ be the target
  floor of head $\delta$ (Stage~6 of Corollary~\ref{cor:head} at
  $\Nset=\{\delta\}$). The term
  $\Lambda(m_\delta+\lambda-\gamma)\cdot1\cdot\mu_{n'}$ in head $\delta$ places
  $\star$ exactly $\Lambda\gamma$ below the target, and its value
  $\iota_\delta\Pi_{B_1}x_\star$ is zero. Thus $\star$ is a non-target key
  covered by Corollary~\ref{cor:scale}.
  \item \emph{As a query}, $\star$ attends only to itself, and its pattern
  scratch $B_3^{\mathrm{pat}}$ is zero after Layer~1: in Construction~I every
  head returns $\iota_\delta\Pi_{B_1}x_\star=0$, and in Construction~II this is
  shown in Step~5 of the proof of Theorem~\ref{thm:onehead-app}.
\end{enumerate}
\end{lemma}

\begin{lemma}[Per-block error of a single head]
\label{lem:blockerror}
Suppose a Layer-1 head places attention mass at least $1-\eblock$ on the
routing cell $(t-1,z+\delta)$ and reads values from $B_1$ into one width-$V$
block of $B_3^{\mathrm{pat}}$. Then that block is an
$\eblock$-approximate one-hot for $e(x_{t-1,z+\delta})$. In particular,
Corollary~\ref{cor:scale} guarantees
\[
    \eblock\leq Ne^{-\Lambda\gamma},
\]
and hence any prescribed $\eblock>0$ is achieved whenever
\[
    \Lambda\geq \gamma^{-1}\log(N/\eblock).
\]
\end{lemma}

\begin{proof}
The written block is a non-negative weighted sum of the state one-hots carried
by visible cell tokens. The addressed routing cell contributes mass at least
$1-\eblock$ to its correct state coordinate, while all remaining attention
mass is at most $\eblock$, so the block places at most $\eblock$ on the other
$V-1$ coordinates.

Three kinds of key make up that remaining mass, and all are covered. Ordinary
cell keys at non-target offsets are $\Lambda\gamma$ below the target by
isolation. Separator keys are $\Lambda\gamma$ below it by
Proposition~\ref{prop:sepgate-l1}, and in addition contribute the zero vector,
since a separator has $B_1=0$; they can only lower the block's total mass, never
move it to a wrong coordinate. The key $\star$, when present, is
$\Lambda\gamma$ below the target and contributes the zero vector by
Lemma~\ref{lem:star-l1}\emph{(b)}. The block is therefore non-negative with total
mass at most $1$, so Definition~\ref{def:eclose} asks only for a
sub-probability vector.

For the scale bound, apply Corollary~\ref{cor:scale} with the target set the
single addressed cell. Every other visible position is a non-target key, so
$N_{\neq}\leq N-1\leq N$, giving $\eblock\leq Ne^{-\Lambda\gamma}$ and hence
the stated condition on $\Lambda$.
\end{proof}

\subsection{Construction I: One head per offset}
\label{app:con:parallel}

In the first construction, head $\delta$ addresses one offset and writes its
state into the block reserved for that same offset, so routing and separation
happen together.

\begin{theorem}[Dedicated-heads gather]
\label{thm:parallel-app}
Build Layer~1 with $h=|\Uset|$ heads, one head for each
$\delta\in\Uset$, and no MLP, incorporating the separator gate of
Proposition~\ref{prop:sepgate-l1}. At logit scale $\Lambda$, this realizes
condition~\emph{(i)} of Definition~\ref{def:sih-app} with
\[
    \eblock\leq Ne^{-\Lambda\gamma},
\]
and with
\[
    \dmlp=0,
    \qquad
    \dmodel=(|\Uset|+2)V+2d+6,
    \qquad
    \dpos=2d+3.
\]
The same holds with the auxiliary token $\star$ configured as in
Lemma~\ref{lem:star-l1}\emph{(b)}.
\end{theorem}

\begin{proof}
Fix $\delta\in\Uset$. Head $\delta$ must reach the key at cell
$\kappa(n)+\delta$, so it applies Corollary~\ref{cor:head} to the singleton
offset set
\[
    \Nset_\delta
    \coloneqq
    \{\delta\}.
\]
For a singleton the spatial addressing width is $\dS=2d$
by Lemma~\ref{lem:degreeone}\emph{(ii)}, independently of which singleton it
is, so all $|\Uset|$ heads share one spatial table of width $2d$ and differ
only through their query-key projections, which rotate the $2\times2$ blocks
\eqref{eq:block} by the angle belonging to their own offset. With the two
temporal coordinates and the indicator $\zeta_n$ of \eqref{eq:zeta}, the
positional width is $2d+3$.

At a position $n$ with $\tau(n)\geq1$, head $\delta$ therefore places all but at
most $Ne^{-\Lambda\gamma}$ of its attention mass on the cell
$(\tau(n)-1,\;\kappa(n)+\delta)$, using
Proposition~\ref{prop:sepgate-l1} to cover separator keys. Its value projection
is
\[
    W_V^{(1),\delta}
    =
    \iota_\delta\circ\Pi_{B_1},
\]
where $\iota_\delta$ embeds $\mathbb{R}^V$ into the $\delta$-indexed block of
$B_3^{\mathrm{pat}}$. By Lemma~\ref{lem:blockerror}, that block is an
$\eblock$-approximate one-hot for $e(x_{\tau(n)-1,\,\kappa(n)+\delta})$ with
$\eblock\leq Ne^{-\Lambda\gamma}$.

This argument depends on the query solely through $\kappa(n)$ and $\tau(n)$, applying to every position with $\tau(n)\geq1$ (both cells and separators) as required by Definition~\ref{def:sih-app}(i). The addressed cells are causally visible in all cases: step $t-1$ strictly precedes any cell token at step $t\geq1$ (since $\rho$ is order-preserving), and Lemma~\ref{lem:sepsucc}(ii) establishes the same for separators. This guarantees coverage for row-final cells and cells with unresolved queries, both necessary for Layer~2 (Remark~\ref{rmk:gather-scope}).

Repeating this independently for every $\delta\in\Uset$ produces all required
blocks. Because the heads write to disjoint blocks, no Layer-1 MLP is needed.
The residual width is therefore
\[
    \underbrace{3}_{B_0}
    +
    \underbrace{V}_{B_1}
    +
    \underbrace{(2d+3)}_{B_2}
    +
    \underbrace{|\Uset|V}_{B_3}
    +
    \underbrace{V}_{B_4}
    =
    (|\Uset|+2)V+2d+6.
\]
Thus condition~\emph{(i)} of Definition~\ref{def:sih-app} holds.
\end{proof}

\subsection{Construction II: A shared head with MLP decoding}
\label{app:con:onehead}

Construction~II uses a single Layer-1 head to serve all offsets in $\Uset$. Since an attention head outputs a single weighted sum, attending to all routed cells simultaneously would normally destroy their spatial identities.

The construction resolves this by splitting the information across
\emph{weights} and \emph{values}. The state of a routed cell controls its
scalar attention weight, while its positional color determines the direction
of its value vector. Because the routing cells have distinct colors, they
occupy distinct coordinates of the Shared-head aggregate. A fixed reference
token carries a known weight through the same softmax, so the MLP can cancel
the shared normalization, and the query stores the color
expected at each offset. The MLP can therefore select the coordinate belonging
to a given offset and recover its discrete state by threshold comparisons.

The construction critically relies on the discrete state alphabet. Consecutive states differ by $1$, so a half-integer threshold is separated by $1/2$ from any valid state. The MLP recovers the exact state provided contamination from non-routing cells remains strictly below this margin, alleviating the need for perfectly sparse attention.

\paragraph{Layer-1 reference token.}
Construction~II uses the auxiliary token $\star$ of
Lemma~\ref{lem:star-l1}\emph{(a)} as a Layer-1 reference for the shared softmax
normalization: its logit $\Lambda g_\star$ is set by the marker term and so
carries no temporal factor, and its value is the dedicated coordinate
$e_\star$. Its Layer-2 role is fixed in Appendix~\ref{app:synthesis}.

\begin{theorem}[Shared-head gather]
\label{thm:onehead-app}
Let Layer~1 have a single head whose value is the own-color block
\eqref{eq:owncolour}, for a coloring $c$ admissible for $\Uset$ with $M$ colors
in the sense of Definition~\ref{def:colouring}; in particular an ordinary cell
$z'$ emits $e_{c(z')}$ and a separator emits $0$. Let $\gamma$ be the isolation
margin of Corollary~\ref{cor:head} for $\Nset=\Uset$, and let $n_{\max}$ be the
largest number of causally visible cells carrying one color. By
Corollary~\ref{cor:colour-interface}\emph{(c)},
$n_{\max}\leq T\max_{s}\lvert c^{-1}(s)\rvert$, and
$n_{\max}\leq T\lceil\lvert\Grid\rvert/M\rceil$ if $c$ is equitable. If
\begin{equation}
\label{eq:onehead-scale}
    \Lambda\gamma
    \geq
    \log\bigl(2V(n_{\max}-1)\bigr),
\end{equation}
then a one-hidden-layer ReLU MLP of width
\[
    \dmlp
    \leq
    2|\Uset|M(V-1)+M+2
\]
recovers $e(x_{\tau(n)-1,\,\kappa(n)+\delta})$ for every $\delta\in\Uset$, at
every position with $\tau(n)\geq1$. Hence
condition~\emph{(i)} of Definition~\ref{def:sih-app} holds exactly, with
$\eblock=0$.

The positional and residual widths are
\[
    \dpos
    \leq
    \underbrace{\dS+3}_{\text{addressing},\ \zeta}
    +
    \underbrace{|\Uset|(M-1)+1}_{\text{query addresses}}
    +
    \underbrace{M}_{B_2^{\mathrm{col}}},
    \qquad
    \dS=O\bigl(|\Uset|\,d\log|\Uset|\bigr),
\]
and
\[
    \dmodel
    =
    (|\Uset|+2)V+\dpos+M+4.
\]
\end{theorem}

\begin{proof}
The computation is implemented in five sequential stages.

Throughout, write $z\coloneqq\kappa(n)$ and $t\coloneqq\tau(n)$ for the query's
virtual cell and time step, so that the argument applies verbatim at separator
queries as well as at cell tokens.

\paragraph{1. Encode the state in the attention weight.}
At the query $n$, the Layer-1 logit of a causally visible key $n'$ is
\begin{equation}
\label{eq:onehead-logit}
    \ell(n,n')
    =
    \underbrace{
        \Lambda\Bigl[
            g\bigl(z+\usucc-\kappa(n')\bigr)
            +
            \lambda\,h\bigl(t-\tau(n')\bigr)
        \Bigr]
        -
        \Gamma_1 s_{n'}
    }_{\text{positional}}
    +
    \underbrace{
        \bigl\langle w,\Pi_{B_1}x_{n'}\bigr\rangle
    }_{\text{content}},
    \qquad
    w_v=\log(v+1).
\end{equation}
The positional part is the spatiotemporal score of Corollary~\ref{cor:head} at
$\Nset=\Uset$, which is large exactly on the routing cells $\kappa(n')=z+\delta$,
$\delta\in\Uset$, in step $t-1$, together with the separator gate of
Proposition~\ref{prop:sepgate-l1}. The content part is linear in the one-hot
state, equal to $\log(x_{\tau(n'),\kappa(n')}+1)$ at a cell key, which is $\log(x_{t-1,\kappa(n')}+1)$ at a routed cell, and to $0$ at a
separator, so block-diagonal query-key projections realize
\eqref{eq:onehead-logit} without mixing positional and state coordinates.

Write $w_{n'}$ for the exponentiated positional part of
\eqref{eq:onehead-logit}. Every key in step $t-1$ has $h=1$ and so carries the
same temporal factor $e^{\Lambda\lambda}$; for a routing cell at offset
$\delta\in\Uset$ we absorb that factor into
\begin{equation}
\label{eq:wdelta}
    w_\delta
    \coloneqq
    e^{\Lambda\bigl(g(\usucc-\delta)+\lambda\bigr)} ,
\end{equation}
the positional weight the kernel assigns to the key at offset $\delta$. Let $Z$
denote the softmax partition function over all causally visible keys. The
attention weight assigned to the routed cell at offset $\delta$ is then
\[
    \frac{(x_{t-1,z+\delta}+1)w_\delta}{Z},
\]
so the state has been encoded as a scalar factor in the attention weight.

Two bookkeeping points carry through the remaining steps. Each $w_\delta$ is a
fixed constant per offset, known when the weights are set, and it carries
$e^{\Lambda\lambda}$ inside it; Steps~3 and~4 use $w_\delta$ only through the
ratios $w_\delta/w_\star$ against the reference weight and $Z_{\max}/w_\delta$,
both fixed at the same time. The partition function $Z$ is the one quantity
here which the weights cannot know, and the reference coordinate of Step~3
removes it.

\paragraph{2. Encode routed-cell identity in the value.}
The value projection reads the own-color block \eqref{eq:owncolour},
\[
    W_V^{(1)}
    =
    \Pi_{B_2^{\mathrm{col}}} ,
\]
so an ordinary cell $z'$ outputs its color code $e_{c(z')}\in\mathbb{R}^M$ and
a separator outputs the zero vector, the block being $(1-s_{n'})e_{c(\kappa(n'))}$.
Separator keys therefore contribute nothing to the aggregate, so this
construction attains $\eblock=0$ exactly, with no separator
residue. The reference token outputs an additional basis
vector $e_\star$, so the Shared-head aggregate is
\[
    a\in\mathbb{R}^{M+1}.
\]
Because the routed cells $z+\delta$, $\delta\in\Uset$, have distinct colors,
the contribution of each routed cell lands in a different coordinate:
\[
    a[c(z+\delta)]
    =
    \frac{(x_{t-1,z+\delta}+1)w_\delta}{Z}
    +
    r_\delta,
    \qquad
    a[\star]
    =
    \frac{w_\star}{Z},
\]
where
\[
    w_\star\coloneqq e^{\Lambda g_\star}
\]
and $r_\delta$ is the total contribution from non-routing cells which happen
to share the color $c(z+\delta)$.

Ignoring $r_\delta$ for the moment, the desired state is already present in
the scalar $a[c(z+\delta)]$. Two obstacles remain: the common normalization
$Z$ is unknown to the MLP, and the MLP must select the color coordinate
corresponding to offset $\delta$.

\paragraph{3. Bound contamination and cancel the unknown normalization.}
Admissibility guarantees distinct colors only among the routing cells, so
non-routing \emph{cells} may still contribute to the same color coordinate;
separators cannot, by Step~2. At most $n_{\max}-1$ visible non-routing cells
share any given color, where $n_{\max}$ is the bound of
Corollary~\ref{cor:colour-interface}\emph{(c)}. Each such cell contributes a
content factor at most $V$.

Its positional weight $w_{n'}$ is smaller than every routing cell's by at least
$e^{-\Lambda\gamma}$. This is Stage~6 of the proof of
Corollary~\ref{cor:head}, which puts all three kinds of non-routing key at
least $\Lambda\gamma$ below the target floor $\Lambda(m+\lambda)$, where
$m\coloneqq\min_{\delta\in\Uset}g(\usucc-\delta)$: a key in step $t-1$ at an
offset outside $\Uset$ by the spatial margin $\gamma$, a key in any other time
step by the temporal margin, since $\lambda\gamma_T\geq R+\gamma$ by
\eqref{eq:lambda}, and a separator key by
Proposition~\ref{prop:sepgate-l1}. Excluding the wrong-time keys is the
temporal term's duty, and \eqref{eq:wdelta} carries that term. Writing
$w_{\min}\coloneqq\min_{\delta\in\Uset}w_\delta=e^{\Lambda(m+\lambda)}$, every
non-routing key has $w_{n'}\leq e^{-\Lambda\gamma}w_{\min}$, and hence
$w_{n'}\leq e^{-\Lambda\gamma}w_\delta$ for every $\delta\in\Uset$. Summing
over the sharing cells,
\[
    Zr_\delta
    \leq
    V(n_{\max}-1)e^{-\Lambda\gamma}w_\delta.
\]
Under \eqref{eq:onehead-scale},
\begin{equation}
\label{eq:stray-half-margin}
    Zr_\delta
    \leq
    \frac12 w_\delta.
\end{equation}

The reference coordinate removes the unknown partition function. For an
offset $\delta$, color slot $s$, and threshold
$v\in\{1,\dots,V-1\}$, define
\[
    C_{\delta,s,v}
    \coloneqq
    a[s]
    -
    \Bigl(v+\frac12\Bigr)
    \frac{w_\delta}{w_\star}a[\star].
\]
At the target color $s=c(z+\delta)$,
\[
    C_{\delta,c(z+\delta),v}
    =
    \frac{w_\delta}{Z}
    \Bigl(x_{t-1,z+\delta}-v+\frac12\Bigr)
    +
    r_\delta.
\]
Thus, apart from the controlled contamination $r_\delta$, the sign of
$C_{\delta,c(z+\delta),v}$ tells us whether
$x_{t-1,z+\delta}\geq v$.

\paragraph{4. Use the query's color address to select the correct slot and
threshold it.}
The query token stores $g_\delta(z) = e_{c(z+\delta)}$ in its positional block $B_2$; thus $g_\delta(z)[s] = 1$ exactly at the color slot for offset $\delta$.

Selecting the slot and thresholding its scalar are nonlinear operations. Since the architecture (\S\ref{sec:task:model}) possesses only a single hidden ReLU layer, we implement both within the same hidden units. Letting $Z_{\max}\geq Z$ be a strict upper bound on the partition function, we define
\begin{equation}
\label{eq:gate}
    P_{\delta,s,v}
    \coloneqq
    G_\delta C_{\delta,s,v}
    -
    B_\delta\bigl(1-g_\delta(z)[s]\bigr),
    \qquad
    G_\delta
    \coloneqq
    \frac{4Z_{\max}}{w_\delta},
    \qquad
    B_\delta
    \coloneqq
    G_\delta.
\end{equation}
Use the pair of hidden units
\[
    \operatorname{ReLU}(P_{\delta,s,v}),
    \qquad
    \operatorname{ReLU}(P_{\delta,s,v}-1),
\]
whose difference is the clamp
\[
    \chi_{\delta,s,v}
    =
    \min\{1,\max\{0,P_{\delta,s,v}\}\}.
\]

We claim that
\[
    \sum_{s=1}^M\chi_{\delta,s,v}
    =
    \mathbf{1}[x_{t-1,z+\delta}\geq v].
\]
First consider a wrong color slot, for which $g_\delta(z)[s]=0$. Since
$a[s]\leq1$, we have $C_{\delta,s,v}\leq1$, so
\[
    P_{\delta,s,v}
    \leq
    G_\delta-B_\delta
    =
    0.
\]
Thus every wrong slot contributes zero.

At the target slot, the penalty term vanishes. If
$x_{t-1,z+\delta}\geq v$, then
\[
    x_{t-1,z+\delta}-v+\frac12
    \geq
    \frac12,
\]
and $r_\delta\geq0$, so
\[
    P_{\delta,c(z+\delta),v}
    \geq
    \frac{4Z_{\max}}{w_\delta}
    \frac{w_\delta}{2Z}
    =
    2\frac{Z_{\max}}{Z}
    \geq2.
\]
The clamp therefore equals $1$.

If $x_{t-1,z+\delta}<v$, then
\[
    x_{t-1,z+\delta}-v+\frac12
    \leq
    -\frac12.
\]
Using \eqref{eq:stray-half-margin},
\[
    C_{\delta,c(z+\delta),v}
    \leq
    -\frac{w_\delta}{2Z}
    +
    r_\delta
    \leq
    0,
\]
so the clamp evaluates to $0$. Thus, the MLP recovers every threshold indicator exactly.

\paragraph{5. Recover the one-hot state and clear the temporary aggregate.}
Threshold indicators determine the discrete state. With the convention
$\chi_{\delta,s,V}\coloneqq0$, coordinate $v\in\{1,\dots,V-1\}$ of the
$\delta$-th pattern block receives
\[
    \sum_{s=1}^M
    \bigl(
        \chi_{\delta,s,v}
        -
        \chi_{\delta,s,v+1}
    \bigr),
\]
while coordinate $0$ receives
\[
    1-\sum_{s=1}^M\chi_{\delta,s,1}.
\]
These are exactly the coordinates of
$e(x_{t-1,z+\delta})=e(x_{\tau(n)-1,\,\kappa(n)+\delta})$. The constant is
supplied by the hidden unit $\operatorname{ReLU}(1-\mu_n)$, which reads the
constant and marker coordinates of $B_0$; it equals $1$ at ordinary tokens and
$0$ at $\star$.

At $\star$ the pattern scratch is zero. Since $B_2=0$ there, the shared constant
of Corollary~\ref{cor:colour-interface}\emph{(b)} vanishes and
$g_\delta(\star)=0$. Since $\star$ attends only to itself, $a=e_\star$, so
$C_{\delta,s,v}=-(v+\tfrac12)w_\delta/w_\star<0$ and
$P_{\delta,s,v}<0$ for all $\delta,s,v$. Every threshold unit and the constant
unit therefore vanish at $\star$.

Because $z$ and $t$ denote the query's virtual cell and time step, recovery holds at every position with $\tau(n)\geq1$ (both cells and separators). Routed cells reside in step $\tau(n)-1$ and are causally visible due to $\rho$'s order preservation ($t\geq1$) and Lemma~\ref{lem:sepsucc}(ii) (separators). This fulfills Definition~\ref{def:sih-app}(i) and Remark~\ref{rmk:gather-scope} entirely.

The Shared-head aggregate $a$ is temporary. As $a$ is non-negative, $M+1$ additional hidden units pass its coordinates unchanged through ReLU and write $-a$ back to $B_3^{\mathrm{agg}}$. The residual connection clears $B_3^{\mathrm{agg}}$, leaving the exact state one-hots in $B_3^{\mathrm{pat}}$.

Finally, the threshold bank uses
\[
    2|\Uset|M(V-1)
\]
hidden units, aggregate cancellation uses $M+1$, and the constant uses one,
giving
\[
    \dmlp
    \leq
    2|\Uset|M(V-1)+M+2.
\]
The positional table consists of the $\dS+2$ spatial-temporal coordinates, the
initial-step indicator $\zeta_n$ of \eqref{eq:zeta}, the $|\Uset|(M-1)+1$
query-side color-address coordinates, and the width-$M$ own-color block
\eqref{eq:owncolour}, so
\[
    \dpos
    \leq
    \dS+3+|\Uset|(M-1)+1+M .
\]
The residual stream contains
\[
    3+V+\dpos+\bigl(|\Uset|V+M+1\bigr)+V
\]
coordinates, yielding
\[
    \dmodel
    =
    (|\Uset|+2)V+\dpos+M+4.
\]
Thus all $|\Uset|$ pattern blocks are recovered exactly and
condition~\emph{(i)} of Definition~\ref{def:sih-app} holds with
$\eblock=0$.
\end{proof}

\begin{rmk}
The factor $M$ in the one-hidden-layer width stems from performing color selection and state thresholding within a single nonlinear stage. A two-hidden-layer architecture could decouple these (selecting the color first, then thresholding), reducing the width to
\[
    |\Uset|(M+2V-2).
\]
We present the wider one-hidden-layer construction to match the architecture in \S\ref{sec:task:model}.
\end{rmk}

\begin{corollary}[Grid-independent width bounds]
\label{cor:onehead-gridfree}
Take $c$ to be the equitable degree-bound coloring supplied by
Corollary~\ref{cor:colour-interface}\emph{(d)}, so that
\[
    M
    \leq
    |\Uset|^2-|\Uset|+1
    \qquad\text{and}\qquad
    n_{\max}\leq T\lceil\lvert\Grid\rvert/M\rceil .
\]
Then Theorem~\ref{thm:onehead-app} holds with
\[
    \dmlp
    \leq
    2|\Uset|
    \bigl(|\Uset|^2-|\Uset|+1\bigr)(V-1)
    +
    |\Uset|^2-|\Uset|+3,
\]
and, writing $m\coloneqq|\Uset|$,
\[
    \dpos
    \leq
    \dS
    +
    m^{3}-m+5 ,
\]
where
\[
    \dS
    =
    O\bigl(|\Uset|\,d\log|\Uset|\bigr).
\]
Thus, both widths depend solely on $|\Uset|$, $V$, and $d$. For the three experimental grids, the periodic colorings of Table~\ref{tab:instance-colour} yield smaller values than this degree bound; Theorem~\ref{thm:onehead-app} holds with those tighter $M$ values.
\end{corollary}

\begin{proof}
We obtain the MLP bound by substituting $M\leq m^2-m+1$ into
Theorem~\ref{thm:onehead-app}. For the positional bound, the color coordinates of
Theorem~\ref{thm:onehead-app} total $m(M-1)+M=(m+1)M-m$, which is increasing in
$M$, so at $M=m^2-m+1$ the sum-of-cubes identity
$(m+1)(m^2-m+1)=m^3+1$ gives
\[
    (m+1)M-m
    \;\leq\;
    m^{3}+1-m .
\]
Adding the $\dS+3$ addressing and indicator coordinates and the single shared
constant already counted in $m(M-1)+1$ yields
$\dpos\leq\dS+m^3-m+5$.
\end{proof}

\begin{rmk}
We state the bound $m^3-m+5$ for precision: since $m\geq2$, $m^3-m+5\leq m^3+3$ (equality at $m=2$). Here $m\geq2$ follows from $W_d\geq2$ in Definition~\ref{def:ca}: either $\Uset\supsetneq D_K$, or $D_K+\usucc=D_K$, in which case $\Uset=D_K$ is a union of cosets of $\langle\usucc\rangle$, each of size $W_d\geq2$. The own-color and indicator coordinates are therefore absorbed by the grid-independent form in Theorem~\ref{thm:gather}(\ref{thm:onehead}).
\end{rmk}

\subsection{The spatial gather interface}
\label{app:con:interface}

The two constructions differ only in how they produce the same Layer-1
representation. In Appendix~\ref{app:synthesis}, we therefore treat them
through the following common interface.

\begin{corollary}[Spatial gather interface]
\label{cor:gather-interface}
Either Layer-1 construction realizes condition~\emph{(i)} of
Definition~\ref{def:sih-app}, at every position $n$ with $\tau(n)\geq1$ and for
every $\delta\in\Uset$:
\begin{itemize}[leftmargin=1.6em,itemsep=2pt,topsep=3pt]
    \item Construction~I uses $|\Uset|$ heads and no MLP, with
    \[
        \eblock\leq Ne^{-\Lambda\gamma},
    \]
    at the widths of Theorem~\ref{thm:parallel-app}.

    \item Construction~II uses one head and a one-hidden-layer separating MLP,
    with
    \[
        \eblock=0,
    \]
    at the widths of Theorem~\ref{thm:onehead-app} and the grid-independent
    bounds of Corollary~\ref{cor:onehead-gridfree}.
\end{itemize}
\end{corollary}

\begin{proof}
Immediate from Theorems~\ref{thm:parallel-app} and~\ref{thm:onehead-app}.
\end{proof}

Layer~2 also reads the pattern scratch of invalid keys, including positions with
$\tau(n)=0$, where condition~\emph{(i)} makes no claim. The Layer-2 gate
requires only the following weaker property there.

\begin{lemma}[Pattern scratch at every position]
\label{lem:scratch-all}
Under either construction, at every ordinary position $n$ other than the final
\texttt{TIME\_SEP}, including those with $\tau(n)=0$, every width-$V$ block of $B_3^{\mathrm{pat}}$ is a sub-probability
vector. At $\star$, $B_3^{\mathrm{pat}}=0$.
\end{lemma}

\begin{proof}
The claim at $\star$ is Lemma~\ref{lem:star-l1}\emph{(c)}.

\emph{Construction~I.} Block $\delta$ is
$\sum_{n'}\pi_{n'}\Pi_{B_1}x_{n'}$ with softmax weights $\pi_{n'}\geq0$ summing
to $1$, and each $\Pi_{B_1}x_{n'}$ is a one-hot or zero.

\emph{Construction~II.} Fix $\delta$. The wrong-slot bound
$P_{\delta,s,v}\leq0$ of Step~4 uses only $a[s]\leq1$ and so holds at every
position. By Corollary~\ref{cor:colour-interface}\emph{(b)}, $g_\delta(z)$ is a
one-hot at every ordinary position, so at most one slot $s_\delta$ has
$\chi_{\delta,s,v}\neq0$; write $\chi_v\coloneqq\chi_{\delta,s_\delta,v}$ and
$\chi_V\coloneqq0$. Since $a[\star]=w_\star/Z>0$ and $w_\delta/w_\star>0$,
$C_{\delta,s_\delta,v}$ is nonincreasing in $v$; $P$ is increasing in $C$ and the
clamp is nondecreasing, so $1\geq\chi_1\geq\chi_2\geq\dots\geq\chi_V=0$. Hence
coordinate $v\geq1$ of the block, $\chi_v-\chi_{v+1}$, is non-negative;
coordinate $0$, $1-\chi_1$, is non-negative; and the coordinates telescope to
$(1-\chi_1)+(\chi_1-\chi_V)=1$. The block is therefore a probability vector.
\end{proof}

\subsection{Construction sizes for the experimental settings}
\label{app:con:instance}

Table~\ref{tab:instance-sizes} evaluates the two constructions on the three
experimental settings of Appendix~\ref{app:not:instance}, at the colorings of
Table~\ref{tab:instance-colour}. Three entries need comment.

First, the addressing dimension for Construction~II is omitted because the
bound of Theorem~\ref{thm:isolation} is vacuous on these grids: at
$\lvert\Uset\rvert = 8$ and $d = 2$ it exceeds $\lvert\Gd\rvert = 36$. Four
characters suffice to isolate $\Uset$ on $\mathbb{Z}_6^2$
(\S\ref{app:addr:instance}), and the general bound is useful when
$\lvert\Uset\rvert\,d\log\lvert\Uset\rvert$ is much smaller than the volume.
For $d = 1$ the bounding box equals the routing set, so
Lemma~\ref{lem:degreeone}\emph{(i)} applies directly and yields
$\dS \leq 2d\lvert\Uset\rvert = 8$ without randomization.

Second, the $\dmlp$ row reports the one-hidden-layer width
$2\lvert\Uset\rvert M(V-1)+M+2$ of Theorem~\ref{thm:onehead-app}. The
two-hidden-layer alternative of the remark after that theorem would give
$\lvert\Uset\rvert(M+2V-2)$ instead, which is $24$ against $38$ for 1D
$V{=}2$, $32$ against $70$ for 1D $V{=}3$, and $112$ against $206$ for 2D von
Neumann.

Third, the margin $\gamma$ and the logit scale $\Lambda$ belong to
Construction~I, whose heads address one offset each. There
Lemma~\ref{lem:degreeone}\emph{(ii)} gives $\gamma = 1-\cos(2\pi/\Wmax)$, and
Lemma~\ref{lem:blockerror} at the block-error threshold on the line above gives
$\Lambda \geq \gamma^{-1}\log\bigl((2k+2)N\bigr)$. Construction~II addresses
all of $\Uset$ with a single head, so it has a different margin and asks for
\eqref{eq:onehead-scale} instead; Example~\ref{ex:union} works that case out on
$\mathbb{Z}_6^2$.

\begin{table}[t]
\centering
\caption{What the two constructions require on the three experimental
settings, at the colorings of Table~\ref{tab:instance-colour}.}
\label{tab:instance-sizes}
\small
\begin{tabular}{@{}lccc@{}}
\toprule
& 1D, $V{=}2$ & 1D, $V{=}3$ & 2D von Neumann \\
\midrule
\multicolumn{4}{@{}l}{\emph{Construction~I} (Thm.~\ref{thm:parallel-app})}\\
heads, $\dmlp$                          & $4$, $0$ & $4$, $0$ & $8$, $0$ \\
$\dS = 2d$, $\dpos = \dS{+}3$           & $2$, $5$ & $2$, $5$ & $4$, $7$ \\
width of $B_3 = \lvert\Uset\rvert V$    & $8$ & $12$ & $16$ \\
residual $\dmodel$                      & $20$ & $26$ & $30$ \\
block-error threshold $1/(2k{+}2)$      & $1/8$ & $1/8$ & $1/12$ \\
margin $\gamma$                         & $0.076$ & $0.019$ & $0.5$ \\
logit scale $\Lambda \geq$              & $95$ & $420$ & $18$ \\
\addlinespace
\multicolumn{4}{@{}l}{\emph{Construction~II} (Thm.~\ref{thm:onehead-app})}\\
colors $M$                              & $4$ & $4$ & $12$ \\
$\dmlp \leq 2\lvert\Uset\rvert M(V{-}1) + M + 2$ & $38$ & $70$ & $206$ \\
width of $B_3^{\mathrm{agg}} = M{+}1$   & $5$ & $5$ & $13$ \\
\bottomrule
\end{tabular}
\end{table}

\apppart{Part IV}{Layer 2 and end-to-end guarantees}
\section{Layer 2: Bayesian matching and end-to-end guarantees}
\label{app:synthesis}

Per Appendix~\ref{app:constructions}, Layer~1 endows each query token with one width-$V$ pattern block per offset $\delta\in\Uset$. Layer~2 then compares the query configuration at the current token against the parent configuration at every preceding token, retrieving the states that follow matching precedents.

With exact one-hot pattern blocks, this comparison is immediate: the inner product between two aligned $k$-block configurations equals the number of agreeing offsets. An exact match attains score $k$; every non-match scores at most $k-1$. If the softmax isolates exact matches, equal scores compel it to average their state one-hots, yielding the empirical conditional frequency. An additional token injects the Dirichlet pseudocount, transforming this frequency into the posterior predictive.

Finite Layer-1 error perturbs this ideal in two ways: it narrows the score gap between matches and non-matches, and it disperses the scores among true matches, preventing uniform softmax averaging. The Layer-2 scale $\beta$ amplifies both effects: increasing $\beta$ suppresses non-matches but exacerbates dispersion within the matched set.

This section is structured in six parts. First, we formalize the matching head (\S\ref{app:syn:head}). Second, we prove the exact counting and Dirichlet smoothing identity in the zero-error limit (\S\ref{app:syn:bayesmatch}). Third, we bound the two finite-error perturbations: the loss of separation (\S\ref{app:syn:separation}) and the dispersion among matched scores with its predictive cost (\S\ref{app:syn:budget}). Fourth, we configure the key-validity gate based on these bounds (\S\ref{app:syn:keygate}). Fifth, we balance the perturbations to establish the end-to-end spatial induction theorem (\S\ref{app:syn:mainthm}). Finally, we detail the recipe for setting the finite parameters (\S\ref{app:syn:params}).

\subsection{The configuration-matching head}
\label{app:syn:head}
\label{app:con:l2match}

While the Layer-1 representation stores all offsets in $\Uset$, the Layer-2 query and key require distinct subsets: the current token exposes the parent configuration of its successor (indexed by $D_Q=D_K+\usucc$), whereas an earlier token exposes its own parent configuration (indexed by $D_K$).

\begin{definition}[The Layer-2 head]
\label{def:l2match}
Define the selection projections on $B_3^{\mathrm{pat}}$
\[
    \Sigma_Q
    \coloneqq
    \bigl(\text{block }\usucc+\delta\bigr)_{\delta\in D_K},
    \qquad
    \Sigma_K
    \coloneqq
    \bigl(\text{block }\delta\bigr)_{\delta\in D_K}.
\]
Thus $\Sigma_Q x_{t,z}$ and $\Sigma_K x_{t',z'}$ are both concatenations of
$k$ width-$V$ blocks in the same offset order: the former represents the
query configuration of the target $z+\usucc$, and the latter represents the
earlier token's parent configuration.

A position $n'$ is a \emph{valid key} if it is a cell token with $\tau(n')\geq1$; otherwise, it is an \emph{invalid key}. Invalid keys (separators and initial-step cells) record no transitions of $f$: separators carry no state, while initial-step cells lack a defined parent configuration $P_{0,z'}$ and their states are not drawn from $f$. This matches the restriction $t'\geq1$ on the counts \eqref{eq:counts}. Write
\begin{equation}
\label{eq:nu}
    \nu_{n'}
    \coloneqq
    s_{n'}+\zeta_{n'}
    \;\in\;\{0,1,2\},
\end{equation}
which is $0$ at a valid key and at least $1$ at an invalid one; both summands
are already present in the stream, $s$ in $B_0$ and $\zeta$ in $B_2$
(Definition~\ref{def:stream}).

For query scale $\beta>0$ and key-validity gate $\Gamma_2>0$, define
\begin{gather*}
    W_Q^{(2)}
    =
    \beta\,\Sigma_Q
    \oplus
    \Pi_{B_0,\mathbf 1},
    \qquad
    W_K^{(2)}
    =
    \Sigma_K
    \oplus
    \bigl(\sigma_\star\Pi_{\mu}-\Gamma_2\Pi_{\nu}\bigr),
    \\
    W_V^{(2)}
    =
    \Pi_{B_1}
    +
    \tfrac1V\mathbf 1\cdot\bigl(\Pi_{B_0,\mathrm{sep}}+\Pi_{B_0,\mu}\bigr) .
\end{gather*}
For a valid key $(t',z')$, the score is
\[
    \beta
    \bigl\langle
        \Sigma_Q x_{t,z},
        \Sigma_K x_{t',z'}
    \bigr\rangle,
\]
and the value is its state one-hot $e(x_{t',z'})$. For an invalid key, $\Gamma_2\nu_{n'}\geq \Gamma_2$ is subtracted from this score (Proposition~\ref{prop:keygate-l2} bounds $\Gamma_2$). The rank-one term in $W_V^{(2)}$ assigns separators and $\star$ the value $V^{-1}\mathbf 1$ rather than zero, ensuring the head outputs a probability vector (the readout convention of
\S\ref{app:syn:bayesmatch}). The auxiliary token $\star$ is neither a valid nor an invalid key: its pattern scratch is zero (Lemma~\ref{lem:scratch-all}) and $\nu_\star=0$, so its score is $\sigma_\star$ at every query. The output is written to $B_4$, without a Layer-2 MLP.
\end{definition}

The projections $\Sigma_Q$ and $\Sigma_K$ do not create new residual
coordinates: they only select and reorder blocks already present in
$B_3^{\mathrm{pat}}$. Neither does the gate, since $\nu$ is the sum of two
coordinates the stream already carries; and neither does the value term, which
is rank one.

For a fixed query configuration $c=Q_{t,z}$, let
$\mathcal M_{t,z}$ denote the set of earlier valid keys, that is cell-token
positions $(t',z')$ with $t'\geq1$, whose parent configuration equals $c$. Its
cardinality is $n_c$, matching \eqref{eq:counts} exactly, and
$n_{c,s}$ denotes the number of those precedents whose own state is
$s\in\States$.

\paragraph{The score when Layer~1 is exact.}
If all selected blocks are exact one-hots, the inner product
\[
    \bigl\langle
        \Sigma_Q x_{t,z},
        \Sigma_K x_{t',z'}
    \bigr\rangle
\]
counts the agreeing offsets in $D_K$. Thus, members of $\mathcal M_{t,z}$ score $\beta k$, while non-members score at most $\beta(k-1)$. This integer combinatorial margin persists (albeit degraded) even when Layer-1 blocks are approximate.

The gate strength $\Gamma_2$ is set in \S\ref{app:syn:keygate},
once the matched-score floor and the separation are available.

\subsection{Exact counting and Dirichlet smoothing}
\label{app:syn:bayesmatch}
\label{app:syn:counting}
\label{app:syn:sink}

We first analyze the ideal computation, where exact matches receive equal scores and all other positions are excluded. In this limit, the attention head acts as a counter.

The following identity governs the hard-attention limit and requires three hypotheses, which an exact gather alone cannot guarantee. While an exact gather ensures the first (equal one-hot configurations yield equal inner products), a finite softmax retains mass everywhere (making the second a strict limit), and the third necessitates a sink token. Theorem~\ref{thm:main-app} establishes the finite-parameter counterpart, offering an approximation with explicit error bounds.

\begin{theorem}[Exact Bayesian matching in the hard-attention limit;
Proposition~\ref{thm:sink}]
\label{thm:bayesmatch}
\label{lem:average}
\label{thm:sink-app}
Suppose the Layer-2 head of Definition~\ref{def:l2match} satisfies:
\begin{enumerate}[label=(H\arabic*),leftmargin=2.6em,itemsep=1pt,topsep=2pt]
  \item it assigns a common score $a$ to all $n_c$ matching precedents in
        $\mathcal M_{t,z}$;
  \item it assigns score $-\infty$ to every other position, non-matching valid
        keys and invalid keys alike;
  \item for part~\emph{(ii)}, the sink carries the prescribed value and score.
\end{enumerate}
Then:
\begin{enumerate}[label=(\roman*),leftmargin=2.2em,itemsep=2pt,topsep=3pt]
    \item \emph{Empirical counting.}
    With value $e(x_{t',z'})$ at precedent $(t',z')$, the head output is
    \[
        \frac1{n_c}
        \sum_{(t',z')\in\mathcal M_{t,z}}
        e(x_{t',z'})
        =
        \left(
            \frac{n_{c,s}}{n_c}
        \right)_{s\in\States}.
    \]

    \item \emph{Dirichlet smoothing.}
    Fix $\alpha>0$. Add one auxiliary token $\star$ with value
    \[
        \frac1V\mathbf 1
    \]
    and score
    \[
        a+\log(V\alpha).
    \]
    Then the head output is exactly
    \[
        \left(
            \frac{n_{c,s}+\alpha}{n_c+V\alpha}
        \right)_{s\in\States}
        =
        \widehat P_\alpha(\cdot\mid Q_{t,z}),
    \]
    the posterior predictive distribution of Theorem~\ref{thm:bayes}.
\end{enumerate}
No Layer-2 MLP is required.
\end{theorem}

\begin{proof}
For part~(i), equal scores induce a uniform softmax over the $n_c$ matching precedents; averaging their state one-hots yields the normalized state counts.

For part~\emph{(ii)}, every matching precedent has unnormalized weight $e^a$,
while $\star$ has weight
\[
    e^{a+\log(V\alpha)}
    =
    V\alpha e^a.
\]
The total unnormalized mass is therefore
$(n_c+V\alpha)e^a$. In coordinate $s$, the matching precedents contribute
$n_{c,s}e^a$, while the auxiliary token contributes
$V\alpha e^a\cdot(1/V)=\alpha e^a$. Hence
\[
    \frac{
        n_{c,s}e^a+\alpha e^a
    }{
        (n_c+V\alpha)e^a
    }
    =
    \frac{n_{c,s}+\alpha}{n_c+V\alpha}.
\]
\end{proof}

\paragraph{The auxiliary token.}
The Dirichlet sink is the token $\star$ of Lemma~\ref{lem:star-l1}. By
Definition~\ref{def:l2match}, its Layer-2 score is $\sigma_\star$ at every
query and its value is $V^{-1}\mathbf 1$. For $\alpha>0$ we set
$\sigma_\star=a+\log(V\alpha)$. In the deterministic limit $\alpha\to0$ the
Dirichlet role vanishes, but in Construction~II the Layer-1 reference role of
Lemma~\ref{lem:star-l1}\emph{(a)} does not, so $\star$ is retained; we then set
$\sigma_\star\leq a-\Delta$, and $\star$ enters
Theorem~\ref{thm:main-app}\emph{(ii)} as an extraneous position. Construction~I
includes $\star$ only when $\alpha>0$, with the Layer-1 score of
Lemma~\ref{lem:star-l1}\emph{(b)}. In either case $N$ counts $\star$ when it is
present, and the Layer-1 guarantees are unchanged by
Lemma~\ref{lem:star-l1}.

\paragraph{Readout convention.}
Every value used by Layer~2 is a probability vector, by construction. Cell
keys, valid or not, carry state
one-hots; the sink $\star$ carries the uniform vector $V^{-1}\mathbf 1$; and
separators carry $V^{-1}\mathbf 1$ as well, by the rank-one term added to
$W_V^{(2)}$ in Definition~\ref{def:l2match}. Had separators kept the value $0$
that their empty $B_1$ would otherwise supply, the retrieved vector would have
had total mass $1-\varepsilon_{\mathrm{sep}}<1$ at every finite gate strength,
since a finite softmax places positive weight on every unmasked position, and
it could not have been read as a distribution without an explicit
renormalization. As constructed, the vector written to $B_4$ is a convex
combination of probability vectors and so lies exactly in the probability
simplex. For stochastic claims we take
\[
    \Pi_{B_4}x_n^{(2)}
\]
directly as the predictive distribution; we refer to this as the
\emph{simplex readout}.

A standard language-model head $\operatorname{softmax}(W_Ux_n^{(2)})$ preserves the predicted state for $W_U=\Pi_{B_4}$ (since softmax is monotonic), but generally distorts the full distribution. Consequently, our exact distributional bounds invoke the simplex readout, while deterministic $\arg\max$ statements hold for either.

Theorem~\ref{thm:bayesmatch} identifies the ideal computation. With finite
weights, two departures from that ideal remain: leakage from non-matching
positions, and score dispersion among the matches. We bound them separately.

\subsection{Separation of matches from non-matches}
\label{app:syn:separation}

The first question is whether approximate pattern blocks still leave a
positive score gap between true matches and non-matches.

\begin{lemma}[Separation of matching configurations]
\label{lem:noise}
Let
\[
    P=(P_1,\dots,P_k),
    \qquad
    P'=(P'_1,\dots,P'_k)
    \in\mathbb R^{kV},
\]
where each $P_h$ is an $\eblock$-approximate one-hot for $e(y_h)$ and each
$P'_h$ is an $\eblock$-approximate one-hot for $e(y'_h)$ in the sense of
Definition~\ref{def:eclose}. Write
\[
    y=(y_1,\dots,y_k),
    \qquad
    y'=(y'_1,\dots,y'_k).
\]
Then
\[
    \langle P,P'\rangle
    \geq
    k-2k\eblock
    \qquad\text{if }y=y',
\]
while
\[
    \langle P,P'\rangle
    \leq
    (k-1)+2\eblock
    \qquad\text{if }y\neq y'.
\]
Hence the two score ranges are disjoint whenever
\[
    \eblock
    <
    \frac1{2k+2}.
\]
Under the Layer-2 scale $\beta$, every matching position then outscores every
non-matching position by at least
\[
    \Delta
    \coloneqq
    \beta\bigl(1-(2k+2)\eblock\bigr).
\]
\end{lemma}

\begin{proof}
We use two elementary blockwise bounds.

First, every block inner product is at most $1$. Indeed, since $P_h$ has total
mass at most $1$ and every coordinate of $P'_h$ lies in $[0,1]$,
\begin{equation}
\label{eq:blockle1}
    \langle P_h,P'_h\rangle
    \leq
    \Bigl(\max_v P'_h[v]\Bigr)
    \sum_v P_h[v]
    \leq
    1.
\end{equation}

Now suppose the two blocks represent the same state,
$y_h=y'_h$. Keeping only their common target coordinate gives
\[
    \langle P_h,P'_h\rangle
    \geq
    P_h[y_h]P'_h[y_h]
    \geq
    (1-\eblock)^2
    \geq
    1-2\eblock.
\]
If $y=y'$, summing this bound over all $k$ blocks yields
\[
    \langle P,P'\rangle
    \geq
    k-2k\eblock.
\]

Next suppose $y_h\neq y'_h$ for at least one block $h$. Since $y_h$ is an
off-target coordinate for $P'_h$,
\[
    P'_h[y_h]\leq\eblock.
\]
Therefore
\[
\begin{aligned}
    \langle P_h,P'_h\rangle
    &=
    P_h[y_h]P'_h[y_h]
    +
    \sum_{v\neq y_h}P_h[v]P'_h[v]
    \\
    &\leq
    \eblock
    +
    \sum_{v\neq y_h}P_h[v]
    \\
    &\leq
    2\eblock,
\end{aligned}
\]
because $P_h$ places at most $\eblock$ total mass away from its own target
coordinate. The remaining $k-1$ blocks contribute at most $1$ each by
\eqref{eq:blockle1}, so
\[
    \langle P,P'\rangle
    \leq
    (k-1)+2\eblock.
\]

The ranges are disjoint exactly when
\[
    k-2k\eblock
    >
    (k-1)+2\eblock,
\]
equivalently $\eblock<1/(2k+2)$. Multiplying the difference by $\beta$ gives
the stated gap $\Delta$.
\end{proof}

For deterministic rules, this gap is all we need: every
matching precedent carries the same correct next-state value.

\begin{corollary}[Deterministic matching]
\label{lem:det-retrieval}
Let the rule be deterministic, let there be $n_c\geq1$ matching precedents
among at most $N$ visible positions, and let $\Delta$ be the gap of
Lemma~\ref{lem:noise}. Assume every other visible position scores at least
$\Delta$ below every matching precedent, whether it is a non-matching valid
key, a separator, an initial-step cell, or $\star$. Lemma~\ref{lem:noise}
supplies this for non-matching valid keys; Proposition~\ref{prop:keygate-l2} of
\S\ref{app:syn:keygate} supplies it for the invalid keys, and the choice
$\sigma_\star\leq a-\Delta$ for $\star$; until then the assumption is carried
as a hypothesis. Then the Layer-2 head places mass at
least $1-\eret$ on the matching precedents, where
\[
    \eret
    \coloneqq
    \frac{
        (N-n_c)e^{-\Delta}
    }{
        n_c+(N-n_c)e^{-\Delta}
    }.
\]
If $u^\star=e(f(Q_{t,z}))$ is their common value, the output $\hat y$ obeys
\[
    \|\hat y-u^\star\|_2
    \leq
    2\eret.
\]
In particular, if
\[
    \Delta
    >
    \log\bigl(2\sqrt2\,(N-n_c)/n_c\bigr),
\]
then the predicted state is exactly correct.
\end{corollary}

\begin{proof}
Let $r^+_{\min}$ be the smallest matched score. By hypothesis every other
visible position scores at most $r^+_{\min}-\Delta$, so the total mass $w$ on
the matching set satisfies
\[
    w
    \geq
    \frac{
        n_c e^{r^+_{\min}}
    }{
        n_c e^{r^+_{\min}}
        +
        (N-n_c)e^{r^+_{\min}-\Delta}
    }
    =
    1-\eret.
\]
All matching precedents carry the same value $u^\star$, so the head output can
be written
\[
    \hat y
    =
    w u^\star
    +
    (1-w)\rho
\]
for some probability vector $\rho$. Hence
\[
    \|\hat y-u^\star\|_2
    =
    (1-w)\|\rho-u^\star\|_2
    \leq
    2(1-w)
    \leq
    2\eret.
\]
The stated condition makes this perturbation smaller than the distance from
$u^\star$ to any decision boundary between two basis vectors, and therefore
preserves its $\arg\max$.
\end{proof}

\begin{rmk}
The displayed threshold is only a sufficient condition and is not tight. Its
role is to provide a finite explicit scale compatible with the statement of
Theorem~\ref{thm:main-app}\emph{(ii)}.
\end{rmk}

\subsection{Matched-score dispersion and predictive error}
\label{app:syn:budget}

For stochastic rules, mere separation is insufficient. Since matching precedents may carry distinct next states, recovering the true counts requires the softmax to remain nearly uniform \emph{within} the matched set. Approximate Layer-1 blocks perturb these matched scores.

\begin{lemma}[Score dispersion among matches]
\label{lem:dispersion}
In the setting of Lemma~\ref{lem:noise}, define
\[
    a
    \coloneqq
    \beta(k-2k\eblock).
\]
Every matched score can be written
\[
    a+\eta_j,
    \qquad
    0\leq\eta_j\leq\eta,
    \qquad
    \eta
    \coloneqq
    2\beta k\eblock,
\]
and every non-matching score is at most
\[
    a-\Delta,
    \qquad
    \Delta
    =
    \beta\bigl(1-(2k+2)\eblock\bigr).
\]
\end{lemma}

\begin{proof}
For a matching pair, Lemma~\ref{lem:noise} gives the lower bound
$k-2k\eblock$, while \eqref{eq:blockle1} gives the upper bound $k$. Hence the
matched Layer-2 scores lie in
\[
    \bigl[
        \beta(k-2k\eblock),
        \beta k
    \bigr],
\]
an interval of width $2\beta k\eblock$. Taking the lower endpoint as $a$
gives the claimed representation $a+\eta_j$.

For a non-match, Lemma~\ref{lem:noise} gives score at most
\[
    \beta\bigl((k-1)+2\eblock\bigr)
    =
    \beta(k-2k\eblock)
    -
    \beta\bigl(1-(2k+2)\eblock\bigr)
    =
    a-\Delta.
\]
\end{proof}

Layer-1 error thus introduces a tradeoff between:
\[
    \Delta
    =
    \beta\bigl(1-(2k+2)\eblock\bigr),
    \qquad
    \eta
    =
    2\beta k\eblock.
\]
A large $\Delta$ suppresses non-matches, while a small $\eta$ maintains uniform weighting among matches. Increasing $\beta$ scales both simultaneously.

We next quantify the cost of the second effect.

\begin{lemma}[KL error from matched-score dispersion]
\label{lem:biasbudget}
In the setting of Theorem~\ref{thm:bayesmatch}\emph{(ii)}, suppose the matching
scores are $a+\eta_j$ with
$|\eta_j|\leq\eta$, while the sink remains at
$a+\log(V\alpha)$. Let $p_\eta$ denote the output after restricting attention
to the matching precedents and sink. Then, for every $s\in\States$,
\[
    p_\eta(s)
    \geq
    e^{-2\eta}
    \widehat P_\alpha(s\mid Q_{t,z}),
\]
and therefore
\[
    \KL\!\left(
        \widehat P_\alpha(\cdot\mid Q_{t,z})
        \,\middle\|\,
        p_\eta
    \right)
    \leq
    2\eta.
\]
No small-$\eta$ assumption is required.
\end{lemma}

\begin{proof}
The $s$-th coordinate of the dispersed matched-plus-sink output is
\[
    p_\eta(s)
    =
    \frac{
        \sum_{j:x_j=s}e^{\eta_j}
        +
        \alpha
    }{
        \sum_j e^{\eta_j}
        +
        V\alpha
    }.
\]
Because $e^{-\eta}\leq e^{\eta_j}\leq e^\eta$ and
$e^{-\eta}\leq1\leq e^\eta$, the numerator lies in
\[
    \bigl[
        e^{-\eta}(n_{c,s}+\alpha),
        \;
        e^\eta(n_{c,s}+\alpha)
    \bigr],
\]
while the denominator lies in
\[
    \bigl[
        e^{-\eta}(n_c+V\alpha),
        \;
        e^\eta(n_c+V\alpha)
    \bigr].
\]
Thus
\[
    p_\eta(s)
    \geq
    e^{-2\eta}
    \frac{n_{c,s}+\alpha}{n_c+V\alpha}
    =
    e^{-2\eta}
    \widehat P_\alpha(s\mid Q_{t,z}).
\]
Taking logarithms and averaging under $\widehat P_\alpha$ gives the KL bound.
\end{proof}

\begin{rmk}
Lemma~\ref{lem:dispersion} actually gives $\eta_j\geq0$, in which case the
constant can be improved. We keep the symmetric $2\eta$ form because it makes
the perturbation lemma independent of how the matched-score interval is
anchored and matches the statement used below.
\end{rmk}

\subsection{The key-validity gate}
\label{app:syn:keygate}

With the matched-score floor $a$ and the separation $\Delta$ now fixed, we
can set the gate strength $\Gamma_2$ which puts every invalid key under the
same ceiling as a non-matching valid key.

\begin{proposition}[The Layer-2 key-validity gate]
\label{prop:keygate-l2}
In the notation of Lemma~\ref{lem:dispersion}, let
$a\coloneqq\beta(k-2k\eblock)$ be the matched-score floor and $\Delta$ the
separation. If
\begin{equation}
\label{eq:gamma2}
    \Gamma_2
    \;\geq\;
    \beta k-a+\Delta
    \;=\;
    2\beta k\eblock+\Delta ,
\end{equation}
then every invalid key scores at most $a-\Delta$, the same ceiling that
Lemma~\ref{lem:dispersion} gives for non-matching valid keys. Moreover every
key's Layer-2 value is a probability vector.
\end{proposition}

\begin{proof}
\emph{The gate.}
An invalid key may carry nonzero pattern scratch, meaning its ungated score is not necessarily small (separators inherit virtual cells and receive Layer-1 gathers, while initial-step cells assemble gathers from the cells they can see). However, \eqref{eq:blockle1} bounds the score: every block of $B_3^{\mathrm{pat}}$ is a sub-probability vector at every ordinary position, including those with $\tau(n')=0$ (Lemma~\ref{lem:scratch-all}), so each of the $k$ inner products is at most $1$, yielding
\[
    \beta\bigl\langle\Sigma_Q x_{t,z},\Sigma_K x_{n'}\bigr\rangle
    \;\leq\;
    \beta k
    \eqqcolon
    S_{\max}
\]
for every key, valid or not. Since $\nu_{n'}\geq1$ at an invalid key, its
gated score is at most $S_{\max}-\Gamma_2\leq a-\Delta$ by \eqref{eq:gamma2}.

\emph{The values.}
A valid key has $s_{n'}=0$, so its value is $\Pi_{B_1}x_{n'}=e(x_{t',z'})$, a
one-hot. A separator has $s_{n'}=1$ and $\Pi_{B_1}x_{n'}=0$, so its value is
$V^{-1}\mathbf 1$. An initial-step cell has $s_{n'}=0$ and carries its own state
one-hot $e(x_{0,z'})$. The token $\star$ has $\mu_\star=1$ and
$\Pi_{B_1}x_\star=0$, so its value is $V^{-1}\mathbf 1$. All four are
probability vectors.
\end{proof}

\begin{rmk}
\label{rmk:gate-consequences}
Two consequences simplify the remainder of this appendix.

First, invalid keys require no separate error budget. Gated below $a-\Delta$, they become indistinguishable from non-matching valid keys and are absorbed into the $N$ visible positions. Theorem~\ref{thm:main-app} bounds their mass collectively, eliminating separate $\varepsilon_{\mathrm{sep}}$ bookkeeping.

Second, because all values are probability vectors, the head output (a convex combination) lies exactly in the simplex. The readout convention of \S\ref{app:syn:bayesmatch} dictates why separators cannot retain the value $0$.

Like the Layer-1 gate, this mechanism uses a rank-one term over existing stream coordinates, avoiding any width increase.
\end{rmk}

\begin{remark}[The initial-step half of the gate is not optional]
\label{rmk:initstep-counterexample}
Gating separators alone falsifies the end-to-end theorem. Take
$D_K=\{-1,0,1\}$, $\Uset=\{-1,0,1,2\}$, $V=2$, eight cells, and the constant
rule $f\equiv0$, on the trajectory $x_0=11111111$, $x_1=x_2=00000000$. Query
the first cell of $x_1$: its query configuration is $111$, the transition
$111\mapsto0$ has been observed, so the query is resolvable with $n_c=1$ and
the answer is $0$. Every causally visible cell of the initial time step has
state $1$, so every Layer-1 head returns $e(1)$ there, and each of the eight
initial-step tokens presents the spurious parent configuration $111$ while
carrying its own state $1$. Layer~2 sees eight spurious matches valued $1$
against one genuine match valued $0$, all at the same score, and returns
$p(1)=8/9$. This survives at finite parameters meeting every other condition:
at $\Lambda=100$ and $\beta=20$ one has $\eblock\leq5.15\times10^{-12}$ and
$\Delta\approx20$, far above the threshold
$\log(2\sqrt2(N-n_c)/n_c)\approx4.30$ of
Theorem~\ref{thm:main-app}\emph{(ii)}. The $\zeta$ half of $\nu$ in
\eqref{eq:nu} removes the eight spurious keys.
\end{remark}

\subsection{End-to-end spatial induction}
\label{app:syn:mainthm}

We now have all the pieces of the finite-parameter argument; Layer~1 fixes the
per-block error $\eblock$, which in turn sets the separation gap $\Delta$
between matches and non-matches and the dispersion $\eta$ within the matched
set. The sink token supplies the Dirichlet smoothing. It remains to choose one
scale $\beta$ which suppresses the non-matching mass and bounds the
matched-set dispersion at the same time.

\begin{theorem}[End-to-end spatial induction; Theorem~\ref{thm:main} of the
body]
\label{thm:main-app}
\label{cor:deterministic}
Let Layer~1 be either construction of
Corollary~\ref{cor:gather-interface}, at per-block error $\eblock$, with the
separator gate of Proposition~\ref{prop:sepgate-l1} at strength
$\Gamma_1\geq\Lambda(R+\gamma)$. Let Layer~2 be the head of
Definition~\ref{def:l2match} with scale $\beta$ and key-validity gate
$\Gamma_2\geq2\beta k\eblock+\Delta$, as in
Proposition~\ref{prop:keygate-l2}. Write
\[
    \Delta
    \coloneqq
    \beta\bigl(1-(2k+2)\eblock\bigr),
    \qquad
    a\coloneqq\beta(k-2k\eblock),
\]
and, whenever $\star$ is present, let it be configured as in
Lemma~\ref{lem:star-l1} with Layer-2 score $\sigma_\star$.

\begin{enumerate}[label=(\roman*),leftmargin=2.2em,itemsep=2pt,topsep=3pt]
    \item \emph{Stochastic rules.}
    Fix $\kappa\in(0,1]$ and $\alpha>0$, assume $N\geq V\alpha$, and include
    $\star$ with $\sigma_\star=a+\log(V\alpha)$. If
    \[
        \Delta
        \geq
        \log
        \frac{
            4N
        }{
            (n_c+V\alpha)\kappa
        }
        \qquad\text{and}\qquad
        \eblock
        \leq
        \frac{\kappa}{16k\Delta},
    \]
    then at every resolvable query, cell token or separator,
    under the simplex readout, the Layer-2 output $p$ satisfies
    \[
        \KL\!\left(
            \widehat P_\alpha(\cdot\mid Q_n)
            \,\middle\|\,
            p
        \right)
        \leq
        \kappa.
    \]

    \item \emph{Deterministic rules.}
    Let the rule be deterministic and take $\alpha\to0$, with
    $\sigma_\star\leq a-\Delta$ whenever $\star$ is present (always in
    Construction~II). If
    \[
        \eblock
        <
        \frac1{2k+2}
    \]
    and
    \[
        \Delta
        >
        \log\bigl(
            2\sqrt2\,(N-n_c)/n_c
        \bigr),
    \]
    then at every resolvable query, cell token or separator, the predicted
    state is exactly correct. This conclusion holds for either the simplex
    readout or the standard softmax readout.
\end{enumerate}
\end{theorem}

\begin{rmk}
Both gate hypotheses are indispensable. Dropping $\Gamma_1$ to the weaker
choice $\Gamma_1=\Lambda$ breaks the per-block error bound
whenever $\gamma>1$; dropping the $\zeta$ half of $\Gamma_2$ makes part~(ii)
outright false, by Remark~\ref{rmk:initstep-counterexample}.
\end{rmk}

\begin{proof}
Throughout, call a visible position \emph{extraneous} if it is neither a
matching precedent nor the sink. Three kinds of position are extraneous:
non-matching valid keys, separators, and initial-step cells; in part~(ii),
where $\star$ plays no Dirichlet role, $\star$ is extraneous as well. We treat
them together, and the gate hypotheses enter only here.

\emph{Part~(i).}
We separate the error into two terms: bias within the matched-plus-sink set and
contamination from the extraneous positions.

By Lemma~\ref{lem:dispersion}, the matched scores have the form
\[
    a+\eta_j,
    \qquad
    a=\beta(k-2k\eblock),
    \qquad
    0\leq\eta_j\leq\eta
    \leq
    2\beta k\eblock,
\]
while every non-matching valid key scores at most $a-\Delta$. By
Proposition~\ref{prop:keygate-l2}, the choice of $\Gamma_2$ puts every invalid
key under the same ceiling $a-\Delta$. Hence \emph{every} extraneous position
scores at most $a-\Delta$, and every one of them carries a probability vector
as its value. Place the sink at
\[
    a+\log(V\alpha).
\]

First remove the extraneous positions and renormalize. Let $p_\eta$ denote the
resulting matched-plus-sink output. By Lemma~\ref{lem:biasbudget},
\[
    p_\eta
    \geq
    e^{-2\eta}\widehat P_\alpha
\]
coordinatewise.

Now restore them. Let $w$ be their total attention mass and
let $R$ be the normalized average of their value vectors, which is a
probability vector because each of those values is one. Then
\[
    p
    =
    (1-w)p_\eta+wR
    \geq
    (1-w)e^{-2\eta}\widehat P_\alpha.
\]
Therefore
\begin{equation}
\label{eq:kl-compose}
    \KL\!\left(
        \widehat P_\alpha
        \,\middle\|\,
        p
    \right)
    \leq
    2\eta-\log(1-w).
\end{equation}

We bound the two terms separately.

With at most $N$ visible positions in total, the unnormalized mass of extraneous positions cannot exceed $Ne^{a-\Delta}$. Gating invalid keys proves crucial here: separators and initial-step cells (already counted in $N$) are forced below this $e^{a-\Delta}$ ceiling. Meanwhile, the $n_c$ matching keys each carry weight at least $e^a$, and the sink carries $V\alpha e^a$.
\[
    w
    \leq
    \frac{
        Ne^{-\Delta}
    }{
        n_c+V\alpha
    }
    \leq
    \frac\kappa4
\]
by the first hypothesis. Since $\kappa\leq1$, we have $w\leq1/4$ and therefore
\[
    -\log(1-w)
    \leq
    2w
    \leq
    \frac\kappa2.
\]

For the dispersion term,
\[
    \beta
    =
    \frac{
        \Delta
    }{
        1-(2k+2)\eblock
    }.
\]
The hypotheses imply $\eblock\leq1/(4k+4)$, so the denominator is at least
$1/2$ and therefore $\beta\leq2\Delta$. Consequently
\[
    \eta
    \leq
    2\beta k\eblock
    \leq
    4k\Delta\eblock
    \leq
    \frac\kappa4,
\]
using the second hypothesis. Thus $2\eta\leq\kappa/2$. Substituting the two
bounds into \eqref{eq:kl-compose} gives the claim.

For completeness, the implication
$\eblock\leq1/(4k+4)$ follows from the stated assumptions. Since
$n_c\leq N$, $V\alpha\leq N$, and $\kappa\leq1$, the first hypothesis gives
$\Delta\geq\log2$. Hence
\[
    \eblock
    \leq
    \frac1{16k\Delta}
    \leq
    \frac1{16k\log2}
    \leq
    \frac1{4k+4}.
\]

\emph{Part~(ii).}
Resolvability gives $n_c\geq1$, and by the previous paragraphs every extraneous
position, invalid keys included, scores at most $a-\Delta$; so does $\star$,
since $\sigma_\star\leq a-\Delta$, and its value $V^{-1}\mathbf 1$ is a
probability vector. So
Corollary~\ref{lem:det-retrieval} applies directly with the gap $\Delta$, with
its $N-n_c$ non-matching positions read as the extraneous ones, and
returns the correct $\arg\max$. For Construction~I, choosing the Layer-1
logit scale so that
\[
    \Lambda\gamma
    >
    \log\bigl((2k+2)N\bigr)
\]
ensures
$\eblock<1/(2k+2)$; Construction~II has $\eblock=0$ exactly.
\end{proof}

\paragraph{Separator queries.}
The theorem covers separator queries under the same hypotheses. Two conditions
are needed. On the query side, Definition~\ref{def:sih-app}\emph{(i)} requires
the gather at every position with $\tau(n)\geq1$, separators included, and
Corollary~\ref{cor:gather-interface} discharges it for both constructions; the
underlying facts are Lemma~\ref{lem:sepsucc}\emph{(i)}--\emph{(iii)}, which
place the separator's target cell at $\kappa(n)+\usucc$, its $k$ parents in
step $\tau(n)-1$, and those parents at offsets $D_Q\subseteq\Uset$ from
$\kappa(n)$. On the query side, Layer~2 reads only the $D_Q$ blocks and never the
query token's own state, so a separator query needs nothing further, while a
separator \emph{key} is gated by $\nu$ \eqref{eq:nu} and enters the theorem as
an extraneous position. A separator is therefore a legitimate query and an
illegitimate key, and the single gate $\nu$ separates the two roles. This is
the separator entry of Table~\ref{tab:duty}: the same successor rule,
positional table and attention heads predict the first cell of a row, so no
separate routing construction is required.

\subsection{Choosing finite parameters}
\label{app:syn:params}

The hypotheses of Theorem~\ref{thm:main-app} couple $\eblock$, $\Delta$ and
$\beta$, and one order of choice satisfies them all. Fix the divergence budget
$\kappa\in(0,1]$ and the pseudocount $\alpha>0$.
\begin{enumerate}[label=(\arabic*),leftmargin=2.2em,itemsep=1pt,topsep=2pt]
  \item Take $\Delta=\log\bigl(4N/((n_c+V\alpha)\kappa)\bigr)$, the smallest
        value the first hypothesis allows. It depends on the stream length
        $N$, the smoothing $V\alpha$, the budget $\kappa$ and the match count
        $n_c$; the paragraph on one circuit below removes the dependence on
        $n_c$.
  \item Take $\eblock=\kappa/(16k\Delta)$, the largest per-block error the
        second hypothesis allows at that $\Delta$.
  \item Take $\beta=\Delta/(1-(2k+2)\eblock)$, which realizes the prescribed
        $\Delta$ and, by the proof above, satisfies $\beta\leq2\Delta$, so the
        dispersion obeys $\eta\leq\kappa/4$.
  \item Take $\Gamma_1\geq\Lambda(R+\gamma)$ by
        Proposition~\ref{prop:sepgate-l1} and
        $\Gamma_2\geq2\beta k\eblock+\Delta$ by
        Proposition~\ref{prop:keygate-l2}, and
        $\sigma_\star=a+\log(V\alpha)$.
\end{enumerate}
The Layer-1 logit scale comes last, and it is where the two constructions
differ. Construction~II attains $\eblock=0$ exactly, so $\eta=0$, step~(2) is
vacuous, the dispersion term drops out of \eqref{eq:kl-compose}, and only the
leakage term $-\log(1-w)$ survives; any $\Lambda$ meeting
\eqref{eq:onehead-scale} serves. Construction~I attains
$\eblock\leq Ne^{-\Lambda\gamma}$, so step~(2) becomes a condition on
$\Lambda$, namely $\Lambda\geq\gamma^{-1}\log\bigl(16kN\Delta/\kappa\bigr)$ by
Lemma~\ref{lem:blockerror}. For deterministic rules, take $\alpha\to0$ and
replace step~(1) by $\Delta>\log\bigl(2\sqrt2(N-n_c)/n_c\bigr)$, step~(2) by
$\eblock<1/(2k+2)$, and $\sigma_\star$ by any value at most $a-\Delta$;
Construction~I then needs
$\Lambda\gamma>\log\bigl((2k+2)N\bigr)$, as in part~\emph{(ii)} of the proof.

\paragraph{One query, or one circuit.}
The recipe above reads $n_c$ as the count at the query being bounded, and
step~(1) gives the smallest $\Delta$ meeting the first hypothesis at that
query. The weights of a circuit cannot depend on $n_c$, which is a property of
the context and is observed at run time. For one fixed circuit which works
uniformly over every resolvable query, run the same recipe at the worst case
resolvability permits, $n_c\geq1$. Step~(1) becomes
\[
    \Delta
    =
    \log\frac{4N}{(1+V\alpha)\kappa} ,
\]
and the deterministic replacement becomes
$\Delta>\log\bigl(2\sqrt2\,(N-1)\bigr)$; steps~(2)--(4) and the Layer-1 logit
scale then follow from that $\Delta$ unchanged. The count enters only the
first hypothesis, and it enters decreasingly, so meeting that hypothesis at
$n_c=1$ meets it at every resolvable query. No parameter of the model changes
with the observed count.

\begin{rmk}
The stochastic inequality characterizes a tradeoff between routing and retrieval. Increasing $\beta$ widens both the separation gap $\Delta$ and the matched-score dispersion $\eta$. Consequently, we cannot lower the softmax temperature arbitrarily. For stochastic rules, an excessive logit scale concentrates mass on a strict subset of matching precedents, destroying the uniform weighting required for empirical counting.
\end{rmk}

\apppart{Part V}{Experiments}
\section{Experimental methods and additional results}
\label{app:expdetails}

\subsection{Data generation and rule splits}
\label{app:data}

\paragraph{Trajectory generation.}
Each trajectory fixes one rule table across all $T$ time steps.
Initial cells ($t{=}0$) are drawn i.i.d.\ from $\{0, \ldots, V{-}1\}$, and all spatial axes use periodic boundary conditions.
The neighborhood configuration (left--center--right for 1D; north--west--center--east--south for 2D von Neumann) is base-$V$ encoded to index the rule table.
Deterministic rules return the next cell value directly, whereas stochastic rules return a categorical distribution for sampling.

\paragraph{Serialization and masks.}
In 1D, each time step contributes $\lvert\Grid\rvert$ cell tokens followed by a \texttt{TIME\_SEP} token (omitting the final separator).
In 2D, the grid is serialized row by row, with \texttt{ROW\_SEP} tokens separating rows and \texttt{TIME\_SEP} tokens separating time steps.
This stream length is one token shorter than $N$ (Definition~\ref{def:serialization}), but this final omission affects no query cell.
Neither the rule table nor rule ID is provided in the input.
Training loss spans the full next-token sequence.
Validation and test masks isolate the prediction region ($t = T_0, \ldots, T{-}1$), and reported metrics evaluate only cell-value targets, excluding separators.

\paragraph{Coverage filtering.}
Select evaluation splits apply a trajectory-level coverage filter: candidate trajectories are rejection-sampled until all $V^k$ configurations appear within the $T_0{-}1$ observed transitions of the context ($t = 0, \ldots, T_0{-}2$).
This inclusion criterion guarantees that every scored configuration has been observed in context.
Table~\ref{tab:coverage_filter} details filter usage across splits.

\begin{table}[h]
\centering
\caption{Trajectory coverage filter by split.
``Yes'' means the split rejection-samples trajectories until all $V^k$ neighborhood configurations appear in the context; ``No'' means no filter.
Deterministic runs use five splits (train, val-in without filter, val-in with filter, val-out, test); stochastic runs use three (train, validation, test).
The main-text deterministic validation refers to val-out; val-in is diagnostic only and not used for checkpoint selection.}
\label{tab:coverage_filter}
\small
\begin{tabular}{@{}lccccc@{}}
\toprule
Dataset type & Train & Val-in (no filter) & Val-in (filter) & Val-out / Val & Test \\
\midrule
Deterministic & No & No  & Yes & Yes & Yes \\
Stochastic    & No & --- & --- & Yes & Yes \\
\bottomrule
\end{tabular}
\end{table}

\paragraph{Deterministic rule pools.}
Each deterministic trajectory draws one rule uniformly at random (with replacement across trajectories) from the corresponding split's rule pool.
Training, validation-out, and test rule pools are disjoint at the canonical-rule level.
The split is drawn on canonical rules first, then each canonical rule is expanded through its symmetry orbit, so equivalent rules never appear in different splits.

\subparagraph{ECA $V{=}2, k{=}3$.}
The $256$ elementary rules form $88$ canonical classes under reflection, complementation, and their composition.
Rules are stratified by Wolfram behavioral classes I--IV before splitting.
The canonical-level split is $70 / 9 / 9$ (train / val-out / test); expanding each class through its symmetry orbit gives $203 / 28 / 25$ element rules.
The val-in split draws rules from the training pool with new initial conditions; the val-in-filter variant restricts to training rules that can generate a full-coverage context.

\subparagraph{1D $V{=}3, k{=}3$.}
The rule space has $3^{27} \approx 7.6 \times 10^{12}$ entries and cannot be enumerated.
The equivalence relation uses $S_3$ state permutations; a canonical rule is the lexicographically smallest of its six state-permutation variants.
Rules are stratified by sorted output histogram, and the entropy of each histogram defines an entropy bin.
Rules are sampled to make the resulting canonical set approximately uniform over discrete entropy values.
Pool sizes: 3{,}621{,}300 training / 122{,}640 val-out / 122{,}640 test rules.
For low-entropy strata that rarely admit full-coverage contexts, the val-out and test rules are drawn symmetrically from the coverage-feasible subset; the three canonical pools remain disjoint.

\subparagraph{2D von Neumann $V{=}2, k{=}5$.}
The rule space has $2^{32} \approx 4.3 \times 10^{9}$ entries.
The equivalence relation uses state complementation only.
Rules are stratified by folded Langton $\lambda$-strata (indexed by $\min(\#1, 32-\#1)$).
Small strata are enumerated; large strata are sampled at fixed number-of-ones per rule.
The canonical pools for training, val-out, and test are disjoint.
Near-constant strata contribute val-out and test rules symmetrically from the coverage-feasible subset.
After orbit expansion and deduplication, the pools contain 315{,}859 training / 10{,}904 val-out / 10{,}904 test rules.

\paragraph{Stochastic rule sampling.}
Each stochastic trajectory uses a fresh rule table fixed across all $T$ steps.
For each configuration $c \in [V]^k$, the next-cell categorical distribution $p_c$ is drawn independently from $\text{Dirichlet}(1, \ldots, 1)$ over $[V]$.
As the Dirichlet distribution is continuous, independent tables coincide with probability zero.
Consequently, trajectories almost surely use distinct rules, ensuring validation and test trajectories present rules unobserved during training.
Rule-table shapes are $8 \times 2$ (1D $V{=}2, k{=}3$), $27 \times 3$ (1D $V{=}3, k{=}3$), and $32 \times 2$ (2D VN $V{=}2, k{=}5$).

\paragraph{Dataset sizes.}
All settings use 50M training trajectories, 500k validation, and 500k test trajectories.
The same counts apply to both deterministic and stochastic datasets.
Table~\ref{tab:settings} summarizes the grid, horizon, and context length for each setting.

\begin{table}[h]
\centering
\caption{Grid and trajectory parameters for each experimental setting.
Values shown as ``det / stoch'' when they differ between the two regimes.}
\label{tab:settings}
\small
\begin{tabular}{@{}llcccc@{}}
\toprule
Task & Rule type & Grid $\Grid$ & $T$ & $T_0$ & $V^k$ configs \\
\midrule
1D $V{=}2,\;k{=}3$      & det / stoch & 16          & 10 / 20 & 4 / 10 & 8  \\
1D $V{=}3,\;k{=}3$      & det / stoch & 32 / 16     & 12 / 20 & 8 / 10 & 27 \\
2D VN $V{=}2,\;k{=}5$   & det / stoch & $6 \times 6$ & 12 / 20 & 8 / 10 & 32 \\
\bottomrule
\end{tabular}
\end{table}

\subsection{Models and training}
\label{app:exp:training}

Tables~\ref{tab:arch_det} and~\ref{tab:arch_stoch} give the architecture and training hyperparameters for each setting.
All models use a two-layer transformer with learned absolute positional embeddings (\texttt{nn.Embedding}).
The main text refers to the two first-layer designs as \textbf{Shared-head} (heads$=[1,1]$, mlp$=[1,1]$) and \textbf{Dedicated-heads} (heads$=[H,1]$, mlp$=[0,1]$).

\begin{table}[h]
\centering
\caption{Architecture and training hyperparameters: deterministic settings.}
\label{tab:arch_det}
\small
\begin{tabular}{@{}ll cc cc cccc@{}}
\toprule
Task & Arch.
  & heads & mlp
  & $d_{\text{model}}$ & $d_{\text{ffn}}$
  & lr & wd & $\beta_2$ & ep \\
\midrule
1D $V{=}2,k{=}3$
  & Shared-head & $[1,1]$ & $[1,1]$ & 128  & 512  & $10^{-3}$          & 0.2 & 0.999 & 3 \\
  & Dedicated-heads  & $[4,1]$ & $[0,1]$ & 128  & 512  & $10^{-3}$          & 0.2 & 0.999 & 3 \\
\midrule
1D $V{=}3,k{=}3$
  & Shared-head & $[1,1]$ & $[1,1]$ & 256  & 1024 & $5{\times}10^{-4}$ & 0.2 & 0.95  & 8 \\
  & Dedicated-heads  & $[4,1]$ & $[0,1]$ & 256  & 1024 & $5{\times}10^{-4}$ & 0.2 & 0.95  & 8 \\
\midrule
2D VN $V{=}2,k{=}5$
  & Shared-head & $[1,1]$ & $[1,1]$ & 1024 & 4096 & $10^{-3}$          & 0.2 & 0.95  & 5 \\
  & Dedicated-heads  & $[8,1]$ & $[0,1]$ & 1024 & 4096 & $10^{-3}$          & 0.2 & 0.95  & 5 \\
\bottomrule
\end{tabular}
\end{table}

\begin{table}[h]
\centering
\caption{Architecture and training hyperparameters: stochastic settings.}
\label{tab:arch_stoch}
\small
\begin{tabular}{@{}ll cc cc cccc@{}}
\toprule
Task & Arch.
  & heads & mlp
  & $d_{\text{model}}$ & $d_{\text{ffn}}$
  & lr & wd & $\beta_2$ & ep \\
\midrule
1D $V{=}2,k{=}3$
  & Shared-head & $[1,1]$ & $[1,1]$ & 256  & 1024 & $10^{-3}$          & 0.2 & 0.999 & 15 \\
  & Dedicated-heads  & $[4,1]$ & $[0,1]$ & 128  & 512  & $10^{-3}$          & 0.2 & 0.999 & 8  \\
\midrule
1D $V{=}3,k{=}3$
  & Shared-head & $[1,1]$ & $[1,1]$ & 256  & 1024 & $5{\times}10^{-4}$ & 0.2 & 0.95  & 8  \\
  & Dedicated-heads  & $[4,1]$ & $[0,1]$ & 256  & 1024 & $5{\times}10^{-4}$ & 0.2 & 0.95  & 8  \\
\midrule
2D VN $V{=}2,k{=}5$
  & Shared-head & $[1,1]$ & $[1,1]$ & 1024 & 4096 & $5{\times}10^{-4}$ & 0.2 & 0.95  & 10 \\
  & Dedicated-heads  & $[8,1]$ & $[0,1]$ & 512  & 2048 & $10^{-3}$          & 0.2 & 0.95  & 5  \\
\bottomrule
\end{tabular}
\end{table}

\paragraph{Optimizer and schedule.}
All models are trained with AdamW ($\beta_1 = 0.9$; $\beta_2$ as listed).
The learning-rate schedule is cosine decay with a linear warmup over 0.4\% of total steps.
Gradient norms are clipped to 1.0 at every step.

\paragraph{Precision and hardware.}
Training uses \texttt{bfloat16} automatic mixed precision on 1--4 GPUs; wall-clock time per run ranges from a few hours to about a day.

\paragraph{Training objective.}
The training loss is the average cross-entropy of the next-token prediction over all sequence positions:
\begin{equation*}
\mathcal{L} \;=\; \frac{\sum_{b,j} m_{b,j} \bigl[-\log q_\theta(y_{b,j} \mid x_{b,\le j})\bigr]}{\sum_{b,j} m_{b,j}}.
\end{equation*}
The targets $y_{b,j}$ include both cell-state tokens ($\{0, \ldots, V{-}1\}$) and row/time separators; the model's softmax $q_\theta$ is taken over the full vocabulary.
The loss mask $m_{b,j} \in \{0, 1\}$ is supplied by the dataset and equals 1 at every position on the training split, so no token is excluded.
On the filtered evaluation splits (val-in-filter, val-out, test) the mask is 1 only at positions belonging to the prediction time steps ($t \ge T_0$) and 0 elsewhere, restricting the monitored loss to the region on which the model is scored.

\paragraph{Checkpoint selection.}
Deterministic checkpoints are selected on val-out sequence accuracy (maximized); stochastic checkpoints are selected on $\KL(q_{\text{true}} \| q_{\text{model}})$ evaluated on the val split (minimized).
Early stopping uses patience 999 and min\_delta 0.0005.
The best checkpoint is the one where the selection metric last improved by at least min\_delta.

\subsection{Evaluation protocol and metrics}
\label{app:exp:eval}

All evaluation metrics are computed only over cell-value targets at scored time steps; separator targets are excluded.

\paragraph{Deterministic teacher-forced (TF) cell accuracy.}
At each cell position from $t = T_0$ through $t = T{-}1$, the model predicts the next cell value conditioned on the ground-truth history up to that point.
TF cell accuracy is the fraction of correctly predicted cells (argmax of the model's output versus the ground-truth value).

\paragraph{Deterministic teacher-forced (TF) sequence accuracy.}
The fraction of trajectories in which every scored cell is predicted correctly, under the same teacher forcing as TF cell accuracy.
A trajectory counts only if it has no error at any scored position.

\paragraph{Deterministic autoregressive (AR) cell accuracy.}
Starting from the ground-truth context (time steps $0$ through $T_0{-}1$), the model predicts each subsequent cell via greedy decoding.
Each predicted cell is fed back as input for subsequent predictions.
AR cell accuracy is the fraction of generated cells that match the ground-truth trajectory.

\paragraph{Stochastic AR cell accuracy.}
The model generates a single sampled rollout from step $T_0$ onward at temperature 1 (unscaled softmax); cell accuracy is the fraction of sampled cells matching the stored ground-truth trajectory.
The ground truth is one realization of a stochastic process, so a correct predictive distribution does not yield accuracy~1.
The parenthesized $k$-gram baseline accuracy comes from an independent posterior-predictive sampling rollout. It matches the sampling procedure of the model's rollout, and it does not maximize accuracy (Appendix~\ref{app:bayes:optimal}).
The $k$-gram predictor starts from the same $T_0$-step ground-truth prefix and samples each cell from the empirical conditional distribution given by causal $k$-gram counts with add-one smoothing.
Both rollouts are scored against the same stored ground-truth trajectory.

\paragraph{Stochastic AR $\KL(q_{\text{kgram}} \| q_{\text{model}})$.}
At each cell in the model's own autoregressive rollout, before sampling that cell's value:
(1) $q_{\text{model}}$ is the softmax over the first $V$ output logits, corresponding to the cell states;
(2) $q_{\text{kgram}}$ is computed from causal counts in the model-generated history up to that point, with add-one smoothing.
The KL divergence is micro-averaged over all scored cells and trajectories.
This $k$-gram reference is conditioned on the model's own history and is not taken from the independent $k$-gram rollout.
Taking the softmax over the first $V$ logits is equivalent to taking the full-vocabulary softmax and renormalizing its cell-state probabilities.
During autoregressive generation, separator tokens are inserted deterministically at their prescribed positions rather than sampled from the model.

\paragraph{Stochastic TF cell accuracy.}
Teacher-forced cell accuracy for stochastic rules uses argmax of the model's softmax over the first $V$ logits, compared against the ground-truth cell value.
It is reported with the $k$-gram baseline in parentheses.

\paragraph{Stochastic TF $\KL$.}
Two divergences are reported under teacher forcing, both micro-averaged over scored cells.
$\KL(q_{\text{kgram}} \| q_{\text{model}})$ measures how closely the model's distribution matches the $k$-gram estimate when both are conditioned on the ground-truth history.
$\KL(q_{\text{true}} \| q_{\text{model}})$ compares the model against the conditional distribution of the rule which generated the trajectory.

\paragraph{Seeds and rollouts.}
All evaluations use random seed 42.
One AR rollout is generated per trajectory (both for the model and for the $k$-gram baseline).
The reported values are micro-averages over 500k test trajectories.

\subsection{Full predictive results}
\label{app:exp:fullresults}

Tables~\ref{tab:det_full} and~\ref{tab:stoch_full} expand the compact main-text Table~\ref{tab:results} with additional metrics.

\begin{table}[h]
\centering
\caption{Deterministic evaluation: full metrics on held-out test rules.}
\label{tab:det_full}
\small
\begin{tabular}{@{}ll ccc@{}}
\toprule
Task & Arch. & TF seq acc & TF cell acc & AR cell acc \\
\midrule
1D $V{=}2,\;k{=}3$
  & Shared-head & 0.9990 & 1.0000 & 0.9999 \\
  & Dedicated-heads  & 0.9994 & 1.0000 & 1.0000 \\
\midrule
1D $V{=}3,\;k{=}3$
  & Shared-head & 0.9994 & 1.0000 & 1.0000 \\
  & Dedicated-heads  & 0.9989 & 1.0000 & 1.0000 \\
\midrule
2D VN $V{=}2,\;k{=}5$
  & Shared-head & 1.0000 & 1.0000 & 1.0000 \\
  & Dedicated-heads  & 0.9989 & 1.0000 & 0.9999 \\
\bottomrule
\end{tabular}
\end{table}

\begin{table}[h]
\centering
\caption{Stochastic evaluation: full metrics on held-out test trajectories.
TF KL columns report teacher-forced divergence against the $k$-gram and true-rule baselines.}
\label{tab:stoch_full}
\small
\begin{tabular}{@{}ll cc c cc@{}}
\toprule
& & \multicolumn{2}{c}{TF KL} & TF cell acc & \multicolumn{2}{c}{AR} \\
\cmidrule(lr){3-4} \cmidrule(lr){6-7}
Task & Arch.
  & kg$\|$m & true$\|$m
  & (kgram)
  & cell acc (kgram) & KL kg$\|$m \\
\midrule
1D $V{=}2,k{=}3$
  & Shared-head & 0.0020 & 0.0185 & 0.7491\;(0.7499) & 0.5496\;(0.5485) & 0.0019 \\
  & Dedicated-heads  & 0.0046 & 0.0211 & 0.7469\;(0.7499) & 0.5488\;(0.5485) & 0.0044 \\
\midrule
1D $V{=}3,k{=}3$
  & Shared-head & 0.0026 & 0.0881 & 0.5669\;(0.5674) & 0.3476\;(0.3475) & 0.0026 \\
  & Dedicated-heads  & 0.0004 & 0.0859 & 0.5675\;(0.5674) & 0.3475\;(0.3475) & 0.0004 \\
\midrule
2D VN $V{=}2,k{=}5$
  & Shared-head & 0.0014 & 0.0300 & 0.7345\;(0.7351) & 0.5201\;(0.5197) & 0.0013 \\
  & Dedicated-heads  & 0.0033 & 0.0319 & 0.7333\;(0.7351) & 0.5191\;(0.5197) & 0.0038 \\
\bottomrule
\end{tabular}
\end{table}

All deterministic models exceed 0.998 TF sequence accuracy.
Among the stochastic metrics, the gap between $\KL(q_{\text{kgram}} \| q_{\text{model}})$ and $\KL(q_{\text{true}} \| q_{\text{model}})$ reflects the estimation error of the $k$-gram approximation: the first measures the model against an estimate built from the same context the model sees, the second against the rule which generated it.

\subsection{Attention-alignment analysis}
\label{app:attention}

For each scored query cell, we measure the fraction of attention mass assigned to the positions predicted by the corresponding construction.
For L1, the target set consists of the routing cells that carry the query configuration.
For L2, the target set consists of the matching precedents, the earlier cells whose own parent configuration is a matching configuration for the query.

\paragraph{Aggregation.}
For a single query, the L1 score is computed per head as the sum of attention weights over the target positions, and the per-head scores are then averaged across heads.
The L2 layer uses a single head, so the L2 score is the sum of attention weights over the target positions.
The reported values are micro-averages of these per-query scores across all scored queries and trajectories.

\paragraph{Trajectories.}
Attention weights are extracted under teacher forcing from 20{,}000 test trajectories sampled uniformly without replacement from the test split.

\subsection{Attention ablation}
\label{app:attn_ablation}

We test whether the routing cells and matched precedents identified by the construction are functionally used by the model.
For each scored query, we intervene on one attention row and measure the effect on next-token prediction.
Four intervention conditions are compared against the unablated (\emph{No ablation}) baseline:
(1) \emph{L1 routing cells}: the attention weights of the L1 head at the query's routing cells are set to zero and the row is renormalized;
(2) \emph{L1 random}: the same number of L1 attention weights are zeroed at random positions in the same time step, as a matched control;
(3) \emph{L2 matched precedents}: the L2 attention weights at earlier cells with a matching parent configuration are zeroed and the row is renormalized;
(4) \emph{L2 random}: the same number of L2 attention weights are zeroed at random non-matching completed cells.
For architectures with multiple L1 heads, the L1 intervention is applied per head.
Random controls are averaged over three seeds; queries for which a disjoint random set cannot be formed are left unmodified.

Interventions are applied on the same 20{,}000 test trajectories as the alignment analysis, and we report TF cell accuracy under each condition.

\subsection{Layerwise probe protocol}
\label{app:probes}

\paragraph{Probe targets.}
We train linear probes on frozen post-residual representations to test what information is linearly accessible at each network stage.
Two targets are probed:

\begin{enumerate}[nosep]
  \item \emph{Neighborhood configuration}: a classification target.
  For a scored target cell $x_{t,z}$ at serialized position $m$, we extract the representation at the preceding query position $n=m-1$ and pair it with the target cell's parent configuration $P_{t,z}$.
  For a first-column target, the query position is the preceding separator.
  There are $V^k$ possible configurations.
  The probe is a linear classifier (logistic regression) and we report balanced accuracy (macro-averaged over configurations).

  \item \emph{$k$-gram predictive distribution}: a distributional target.
  The target distribution over the next cell value is computed from causal $k$-gram counts in the trajectory prefix with add-one smoothing.
  The probe is a multinomial logistic regression trained to minimize $\KL(q_{\text{kgram}} \| q_{\text{probe}})$.
  The plotted metric is $1 - \KL(q_{\text{kgram}} \| q_{\text{probe}}) / \KL(q_{\text{kgram}} \| \text{Uniform}([V]))$:
  a score of 0 means the probe predicts no better than uniform; 1 means exact recovery of $q_{\text{kgram}}$.
\end{enumerate}

\paragraph{Extraction points.}
Representations are extracted at five post-residual stages: after embedding, after L1 attention, after L1 MLP, after L2 attention, and after L2 MLP.
Each extraction is the full residual stream after the corresponding sub-layer (after the residual connection).

\paragraph{Probe model and training.}
The configuration probe uses \texttt{sklearn.linear\_model.LogisticRegression} with hard configuration labels and L2 regularization, solved by L-BFGS.
The predictive-distribution probe uses a custom multinomial logistic regression fit with PyTorch L-BFGS, soft-label cross-entropy, and L2 regularization; its data-fit term minimizes $\KL(q_{\text{kgram}} \| q_{\text{probe}})$ up to a constant independent of the probe.
For both probes, the regularization strength $C$ is selected on the probe validation fold from a grid search.
The maximum iteration count is 500 for hyperparameter search and 2500 for the final refit; convergence is determined by the L-BFGS tolerance.

\paragraph{ECA probe split (rotating 5-fold).}
The ECA test rule pool contains 25 expanded rules from 9 canonical Wolfram classes.
A fixed 60/20/20 split would place only 5 rules in the probe test fold, and the reported probe score would depend heavily on which 5 rules are picked.
To eliminate this rule-level nuisance variability, ECA uses a rotating 5-fold split: the 25 test rules are partitioned into 5 groups of 5, and each fold uses 15 rules for training, 5 for validation, and 5 for testing.
Every rule appears once as a test rule and once as a validation rule across the 5 folds.
Reported scores are averaged across folds.

\paragraph{Other deterministic settings (V=3, 2D VN).}
The 1D $V{=}3, k{=}3$ and 2D VN test pools contain thousands of rules, so a single fixed 60/20/20 rule-level split is already well-averaged and no rotation is used.

\paragraph{Stochastic probe splits.}
Because each stochastic trajectory has a unique rule, the probe split is a fixed 60/20/20 split over trajectories.
The rule table bytes are hashed to assign a rule ID per trajectory, and the same splitting logic as for deterministic probes is applied.

\paragraph{Balancing.}
For deterministic probes, cells are balanced across neighborhood configurations by per-label stratified subsampling.
For stochastic probes, the target is the causal $k$-gram predictive distribution on the trajectory prefix.

\subsection{Representation visualizations with t-SNE}
\label{app:exp:tsne}

Figure~\ref{fig:tsne} shows t-SNE embeddings of frozen residual-stream representations from the Shared-head model on deterministic 1D $V{=}2, k{=}3$ (checkpoint from Table~\ref{tab:arch_det}: heads $[1,1]$, mlp $[1,1]$, $d_{\text{model}}=128$).

\paragraph{Activation collection.}
Representations are extracted from 6000 test-split trajectories on the ECA test-rule pool.
For each of the five post-residual stages (embedding, L1 attention, L1 MLP, L2 attention, L2 MLP), we retain 800 points per neighborhood configuration, giving $8 \times 800 = 6400$ points per stage.
Within each configuration, points are stratified by target value (400 per configuration $\times$ target).
Sampling uses seed 42.

\paragraph{Embedding.}
Each stage is embedded independently.
The $(6400, 128)$ activation matrix for a stage is first reduced to 50 dimensions by PCA, then to 2 dimensions by t-SNE with perplexity 30, PCA initialization, automatic learning rate, 1000 iterations, and random seed 42; other t-SNE options are the scikit-learn defaults (Euclidean metric, Barnes--Hut, angle 0.5, early exaggeration 12).

\paragraph{Figure layout.}
Each of the five stages occupies one column of a $2 \times 5$ panel grid; the top row colors points by neighborhood configuration and the bottom row colors by output value.
Point plotting order is shuffled with a separate seed to avoid systematic overplotting.

\paragraph{Additional t-SNE plots.}
We extend the main-text t-SNE to three further settings that cover the axes not shown in the body: Dedicated-heads deterministic 1D (architecture comparison, Figure~\ref{fig:tsne_po_det_1d}), Shared-head stochastic 1D (stochastic regime, Figure~\ref{fig:tsne_sh_stoch_1d}), and Shared-head 2D VN (multidimensional regime, Figure~\ref{fig:tsne_sh_det_2dvn}).
These t-SNE plots are intended as qualitative visualizations.
Each stage is projected from a high-dimensional representation into two dimensions, so visible distances and clusters need not faithfully reflect the original representation.
In particular, a lack of clear clusters does not imply that the relevant information is absent.
We rely on the linear probes for quantitative evidence that the configuration and predictive distribution are decodable at each stage.

\begin{figure}[htb]
\centering
\includegraphics[width=0.95\linewidth]{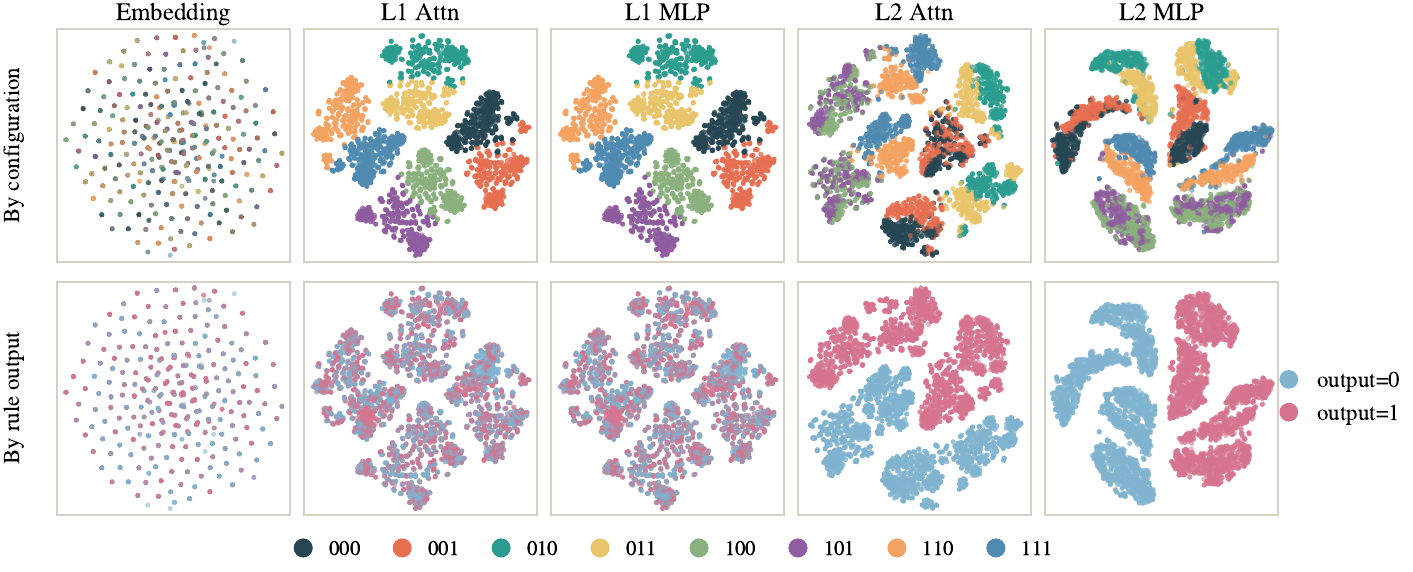}
\caption{t-SNE of frozen residual-stream representations from the Dedicated-heads model on deterministic 1D $V{=}2, k{=}3$.
Top row: colored by neighborhood configuration. Bottom row: colored by the ground-truth output value.
Under the Dedicated-heads construction, the configuration should become accessible after L1 attention rather than after the L1 MLP.}
\label{fig:tsne_po_det_1d}
\end{figure}

\begin{figure}[htb]
\centering
\includegraphics[width=0.95\linewidth]{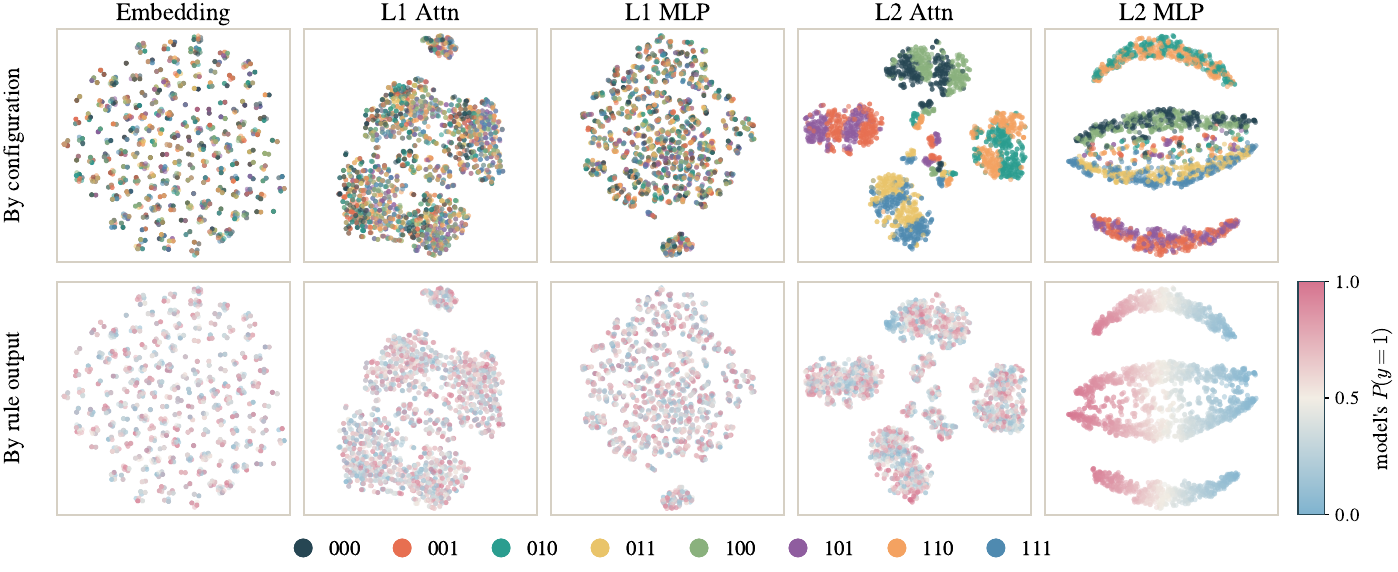}
\caption{t-SNE of frozen representations from the Shared-head model on stochastic 1D $V{=}2, k{=}3$.
Top row: colored by neighborhood configuration. Bottom row: colored by the model's predicted $P(y{=}1)$ at the corresponding query position.
Because rules are drawn from a continuous Dirichlet prior, the output value at a fixed configuration is a sample from the true conditional distribution, so the model's predictive distribution varies continuously across queries with the same configuration.}
\label{fig:tsne_sh_stoch_1d}
\end{figure}

\begin{figure}[htb]
\centering
\includegraphics[width=0.95\linewidth]{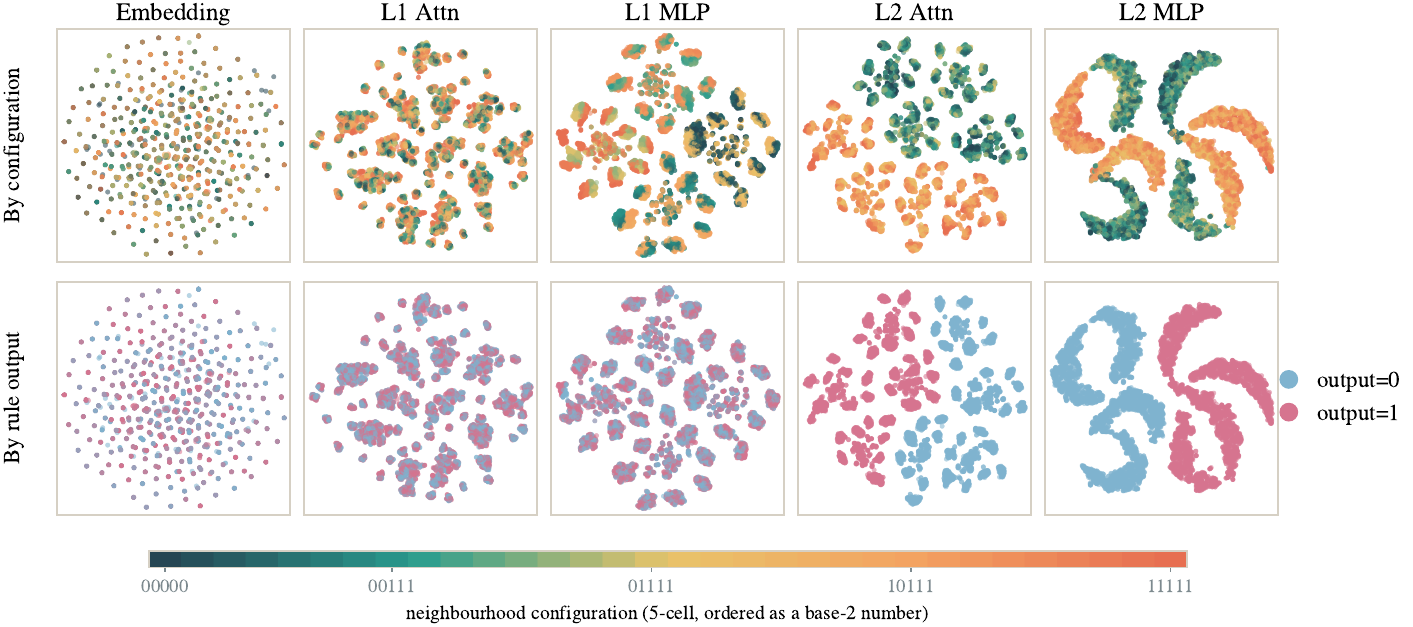}
\caption{t-SNE of frozen representations from the Shared-head model on deterministic 2D von Neumann $V{=}2, k{=}5$.
Top row: colored by neighborhood configuration. Bottom row: colored by the ground-truth output value.
The larger neighborhood ($k=5$, $V^k = 32$ configurations) is harder to display in two dimensions than the ECA case; per-configuration clusters are correspondingly less separated.}
\label{fig:tsne_sh_det_2dvn}
\end{figure}

\subsection{Qualitative rollout setup for Figure~\ref{fig:rollout}}
\label{app:rollout_setup}

Figure~\ref{fig:rollout} uses two checkpoints trained separately for visualization purposes, one on the deterministic 1D $V{=}2, k{=}3$ setting and one on the stochastic 1D $V{=}2, k{=}3$ setting.
Both checkpoints are trained on longer trajectories and a wider grid than the main-text models in order to display long rollouts at a legible scale.
Wolfram Rule~110 and the Stavskaya rule are excluded from the training rule pools of the two respective checkpoints, so the trajectories shown in Figure~\ref{fig:rollout} come from rules the models have never seen.
These checkpoints are used only for Figure~\ref{fig:rollout}: none of the predictive results in Table~\ref{tab:results}, the attention-alignment analysis, or the linear-probe experiments use them.
Table~\ref{tab:rollout_setup} reports their training and evaluation configurations.

\begin{table}[h]
\centering
\caption{Configuration of the visualization checkpoints used for Figure~\ref{fig:rollout}.
The deterministic panel uses greedy decoding; the stochastic panel uses temperature-1 sampling, and the displayed values are per-cell probabilities of state~1.}
\label{tab:rollout_setup}
\small
\begin{tabular}{@{}l cc@{}}
\toprule
 & Deterministic (Rule~110) & Stochastic (Stavskaya) \\
\midrule
Task           & 1D $V{=}2, k{=}3$ & 1D $V{=}2, k{=}3$ \\
Training $\lvert\Grid\rvert, T, T_0$ & $64, 64, 12$ & $64, 64, 12$ \\
Plotted $\lvert\Grid\rvert, T, T_0$  & $64, 64, 12$ & $64, 64, 12$ \\
heads / mlp    & $[4,1] / [0,1]$ & $[1,1] / [1,1]$ \\
$d_{\text{model}}$ / $d_{\text{ffn}}$ & $512 / 2048$ & $512 / 2048$ \\
lr, wd, epochs & $5{\times}10^{-4}$, 0.2, 10 & $10^{-3}$, 0.2, 8 \\
Batch (gbs), hardware & 1024, 2$\times$A6000 DDP & 1024, 2$\times$B300 DDP \\
Checkpoint selection & val-out seq acc (max) & $\KL(q_{\text{true}} \| q_{\text{model}})$ (min) \\
Decoding at plotted step & greedy argmax & Bernoulli sampling at temperature 1 \\
\bottomrule
\end{tabular}
\end{table}

\paragraph{Interpreting the stochastic panel.}
The Ground Truth column displays one sampled trajectory. The Autoregressive prediction column conditions on the model's own sampled history, precluding exact per-cell reproduction even for a Bayes-optimal predictor.
The Teacher-forced prediction column conditions on the ground-truth history, plotting the predictive probabilities against which the reference trajectory can be evaluated.

\end{document}